\documentclass[11pt]{article}
\usepackage[margin=1in]{geometry}
\usepackage[utf8]{inputenc}
\usepackage[T1]{fontenc} 

\usepackage{mathpazo}
\usepackage{xcolor}
\usepackage{color}
\definecolor{RED}{RGB}{250,0,0}
\definecolor{NICERED}{RGB}{190,38,38}
\definecolor{niceRed}{RGB}{190,38,38}
\definecolor{niceYellow}{HTML}{f5b400}
\definecolor{blueGrotto}{HTML}{059DC0}
\definecolor{royalBlue}{HTML}{057DCD}
\definecolor{navyBlue}{HTML}{0B579C}
\definecolor{limeGreen}{HTML}{81B622}
\definecolor{nicePurple}{HTML}{9c27b0}
\definecolor{lightRoyalBlue}{HTML}{def2ff}  
\definecolor{gold}{HTML}{ffa300}
\definecolor{codexTeal}{HTML}{00897B}
\definecolor{codexUltraPurple}{HTML}{6D28D9}

\newcommand{\white}[1]{\textcolor{white}{#1}}

\usepackage[suppress]{color-edits} 
\addauthor[Jane]{jl}{blue}
\addauthor[Vikram]{vk}{niceRed} 
\addauthor[Anay]{am}{gold}
\addauthor[Manolis]{mz}{teal}
\addauthor[GPT]{gpt}{nicePurple}
\addauthor[Claude]{claude}{orange}
\addauthor[Codex]{codex}{codexTeal}
\addauthor[Codex Ultra]{codexultra}{codexUltraPurple}

\usepackage{todonotes}

\usepackage{hyperref}
\hypersetup{
  colorlinks = true, 
  urlcolor = {blueGrotto},
  linkcolor = {royalBlue},
  citecolor = {navyBlue}
}
\usepackage[
  backref=true,
  backend=biber,
  natbib=true,
  style=alphabetic,
  sorting=alphabeticlabel,
  sortcites=true,
  minbibnames=3,
  maxbibnames=999,
  mincitenames=1,
  maxcitenames=4,
  minalphanames=1,
  maxalphanames=6,
  doi=false, url=false, eprint=false, isbn=false, 
]{biblatex}

\DeclareSortingTemplate{alphabeticlabel}{
  \sort[final]{%
    \field{labelalpha} 
  }
  \sort{%
    \field{year}   
  }
  \sort{%
    \field{title}  
  }
}

\AtBeginRefsection{\GenRefcontextData{sorting=ynt}}
\AtEveryCite{\localrefcontext[sorting=ynt]}
\usepackage{booktabs} 
\usepackage{multicol} 
\usepackage{multirow} 
\usepackage{longtable}

\usepackage{subcaption}
\usepackage{graphicx}
\usepackage{float} 
\usepackage{tikz}
\usepackage{pgfplots}
\pgfplotsset{compat=1.17}
\usepgfplotslibrary{fillbetween}
\usetikzlibrary{arrows.meta}
\usetikzlibrary{calc}

\tikzset{
  myNodeFlex/.style={
    draw,
    rectangle,
    rounded corners,
    text centered,
    minimum height=1.5em,
  }
}

\tikzset{
  myNode/.style={
    draw,
    rectangle,
    rounded corners,
    text centered,
    minimum height=1.5em,
    minimum width=3cm,
    text width=5cm,    
  }
}

\tikzset{
  myNodeNarrow/.style={
    draw,
    rectangle,
    rounded corners,
    text centered,
    minimum height=1.5em,
    minimum width=1cm,
  }
}

\tikzset{
  myNodeWide/.style={
    draw,
    rectangle,
    rounded corners,
    text centered,
    minimum height=1.5em,
    minimum width=6cm,
  }
}

\usepackage{physics} 
\usepackage{amsmath}
\usepackage{amssymb}
\usepackage{amsthm}
\usepackage{bbm}
\usepackage{blkarray}
\usepackage{mathtools}
\usepackage{nicefrac}
\usepackage{dsfont}
\usepackage{pifont} 
\usepackage{thm-restate} 
\usepackage[capitalise,noabbrev,nameinlink,sort]{cleveref} 

\usepackage{setspace}

\usepackage{typed-checklist}
\usepackage{enumerate} 
\usepackage{enumitem} 
\usepackage{tcolorbox}
\usepackage[framemethod=TikZ]{mdframed}
\mdfsetup{%
backgroundcolor=gray!0, , 
roundcorner=4pt,
skipabove=15pt,
skipbelow=5pt,
linewidth=1pt}
\usepackage{changepage} 

\usepackage{algorithm}
\usepackage{algpseudocode}[1]

\usepackage{mathrsfs} 
\usepackage{soul} 
\setuldepth{Berlin}

\theoremstyle{plain} 
\newtheorem{theorem}{Theorem}[section]
\newtheorem{corollary}[theorem]{Corollary}
\newtheorem{proposition}[theorem]{Proposition}
\newtheorem{lemma}[theorem]{Lemma}

\newtheorem{claim}[theorem]{Claim}

\newtheorem{assumption}{Assumption}
\newtheorem{condition}{Condition}

\newtheorem{definition}{Definition}

\newtheorem*{definition*}{Definition}
\newtheorem{problem}{Problem}

\theoremstyle{definition} 
\newtheorem{example}[theorem]{Example}

\theoremstyle{remark}
\newtheorem{remark}[theorem]{Remark}

\AfterEndEnvironment{definition}{\noindent\ignorespaces}
\AfterEndEnvironment{claim}{\noindent\ignorespaces}
\AfterEndEnvironment{infdefinition}{\noindent\ignorespaces}
\AfterEndEnvironment{assumption}{\noindent\ignorespaces}
\AfterEndEnvironment{problem}{\noindent\ignorespaces}
\AfterEndEnvironment{openproblem}{\noindent\ignorespaces}
\AfterEndEnvironment{lemma}{\noindent\ignorespaces}
\AfterEndEnvironment{theorem}{\noindent\ignorespaces}
\AfterEndEnvironment{proposition}{\noindent\ignorespaces}
\AfterEndEnvironment{fact}{\noindent\ignorespaces}
\AfterEndEnvironment{question}{\noindent\ignorespaces}
\AfterEndEnvironment{corollary}{\noindent\ignorespaces}
\AfterEndEnvironment{model}{\noindent\ignorespaces}
\AfterEndEnvironment{remark}{\noindent\ignorespaces}
\AfterEndEnvironment{proof}{\noindent\ignorespaces}
\AfterEndEnvironment{fact}{\noindent\ignorespaces}
\AfterEndEnvironment{minftheorem}{\noindent\ignorespaces}
\AfterEndEnvironment{inftheorem}{\noindent\ignorespaces}
\AfterEndEnvironment{maintheorem}{\noindent\ignorespaces}
\AfterEndEnvironment{restatable}{\noindent\ignorespaces}
\AfterEndEnvironment{infassumption}
{\noindent\ignorespaces}
\AfterEndEnvironment{infcorollary}{\noindent\ignorespaces}

\crefname{section}{Section}{Sections}
\crefname{theorem}{Theorem}{Theorems}
\crefname{lemma}{Lemma}{Lemmas}
\crefname{problem}{Problem}{Problems}
\crefname{program}{Program}{Programs}
\crefname{definition}{Definition}{Definitions}
\crefname{conjecture}{Conjecture}{Conjectures}
\crefname{corollary}{Corollary}{Corollaries}
\crefname{construction}{Construction}{Constructions}
\crefname{conjecture}{Conjecture}{Conjectures}
\crefname{claim}{Claim}{Claims}
\crefname{observation}{Observation}{Observations}
\crefname{proposition}{Proposition}{Propositions}
\crefname{fact}{Fact}{Facts}
\crefname{question}{Question}{Questions}
\crefname{problem}{Problem}{Problems}
\crefname{remark}{Remark}{Remarks}
\crefname{example}{Example}{Examples}
\crefname{equation}{Equation}{Equations}
\crefname{appendix}{Appendix}{Appendices}
\crefname{algorithm}{Algorithm}{Algorithms}
\crefname{model}{Model}{Models}
\crefname{figure}{Figure}{Figures}
\crefname{infassumption}{Informal Assumption}{Informal Assumptions}
\crefname{inftheorem}{Informal Theorem}{Informal Theorems}
\crefname{infdefinition}{Informal Definition}{Informal Definitions}
\crefname{minftheorem}{Main Informal Theorem}{Main Informal Theorems}
\crefname{maintheorem}{Main Theorem}{Main Theorems}
\crefname{assumption}{Assumption}{Assumptions}
\crefname{step}{Step}{Steps}
\crefname{result}{Result}{Results}
\crefname{event}{Event}{Events}
\crefname{none}{}{}

\newlist{asmpenum}{enumerate}{1} 
\setlist[asmpenum]{label={\arabic*.},ref=\theassumption.{\arabic*}}
\crefname{asmpenumi}{Assumption}{Assumptions}

\usepackage{etoolbox}

\BeforeBeginEnvironment{subenvironment}{\white{.}\\ \vspace{-10mm}}
\AfterEndEnvironment{subenvironment}{\vspace{4mm}}

\newcommand{\yesnum}{\addtocounter{equation}{1}\tag{\theequation}} 
\makeatletter
\newcommand{\tagnum}[2]{%
    \refstepcounter{equation}%
    \tag{#1) \ (\theequation}%
    \protected@write \@auxout {}{%
        \string \newlabel {#2}{{\theequation}{\thepage}{}{equation.\theequation}{}}%
    }%
}
\makeatother

\makeatletter
\renewcommand{\eqref}[1]{\textup{\eqrefform@{\ref{#1}}}}
\let\eqrefform@\tagform@

\newcommand{\changetag}[1]{%
  \renewcommand\tagform@[1]{\maketag@@@{(\ignorespaces#1\unskip\@@italiccorr)}}%
}
\makeatother

\newcommand{\quadtext}[1]{\quad\text{#1}\quad}
\newcommand{\qquadtext}[1]{\qquad\text{#1}\qquad}
 
\newcommand{\quadand}{\quadtext{and}}
\newcommand{\qquadand}{\qquadtext{and}}

\newcommand{\quadwhere}{\quadtext{where}}

\def\abs#1{\left| #1 \right|}
\def\sabs#1{| #1 |}
\newcommand{\given}{\,\middle|\,}
\newcommand{\sinparen}[1]{(#1)}

\newcommand{\sinbrace}[1]{\{#1\}}

\newcommand{\inbrace}[1]{\left\{#1\right\}}

\newcommand{\inparen}[1]{\left(#1\right)}
\newcommand{\insquare}[1]{\left[#1\right]}
\newcommand{\inangle}[1]{\left\langle#1\right\rangle}

\newcommand{\snorm}[1]{\ensuremath{\| #1 \|}}
\let\norm\relax
\newcommand{\norm}[1]{\ensuremath{\left\lVert #1 \right\rVert}}

\newcommand{\R}{\mathbb{R}}

\newcommand{\E}{\operatornamewithlimits{\mathbb{E}}} 
\newcommand{\Ex}{\E}

\newcommand{\cov}{\ensuremath{\operatornamewithlimits{\rm Cov}}}
\newcommand{\Var}{\ensuremath{\operatornamewithlimits{\rm Var}}}
\newcommand{\corr}{\ensuremath{\operatornamewithlimits{\rm Corr}}}

\newcommand{\argmin}{\operatornamewithlimits{arg\,min}}
\newcommand{\argmax}{\operatornamewithlimits{arg\,max}}

\newcommand{\tv}[2]{\operatorname{d}_{\mathsf{TV}}\sinparen{#1,#2}}

\newcommand{\zo}{\ensuremath{\inbrace{0, 1}}}

\newcommand{\sfrac}[2]{{#1/#2}} 
\newcommand{\nfrac}[2]{\nicefrac{#1}{#2}}

\newcommand{\supp}{\operatorname{supp}}

\newcommand{\iid}{i.i.d.}

\newcommand{\eps}{\varepsilon}
\renewcommand{\epsilon}{\varepsilon}
\makeatletter
\newcommand*{\tran}{{\mathpalette\@tran{}}}
\newcommand*{\@tran}[2]{\raisebox{\depth}{$\m@th#1\intercal$}}
\makeatother

\mathchardef\NABLA"272
\newcommand*{\Nabla}{\boldsymbol\NABLA}
\let\nabla\Nabla

\renewcommand{\hat}{\widehat}
\newcommand{\wh}[1]{\widehat{#1}}

\renewcommand{\bar}{\overline}

\renewcommand{\tilde}{\widetilde}
\newcommand{\wt}[1]{\widetilde{#1}}

\newcommand{\customcal}[1]{\euscr{#1}}

\newcommand{\cC}{\customcal{C}}
\newcommand{\cD}{\customcal{D}}

\newcommand{\cF}{\customcal{F}}

\newcommand{\cN}{\customcal{N}}

\newcommand{\cP}{\customcal{P}} 
\newcommand{\cQ}{\customcal{Q}}

\newcommand{\cX}{\customcal{X}}

\newcommand{\cZ}{\customcal{Z}}

\DeclareMathAlphabet{\mathdutchcal}{U}{dutchcal}{m}{n}
\SetMathAlphabet{\mathdutchcal}{bold}{U}{dutchcal}{b}{n}
\DeclareMathAlphabet{\mathdutchbcal}{U}{dutchcal}{b}{n}

\DeclareMathAlphabet\urwscr{U}{urwchancal}{b}{n}%
\DeclareMathAlphabet\rsfscr{U}{rsfso}{m}{n}
\DeclareMathAlphabet\euscr{U}{eus}{m}{n}
\DeclareFontEncoding{LS2}{}{}
\DeclareFontSubstitution{LS2}{stix}{m}{n}
\DeclareMathAlphabet\stixcal{LS2}{stixcal}{m} {n}

    \renewcommand{\paragraph}[1]{\bigskip \noindent\textbf{#1}~}

\newcommand{\ie}{\textit{i.e.}}
\newcommand{\eg}{\textit{e.g.}}

\newcommand{\hypo}[1]{\mathbbmss{#1}}
\newcommand{\hyA}{\hypo{A}}
\newcommand{\hyB}{\hypo{B}}
\newcommand{\hyC}{\hypo{C}}
\newcommand{\hyD}{\hypo{D}}

\newcommand{\hyF}{\hypo{F}}

\newcommand{\hyH}{\hypo{H}}

\newcommand{\hyQ}{\hypo{Q}}

\newcommand{\hyS}{\hypo{S}}

\newcommand{\hyX}{\hypo{D}}

\newcommand{\hyDX}{\hypo{D}}
\newcommand{\hyDxi}{\hypo{Q}}

\renewcommand{\d}{{\rm d}}

\renewcommand{\hat}{\widehat}
\renewcommand{\tilde}{\widetilde}

\newcommand{\fstar}{f^\star}
\newcommand{\sstar}{s^\star}
\newcommand{\cDxi}{\cQ^\star}

\newcommand{\cDX}{\cD^\star}

\usepackage{array}
\newcolumntype{L}[1]{>{\raggedright\let\newline\\\arraybackslash\hspace{0pt}}m{#1}}
\newcolumntype{C}[1]{>{\centering\let\newline\\\arraybackslash\hspace{0pt}}m{#1}}
\newcolumntype{R}[1]{>{\raggedleft\let\newline\\\arraybackslash\hspace{0pt}}m{#1}}

\newcommand{\normal}[2]{\cN{\inparen{#1,#2}}}

\newcommand{\hyFlin}{\hyF_{\rm lin}}
\newcommand{\cover}[3]{N_{#2,#3}(#1)}

\newcommand\blfootnote[1]{%
  \begingroup
  \renewcommand\thefootnote{}\footnote{#1}%
  \addtocounter{footnote}{-1}%
  \endgroup
}

\title{Learning-Enabled Estimation:\\ Tight Characterizations under Sample Selection Biases}
\date{}

\author{ 
        \centering
        \hspace{4mm}
        \begin{tabular}{C{6.75cm}C{6.75cm}}
        {\bf Vikram Kher} & {\bf Jane H. Lee}\\[0mm]
        Yale University & Yale University\\[-1mm]
        \mbox{\small\href{mailto:vikram.kher@yale.edu}{\texttt{vikram.kher@yale.edu}}} &
        \mbox{\small\href{mailto:jane.h.lee@yale.edu}{\texttt{jane.h.lee@yale.edu}}}
        \\[4mm]
        {\bf Anay Mehrotra} & {\bf Manolis Zampetakis}\\[0mm]
        Stanford University & Yale University\\[-1mm]
        \mbox{\small\href{mailto:anaymehrotra1@gmail.com}{\texttt{anaymehrotra1@gmail.com}}} &
        \mbox{\small\href{mailto:manolis.zampetakis@yale.edu}{\texttt{manolis.zampetakis@yale.edu}}}
        \end{tabular} 
}

\begin{document}

\maketitle  
\setcounter{tocdepth}{2} 

\begin{abstract}
When can we learn from biased samples? We study regression when outcomes are observed only after passing through selection filters that depend on both covariates and outcomes themselves, a ubiquitous challenge spanning clinical trials with patient dropout, labor markets with self-selection, and auctions with strategic entry. Ignoring such selection yields systematically biased conclusions with real-world consequences.
This challenge has a long history in econometrics and statistics, starting with Heckman's seminal two-stage model \cite[Econometrica]{heckman1979sample} and followed by numerous generalizations, \eg{}, \cite[Rev. Econ. Stud.]{lee1982selectivity}, \cite[J. Econometrics]{Ahn1993semiparametric}, \cite[Rev. Econ. Stud.]{das2003nonparametric}. While these works provide various sufficient conditions for identification, a complete characterization of when such regression is possible has remained elusive.

In this work, we provide a characterization for when regression is possible in the presence of sample selection bias.
Our {results establish} the minimal assumptions required on the functional forms of selection processes under which regression remains possible, which are particularly relevant in modern settings where selection mechanisms are increasingly complex and opaque.
As a corollary of our characterization, we show that there are settings where the regression function can be identified even when the selection filter itself cannot. This observation already goes beyond the ``estimate selection filter, then debias regression'' paradigm that is followed by virtually all existing approaches. 
Under natural strengthenings of our identification conditions, we {also} establish finite-sample estimation guarantees with explicit convergence rates and provide oracle-efficient algorithms. 
This yields the first general-purpose estimation method for this broad class of selection problems.
Finally, we explore the implications of our results {for} several well-studied econometric settings with complex selection mechanisms such as auctions with entry costs and labor markets.
\blfootnote{This work is the full version of the extended abstract which is to appear in the proceedings of The Twenty-Seventh ACM Conference on Economics and Computation (EC'26).}

\end{abstract}

\setcounter{page}{1}

\newpage 
\tableofcontents
\newpage

\section{Introduction} \label{sec:intro}
\vspace{-2mm}

In virtually every scientific discipline, data pass through complex selection filters before being observed and analyzed \cite{heckman1979sample,little2019statistical,rawData2013}. 
For instance, in clinical trials we have to draw conclusions despite censoring caused by patient drop-out, in economics and social sciences we rely on data shaped by self-selection and endogenous participation, and in data coming from online platforms we have to deal with feedback loops in which users respond strategically to algorithmic decisions. All of these selection filters, if unaccounted for, lead to fallacious conclusions with consequences extending beyond scientific integrity and can have profound real-world implications. A striking illustration of this comes from the hormone replacement therapy literature, where observational studies suggested large health benefits \cite{stampfer1991estrogen} that were later overturned by evidence from randomized trials once selection into treatment was properly accounted for \cite{rossouw2002risks}.

Methods for correcting errors caused by such selection filters are nearly as old as statistical science itself: Daniel Bernoulli's celebrated smallpox memoir \cite{bernoulli1760essai} corrected life-table estimates for the competing risk of smallpox mortality. Subsequent milestones span epidemiology \cite{greenwood1915statistics}, survey sampling \cite{cochran1977sampling}, labour economics \cite{roy1951earnings}, and Heckman's Nobel prize-winning work on sample-selection models \cite{heckman1979sample}. 

Despite this rich methodological heritage, existing approaches to selection correction fall into two broad categories, each with fundamental limitations. The classical paradigm, exemplified by Heckman's selection model, proceeds by carefully specifying the probabilistic mechanism governing selection, \eg{}, the propensity to be observed, to participate, or to survive. Once articulated, this model provides a likelihood framework where the parameters of interest can be identified and estimated. This approach has proven enormously influential, but its power comes at a price. Each new application demands custom modeling of the selection filter, a process that requires deep domain expertise and often untestable assumptions about the data generation process. More fundamentally, even minor misspecification can propagate into severely biased estimates, precisely because the entire inferential apparatus rests on the accuracy of this initial modeling step.

A more recent wave of methods seeks to escape this modeling burden by learning the selection filter directly from data. The double machine-learning framework \cite{chernozhukov2018double} exemplifies this modern approach: nonparametric or machine learning techniques first estimate nuisance functions, such as propensity scores or other features of the selection mechanism, after which these estimates are plugged into a second stage designed to be robust to moderate estimation errors in the first. This methodology has gained substantial traction precisely because it promises to combine the flexibility of modern predictive methods with the rigor of formal identification strategies. Yet this promise rests critically on a condition that is far from innocuous: the availability of unbiased auxiliary data from which the selection filter can be learned. In the language of causal inference, this corresponds to assumptions such as \textit{unconfoundedness} and \textit{overlap}. In many observational settings of practical and scientific importance, this assumption fails (\eg{}, \cite{hahn2001regressionDiscontinuity,imbens2008regressionDiscontinuity,thistlethwaite1960regressionDiscontinuity,damour2021highDimensional,sommer1991nonComplianceVitaminA,hewitt2006noncompliance,wooldridge2005violatingIgnorability}). Furthermore, there are scenarios where it is not merely that double machine learning lacks the requisite clean data for its first stage—rather, even if an unbiased estimation of the selection filter was possible, the random estimation errors of the first stage can cause standard methods in the second stage, such as maximum likelihood, to completely break down, as was observed by \cite{lee2024unknown}. 

\enlargethispage{\baselineskip}

This gap raises a question that is both foundational and urgent and is our main motivation: 
\vspace{-1.5mm}
\begin{quote}
  \centering
  \emph{Can we perform principled selection-correction when we lack clean, \\ unbiased data to learn the selection filter itself?}
\end{quote}

\noindent To further motivate the research agenda defined by this question, we now discuss several application scenarios where selection functions are complicated, depend on outcomes, and are also unknown.

\begin{enumerate}[leftmargin=15pt]

   \item \textbf{Auctions with Entry Costs:}\label{ex:intro:auction}
	    Consider a platform running repeated auctions where bidders with features $x$ and latent values $y$ must incur entry costs $c$ (including bid preparation and margin requirements\footnote{In an auction, margin refers to the collateral or deposit that a participant must submit as a guarantee against potential losses or obligations arising from their participation in the auction.}) to participate. Bidders enter if $y > c$, a setting well established in econometrics \citep{samuelson1985entry,levin1994entry,gentry_li_2014} and verified empirically \citep{bajari_hortacsu_2003}. If the econometrician wishes to learn $y = f^\star(x) + \xi$ but only observes entering bidders, the data is filtered by $s^\star(x,y) = \Pr(\text{entry} \mid x, y)$. Crucially, learning $s^{\star}$ is nonstandard: we only observe $y$ when $(x, y)$ passes the filter, leaving only positive samples to learn $s$, which is a very challenging and non-standard learning problem.\\[-9pt]

\item \textbf{Extended Roy Model:}\label{ex:intro:roy}
In Roy's classical model \citep{roy1951earnings}, individuals choose among $k$ occupations and only earnings from the chosen occupation are observed. Let $y_j = f_j^\star(x) + \xi_j$ denote potential earnings in occupation $j$. If individuals choose $j^\star \in \arg\max_j y_j$, then $y_1$ is only observed if $y_1 \geq \max_{j \neq 1} y_j$. \citet{heckman_sedlacek_1985} extended the model to include unobserved occupation-specific utilities $u_j$ (capturing amenities, flexibility, match quality), where individuals select $j^\star \in \arg\max_j \{y_j + u_j\}$. This extension has been proved empirically necessary in \citep{heckman_sedlacek_1985}. The goal is to learn $f_1^\star, \ldots, f_k^\star$. In the classical Roy model, observing $(x, \arg\max_j y_j, \max_j y_j)$ reveals information about all $f_j^\star$ since non-chosen occupations satisfy $y_j \leq y_{\max}$. In contrast, the extended model reveals only $(x, j^\star, y_{j^\star})$. For learning $f_j^\star$, only samples where occupation $j$ was chosen are informative, since the constraint $y_j \leq y_i + u_i - u_j$ for $j \neq i$ involves unobserved utilities. This fits our framework: for learning $f_1^\star$, the selection function $s^\star(x,y) = \Pr(j^\star = 1 \mid x, y_1 = y)$ is outcome-dependent, unknown (depending on the joint distribution of $(\xi_2, \ldots, \xi_k, u_1, \ldots, u_k)$), and complicated when $k$ is large. The positive-only samples problem arises again.\\[-9pt]

   \item \textbf{Heckman's Model with Correlated Errors:}\label{ex:intro:heckman}
    In Heckman's model, outcomes $y = f^\star(x) + \xi$ are observed only when $d = \mathds{1}\{g^\star(x) + \xi' > 0\}$. Although the selection rule involves only $x$ and $\xi'$, the selection probability $s^\star(x,y) = \Pr(d = 1 \mid x, y)$ depends on $y$ when errors $\xi$ and $\xi'$ are correlated: conditioning on $y$ reveals information about $\xi$, which shifts the conditional distribution of $\xi'$ and hence $\Pr(d = 1)$. This correlation was central to Heckman's original model and explains why naive OLS on observed samples fails to estimate $f^\star(x)$ \citep{heckman1979sample}. The selection function $s^\star(x,y)$ is unknown in Heckman's original formulation but there is a precise model of $s^{\star}$ that helps us bypass the positive-only data issue. Subsequent work has illustrated the importance of generalizing these models of $s^{\star}$ \citep{lee1982selectivity,lee1983genecon, gallant_nychka_1987} which brings the positive-only data issue back (for more details see \cref{sec:revisiting:dvn}). \\[-9pt]

   \item \textbf{Physical Sciences:}\label{ex:intro:sciences}
    In many measurement pipelines across physical sciences, from astronomy \citep{teerikorpi1997malmquist} to proteomics \citep{webb2015review} to analytical chemistry \citep{palarea2013values}, objects are recorded only when passing detection thresholds. A canonical example is Malmquist bias in astronomical surveys, where distant objects must be intrinsically brighter to exceed detection thresholds \citep{teerikorpi1997malmquist}. Effective selection in such pipelines {is a composition of} multiple filters (detection, classification, quality cuts), each shaped by many factors (engineering constraints, resource allocation, inherited biases), leading to a selection probability $s^\star(x,y)$ that is outcome-dependent, unknown or only partially documented, and potentially complex due to filter composition. 
    This directly fits our regression with selection-bias framework and {faces all the challenges that we have discussed above.}
\end{enumerate}

  \vspace{-3mm}
\subsection{Our Contributions}
\vspace{-1mm}
We study a general model of regression with selection bias that encompasses the applications above. The outcome has the form $y = \fstar(x) + \xi$, where the regression function $\fstar$ is unknown and belongs to a concept class $\hyF$, the noise $\xi$ has distribution $\cQ^\star \in \hyQ$, and the covariates $x$ have distribution $\cD^\star \in \hyD$. Selection bias arises because a sample $(x,y)$ is observed only if it clears a selection stage that retains it with probability $\sstar(x,y) \in \hyS$; otherwise, only $x$ is observed (\cref{def:selection_bias_model}).
Each of $\cQ^\star,\cD^\star,\sstar$ is unknown and assumed to lie in its respective concept class, which may be either parametric or nonparametric.
{We let $\Theta\subseteq \hyF\times\hyS\times\hyD\times\hyQ$ denote the joint concept class of admissible parameter tuples and assume that $\theta^\star\coloneqq(\fstar,\sstar,\cD^\star,\cQ^\star)\in\Theta$.} {As the standard example, we take $\Theta=\hyF\times\hyS\times\hyD\times\hyQ$, which we call the Cartesian-product case. More generally, we allow $\Theta$ to be a non-Cartesian subset of this product, permitting restrictions that couple different components.}
\vspace{-3pt}
\vspace{-1mm}

\paragraph{Main Result (\cref{sec:results}).} Our main result is a necessary and sufficient condition (\cref{cond:identifiability:f}) on {the joint concept class $\Theta$} under which $\fstar$ can be identified from censored data. To the best of our knowledge, this is the first complete characterization of when regression is possible under sample selection bias.
Next, we show that under a natural quantitative strengthening of this condition (\cref{cond:estimation:f}), together with standard regularity and covering assumptions  (\cref{def:regularity}), $\fstar$ can also be estimated from finitely many samples (\cref{thm:estimation:f}).
Our estimator has a useful structural property: its computational bottleneck reduces to a single call to a density-estimation oracle.\footnote{A density estimation oracle is any procedure that\amedit{,} given a family of censored distributions $\hyC$ and \iid{} samples from an unknown $\cC$ in $\hyC$, returns some $\cC' \in \hyC$ close to $\cC$ in TV distance.}
This is noteworthy because sample-efficient algorithms for selection-bias problems do not, in general, admit such reductions (\eg{}, \cite{Kontonis2019EfficientTS}). 
In practice, this could allow powerful (but heuristic) density-estimation methods (\eg{}, diffusion models) to be plugged in, potentially enabling practical algorithms for regression with unknown selection bias (\cref{sec:estimation:oracle}).
\vspace{-3pt}
%

\vspace{-1mm}

\paragraph{Applications (\cref{sec:app}).} We explore the implications of our characterization by revisiting classic models and studying new ones, showing that \cref{cond:identifiability:f} and \ref{cond:estimation:f} enable identification and estimation in settings not captured by existing sufficient conditions.
\begin{enumerate}[leftmargin=15pt,itemsep=-2pt]
    \item First, we exhibit concept classes where $\fstar$ is identifiable even though $\sstar$ itself is \emph{not} identifiable (\cref{sec:app:non-iden-s}). This provides a concrete separation from the ``estimate selection, then debias'' paradigm of most prior work.
    \smallskip
    \item Next, we revisit the nonparametric model of \citet{das2003nonparametric}, among the most general existing frameworks, and show that our characterization strictly generalizes their sufficient conditions. In particular, outcome-dependent truncation mechanisms (natural in physical-science pipelines; see \cref{ex:intro:sciences}) violate their requirements yet remain identifiable under \cref{cond:identifiability:f} (\cref{sec:revisiting:dvn}).
    \smallskip
    \item Finally, as a corollary of our results, we also characterize identifiability for the extended Roy model of \citet{heckman_sedlacek_1985} and the auction-with-entry-costs model of \citet{samuelson1985entry}, for which only sufficient conditions were previously known (\cref{sec:app:otherModels}).
\end{enumerate}

\paragraph{Comparison to \citet{cai2025makestreatmenteffectsidentifiable}.}
A recent work by \citet{cai2025makestreatmenteffectsidentifiable} characterizes identifiability for average treatment effects in causal inference from observational studies, a special case of which can also be formulated as learning from selection bias. In contrast to \citet{cai2025makestreatmenteffectsidentifiable}\amedit{,} though, our work moves away from mean estimation and develops in the direction of regression problems, leading to new technical contributions and applications:
First, moving from mean estimation to regression introduces substantially richer structure, \eg{}, the densities in our framework exhibit intricate dependence on the concept classes $(\hyF, \hyS, \hyD, \hyQ)$, requiring a new perspective and ideas to address them.
Second, we provide a general finite-sample estimation method (\cref{thm:estimation:f}) with end-to-end guarantees; it is an interesting open question whether similar techniques can be applied to the treatment-effect estimation setting of \citet{cai2025makestreatmenteffectsidentifiable}.
Third, our applications exploit structure that only makes sense in regression for economic models: we show the regression function can be identified even when the selection function cannot (\cref{sec:app:non-iden-s}), and we provide identification characterizations for auction and self-selection models (\cref{sec:app:otherModels}) with no direct analogue in treatment-effect estimation. 

\subsection{Illustration of Classic Sample-Selection Models Captured by \label{sec:illustrations}Our Model }\label{subsec:model-examples}
    \label{sec:examples}
    
    \noindent There have been several models of limited dependence proposed in the literature. In this section, we highlight several important, classical models in increasing order of generality. Our identification condition can be used to verify the identification of the target parameter in each of these settings.

    \subsubsection{Heckman's Sample-Selection Model}
        We begin by considering one of the original sample selection models from \citet{heckman1979sample}. During the 1970s, economists became increasingly interested in estimating wage equations for people (\eg{}, \citep{Gronau1981Labor,Gregg1974Labor}). As an example, an economist may be interested in quantifying the linear regression coefficient between a person's wage and a set of observable features such as education level, marital status, or other salient features. In addition to wages being a function of observable features, wages are also influenced by a set of unobservable features like a person's innate motivation and ability. Both observable and unobservable features also influence a person's likelihood of participating in the workforce. 
        
        Given that wages are only known for people who participate in the workforce, Heckman, among others, noticed that if there is a correlation between the unobservable features affecting participation in the workforce and wages, then standard methods for estimating regression coefficients will lead to a biased estimate of the regression. That is to say, the samples are not at random with respect to the underlying population due to selection bias. As a result of this selection bias, \citet{heckman1979sample} showed that standard ordinary least squares methods will produce a biased estimate of the linear regression coefficients. To model this selection bias, \citet{heckman1979sample} introduced a two-equation framework (\cref{def:heckman}) comprising a linear outcome equation, whose coefficient is the target parameter, and a probit selection equation.  
        \begin{definition}[Heckman's model]\label{def:heckman} 
            Heckman's model is defined by vectors $w,v\in \R^d$ and constants $\rho\in (-1,1)$ and $\sigma,\sigma' > 0$. Define 
            \[
            y  =  w^\top x + \xi\,,\quad \xi\sim\normal{0}{\sigma^2}\,, 
            \] 
            \[ 
                d^\star  =  v^\top x + \xi'\,,\quad \xi'\sim\normal{0}{(\sigma')^2}\,,\quad d  =  \mathds{I}\{d^\star>0\}\,, 
            \] 
            where \((\xi,\xi')\) are jointly normal, \(\corr(\xi,\xi')=\rho\), and \((\xi,\xi')\perp x\).
            The researcher always observes $(x,y)$ if $d=1$ and otherwise $(x,\emptyset)$ when $d=0$.
        \end{definition}
        In this setting, $x$ is a vector of observables and selection bias arises in this model because the outcome and selection error terms, $\xi$ and $\xi'$, respectively, are jointly normal with unknown correlation. Consequently, the ordinary least squares method yields a biased estimate because \mbox{$\E[\xi | x, d=1] \neq 0$.} 
        
        To circumvent this issue, Heckman proposes estimating $\E[\xi | x, d=1]$, which he shows is possible due to the parametric assumptions on the outcome and selection equations. More precisely, 
        \[
             \E[\xi | x,d=1] = \rho\sigma\cdot \frac{\phi(v^\top x/\sigma')}{\Phi(v^\top x/\sigma')}\,.
        \]
       Once this term is computed, Heckman shows that the regression coefficient can be recovered by adding the selection correction term as an additional regressor and then performing standard least squares regression on this augmented set of {covariates}. This method has been dubbed the ``Heckman two-step'' and has become standard in labor and applied microeconometrics.

    \subsubsection{A Generalization of Heckman's Model}
    In this section, we consider a generalization of Heckman's model,
    originally proposed by \citet{lee1983genecon}. Lee generalized Heckman's model by allowing the outcome and selection noises to follow known but non-Gaussian distributions.
    \begin{definition}[Lee's Model]\label{def:lee_model}
        Lee's model is defined by vectors $w,v\in \R^d$ and outcome and selection errors $\xi$ and $\xi'$, respectively. The marginal distributions of $\xi$ and $\xi'$ are known and continuous, with distribution functions $F$ and $F'$. Define 
        \[
        y  =  w^\top x + \xi\,,\quad 
            d^\star  =  v^\top x + \xi'\,,\quad d  = \{d^\star>0\}\,, 
        \] 
        with $(\xi,\xi') \perp x$. 
    The researcher always observes $(x,y)$ if $d=1$ and otherwise $(x,\emptyset)$ when $d=0$.
    \end{definition}
    To achieve identifiability, Lee's key insight is to apply a monotone probability integral transformation to the outcome and selection errors. More precisely, Lee reparameterizes the noise by defining $\widehat{\xi} = \Phi^{-1}(F(\xi))$ and $\widehat{\xi'} = \Phi^{-1}(F(\xi'))$. Now, $\widehat{\xi}$ and $\widehat{\xi'}$ are marginally standard normals. Moreover, the original selection criterion is preserved. More precisely, selection occurs when $\xi'  > -v^\top x$. Since $\Phi^{-1}(F(\cdot))$ is a monotone transformation, selection similarly holds when $\widehat{\xi'} = \Phi^{-1}(F(\xi'))> \Phi^{-1}(F(-v^\top x))$. With this transformation applied, Heckman's two-step approach may be applied. As was the case with the polynomial Heckman model, our identification condition also holds for Lee's model. Lee's model, however, is subsumed by more general later works like \cite{das2003nonparametric}.
    \subsubsection{Ahn--Powell's Sample-Selection Model}
	       Heckman's model achieves identification by imposing  highly specific parametric structure on the selection mechanism---notably, {the selection probability $\Pr\inparen{d=1\mid x}$ is assumed to follow a probit form applied to a linear transformation of $x$}. While this strong structure yields tractability, there are settings where this assumption fails, such as when complex filters govern selection (\eg{}, Malmquist bias in astronomy \cite{teerikorpi1997malmquist}). Therefore, it is natural to ask whether we can generalize the selection mechanism to other parametric functions and perhaps even semiparametric/nonparametric functions. A large body of work has since come about that attempts to tackle precisely this; we have already discussed Lee's work \cite{lee1983genecon} from this body. There are also a number of others, including \citet{newey2003instrumental,honore1997pairwise,chernozhukov2025selection}. 
        
        One of the most significant generalized models was proposed by \citet{Ahn1993semiparametric}. In a surprising result, Ahn and Powell showed that the coefficients of a linear outcome equation can be identified {without parametrically specifying the selection propensity}, subject to mild regularity assumptions. Their key assumption is that {the conditional mean of the latent outcome error among selected observations is an unknown smooth function of the marginal propensity score $p(x)\coloneqq\Pr\inparen{d=1\mid x}$}. This assumption is commonly called an ``index restriction.'' Intuitively, this restriction {requires selection-induced biases to be ``summarized'' by this scalar propensity score}. This assumption is a common identifiability condition in the literature (\eg{}, \citet{das2003nonparametric}). 
        \begin{definition}[Ahn--Powell's Model]\label{def:ahn_powell}
            Ahn--Powell's model consists of an outcome equation and {an index restriction}. Let  $x \in \R^d$ denote the set of {covariates}, and $w \in \R^d$ denote the target parameter. {Unlike \Cref{def:heckman}, the outcome noise and selection mechanism are nonparametric}. Let ${p(x)} = \Pr(d=1|x)$ be the selection probability and let $g(\cdot)$ be an unknown smooth function of the selection probability. The outcome equation is
            \[
            y  =  w^\top x + {\xi}\,,\quad \xi\sim\cDxi\,,
            \] 
            and the {index restriction} is
            \[ 
                {\E[\xi\mid x,d=1]=g\!\inparen{p(x)}\,.}
            \]
        The outcome $(x,y)$ is observed if and only if $d=1$; otherwise, $(x,\emptyset)$ is observed when $d=0$.
        \end{definition}
        {In our framework, a tuple $\theta=(f,s,\cD,\cQ)$ induces the marginal propensity score}
        \[
            {p_\theta(x)\coloneqq\Pr_\theta(D=1\mid X=x)=\Ex_{\xi\sim\cQ}\!\insquare{s\!\inparen{x,f(x)+\xi}}\,.}
        \]
        {Thus, within our framework, the index restriction couples $f$, $s$, and $\cQ$, and the corresponding Ahn--Powell instances form a joint concept class $\Theta_{\rm AP}\subseteq\hyF\times\hyS\times\hyD\times\hyQ$.} {In particular, the marginal propensity $p_\theta(x)$ need not equal the outcome-conditional selector $s(x,y)=\Pr_\theta\!\inparen{D=1\mid X=x,Y=y}$, which may depend on $y$.}
        Rather than relying on characterizing the exact selection correction term like Heckman through joint normality, Ahn and Powell pursue a fundamentally different approach to identification that yields a semi-parametric analogue of ``Heckman's two-step'' method. 
            
        In the first step, Ahn and Powell use kernel methods to estimate the conditional mean of the selection variable given the selection {covariates}. In the second step, the outcome equation is identified using a pairwise differencing approach that eliminates the unknown selection correction term. Specifically, observations with a similar conditional selection index expectation---a term that is estimated in the first step---are paired. For these observations, if the selection correction function is smooth, the selection bias is similar across both observations and hence differencing the outcomes {approximately }eliminates the selection term and allows for identification of the linear regression coefficient. This pairwise differencing strategy is a common tool used to eliminate the selection bias (\eg{}, \citet{honore1997pairwise}).

	    Identification of the target parameter $w$ in \citet{Ahn1993semiparametric} requires the following conditions:
        \begin{enumerate}
            \item \textbf{Index Restriction:} {The selected error mean satisfies $\E[\xi\mid x,d=1]=g\!\inparen{p(x)}$ for some unknown smooth function $g$}.
            \item \textbf{Non-Singularity Condition:} {The weighted cross-moment matrix of the residualized regressors and instruments is nonsingular; see \citet[Assumption~3.4]{Ahn1993semiparametric}}.
        \end{enumerate}
        {Here, the index restriction ensures} that all the selection bias is accounted for by {the ``$g$ term''}. {The non-singularity condition} ensures there is sufficient variability in $x$ after conditioning on the selection probability for the pairwise differencing approach to succeed.

    \subsubsection{Censoring on Outcomes: Tobit Model and Variants}\label{ex:outcome-only}
    Models where revealment depends only on the realized outcome date back to \citet{tobin1958LimitedDependent} (Tobit), with comprehensive econometric treatments (\eg{}, \citet{amemiya1981survey}) and recent learning-theoretic analyses for truncated regression \citep{daskalakis2019computationally,plevrakis2021censored}. Deterministic truncation takes
    $s^\star(x,y)=\mathds{1}\!\insquare{y\in[a,b]}$, while probit/logit censoring uses $s^\star(x,y)=\Phi\!\inparen{a_0+a_1 y}$ or $s^\star(x,y)=\sigma\!\inparen{a_0+a_1 y}$.
    More generally, given a survival set $S^\star$ within which the outcomes are observed, we obtain the selection function $\sstar(x,y)\coloneqq \mathds{1}\!\inbrace{y\in S^\star}$.
    The Tobit censored regression model, for instance, uses threshold functions $s(x,y) = \mathds{1}\inbrace{y \geq Y_L}$ or $s(x,y) = \mathds{1}\inbrace{y \leq Y_G}$, depending on whether the data is left- or right-censored (or both). While usually the ``Tobit'' model refers to this censored regression model, \citet{amemiya1985advanced} in his book on Advanced Econometrics categorizes variations of this model into several categories which also capture Heckman's model (type II and III). 
    %

   \subsection{Related Work} \label{sec:related-work} \label{sec:further-related-work}
    As we discussed, regression with sample selection bias is a general problem that arises across many scientific disciplines. Here we do not attempt a comprehensive survey. Instead, we briefly discuss several related settings that are most closely connected to our formulation and highlight some interesting directions for future work. For broader background on sample selection and correction methods, we refer the reader to survey treatments such as \citet{vella1998survey,puhani2000heckmanCritique}, and \citet{heckman1990varieties}, as well as textbook discussions in, \eg{}, \citet{maddala1983limited}. 
    For a complementary perspective from the missing-data literature, see \citet{little2019statistical}. 
    We include more detailed descriptions of key models at the points in the paper where they are directly used for comparison.

\paragraph{Heckman's Model and Follow-up Work.}
    \citet{heckman1979sample} introduced the canonical regression-with-selection model and a two-step correction procedure under strong parametric assumptions on the selection equation and the joint noise distribution. A large follow-up literature relaxes these assumptions in different directions (\eg{}, \citet{lee1982selectivity,Ahn1993semiparametric,das2003nonparametric}; see \cref{sec:illustrations} and the survey by \citet{vella1998survey}). 
    Despite this long line of work, a characterization of when the regression function is identifiable was not known. 
	    Our work fills this gap with a necessary and sufficient condition (\cref{cond:identifiability:f}) that we show meaningfully extends beyond existing sufficient conditions (\cref{sec:app}).
    Separately, recent work studies problems different from regression under selection bias, such as online learning  \citep{singhvi2025online} and multi-armed bandits \citep{shin2021bandit}. 
    It would be interesting to see if our techniques have implications for these settings.

\paragraph{Regression with Censored and Truncated Data.}
    Regression with censored outcomes is a special case of our model in which the covariates are always observed and the selector depends only on the outcome (often through a known censoring rule, such as a detection limit); this setting has a long history in statistics (see, \eg{}, the monograph of \citet{Cohen91}). In our framework, such models correspond to selectors of the form $s(x,y)=s(y)$. 
    Our characterization applies to this setting and, more generally, determines precisely when the regression function is identifiable even when the censoring rule is unknown or more complex (\eg{}, it depends jointly on $(x,y)$).
    A closely related but harder problem is \emph{truncation}, where non-selected samples are entirely unobserved (so when $D=0$ one observes neither $X$ nor $Y$). This problem has recently received significant attention in theoretical computer science, where the focus is on developing computationally efficient algorithms for truncated estimation under various structural assumptions \citep{daskalakis2018efficient,daskalakis2019computationally,daskalakis2020truncated,ilyas2020theoretical,lee2024unknown,lee2026smooth,kouridakis2026regression}. In most of this literature, identifiability is straightforward because the selection rule is known or depends only on $y$; accordingly, the emphasis is on sample complexity and computational efficiency rather than on identification. When selection depends on both $x$ and $y$, however, identifiability itself becomes non-trivial. Our results provide a necessary and sufficient condition for identifiability in this more general setting, and it is an interesting direction to combine our characterization with recent algorithmic techniques to obtain computationally efficient estimators.

\paragraph{Identification across Econometrics, Computer Science, and Statistics.}
    More generally, our characterization contributes to the broader identification literature that asks when a target quantity of interest can be recovered from partially observed data. In econometrics, closely related questions arise in auction models, where one seeks to identify value distributions or demand primitives from bids and limited observables \citep{athey2002identification,cherapanamjeri2022auctionEstimation}, and in self-selection models related to \citet{roy1951earnings}, where selection into sectors distorts observed outcome distributions \citep{manski1990nonparametric,lee_park_2023,cherapanamjeri2023selfselection}. In computer science, analogous partial-observability phenomena appear in language identification from positive data, where only ``accepted'' examples are observed (and no explicit negatives), creating a selection effect that makes the problem challenging \citep{gold1967language,angluin1980inductive,charikar2024facets,kleinberg2025density,kalavasis2024characterizations}. 
    In statistics, foundational work on identifiability in mixture and latent-variable models similarly studies when different parameterizations can generate the same observed distribution \citep{teicher1963identifiability,everitt2013finite}. Our contribution to this literature is a necessary-and-sufficient identification condition for regression under outcome-dependent sample selection. 

\section{Preliminaries }
    \label{sec:preliminaries}

    In this section, we introduce basic notation and conventions used throughout the paper.
    
    \vspace{0.5mm}
    
    \indent \emph{Distributions.} We use calligraphic letters (\eg{}, $\cD$ and $\cP$) to denote distributions. 
    When a distribution is continuous (respectively, discrete), we assume it is specified by a density (respectively, a probability mass function).
    We write $\cD(x)$ for the density or mass of $\cD$ at $x$. 
    Unless otherwise specified, \(\nu\) denotes the Lebesgue measure on \(\R\), and densities over \(\R\) are taken with respect to \(\nu\). We write a.s. to mean almost surely with respect to the underlying probability measure. For a distribution $\cD$, we write $\cD$-a.e. to mean almost everywhere with respect to $\cD$. When the underlying measure is clear from context, we simply write a.e. 
    Given a set $S$, we overload notation, and use $\cD(S)$ for the mass that $\cD$ assigns to $S$.
    The support of a distribution $\cD$ is denoted by $\supp(\cD)$ (\ie{}, $\supp(\cD)\coloneqq \inbrace{z\colon \cD(z)>0}$).
    We use $\normal{\mu}{\Sigma}$ to denote the Gaussian distribution with mean $\mu \in \R^d$ and covariance $\Sigma \in \R^{d \times d}$. Finally, we use $\tv{\cP}{\cQ}$ to denote the total variation distance between $\cP$ and $\cQ$, \ie{}, $\tv{\cP}{\cQ}=\frac{1}{2}\int_z \abs{\cP(z)-\cQ(z)}\d z$.

    \vspace{0.5mm}

    \indent \emph{Random Variables.}
    We use capital letters to denote random variables and reserve $\xi, \xi'$ (and variants) to denote noise variables (also called disturbances or shocks). 
    For an event $E$, we denote its indicator random variable by $\mathds{1}\{E\}$.
    We write $X \bot Y$ to denote that $X$ and $Y$ are independent.

    \vspace{0.5mm}
    
    \indent \emph{Concept Classes.}
    We use script letters (\eg{}, $\hyD, \hyS,$ and $\hyQ$) to denote concept classes, \ie{}, sets of objects such as distributions or functions. Concept classes are widely used in learning theory to model families of distributions or functions; for example, Gaussian mixture models (GMMs) are a common concept class of study that richly captures many empirical settings. On the functional side, generalized linear models (GLMs) are a common example. In general, concept classes can also be much more complicated, \eg{}, all distributions representable by diffusion models of a specific architecture and regression functions representable by, \eg{}, two-layer neural networks. Finally, we use a superscript $\star$ to denote the ground-truth distribution or function generating the data; for example, $\fstar$ and $\sstar$ denote the true regression and selection functions, respectively.  

    \vspace{0.5mm}
    
    \indent \emph{Vectors and Matrices.}
    {For a vector $v \in \R^d$, we use $\|v\|_2$ to denote its $\ell_2$-norm and $\|v\|_\infty$ its $\ell_\infty$-norm. For vectors $v,w \in \R^d$, we write $\inangle{v,w}$ or $v^\top w$ for their inner product $\sum_i v_i w_i$.}

\section{Model of Regression with Sample-Selection Bias }  \label{sec:model}
 
In this section, we formally present the model of regression with sample-selection bias. 
Let $\cX$ be the feature or covariate space; typically, $\cX=\R^d$.
We begin with the standard (non-linear) regression model:
\[
    Y = \fstar(X) + \xi\,.
    \yesnum\label{eq:outcome-model}
\]
Here $\fstar$ is an unknown function that is only known to belong to a set or concept class $\hyF$ and $\xi$ is a random variable corresponding to noise. 
$\xi$ has distribution $\cDxi$ that is possibly unknown.
For instance, if $\hyF$ is the class of linear functions, \ie{}, 
\[
    \hyF = \hyFlin \coloneqq \inbrace{x\to w^\top x\colon w\in \R^d}\,,
    \yesnum\label{eq:def:linear-functions}
\]
and $\cDxi$ is the standard normal $\normal{0}{1}$, then we recover the standard linear regression model.

In the standard regression model (without sample-selection bias), we see samples of the form $(x,y)$ where $x$ is drawn from some distribution $\cDX$ and $y$ is as specified in \cref{eq:outcome-model}. 
The goal is to estimate the function $\fstar$ using the samples.
Due to selection bias, however, the outcomes $y$ in some of the samples are hidden: in particular, for the sample $(x,y)$ the outcome $y$ is observed with probability $\sstar(x,y)$ and the outcome is otherwise hidden and only the covariate $x$ is observed.
(Formally, nature draws a random variable $D\in \zo$ {with $\Pr[D=1]=\sstar(x,y)$ and $y$ is observed if $D=1$.)}

\paragraph{Family of Selection Functions.}
Initial works on regression with sample selection bias assumed that the \textit{selection-function} $\sstar$ is known or has a known parametric structure.
For instance, \citet{heckman1979sample} studied the standard linear regression (\ie{}, $\hyF=\hyFlin$ and $\cDxi=\normal{0}{1}$) when the selection-function $\sstar$ is known to be roughly of the form $\sstar(x,y) = \Phi{\inparen{x^\top v + \rho\inparen{y-\inangle{w,x}}} }$ where $\Phi(\cdot)$ is the cumulative distribution function of the normal distribution and $(\rho,v,w)$ are unknown parameters (see \cref{def:heckman} for the specific form of $\sstar$).
This was later generalized by \citet{lee1982selectivity} to selection functions of the form $\sstar(x,y) = J{\inparen{x^\top v + \rho\inparen{y-\inangle{w,x}}} }$ for a known, invertible transformation function $J(\cdot)$ that can be different from $\Phi(\cdot)$.
These modeling choices and many others on the selection function can be cast in the language of concept classes from statistical learning theory (see \cref{sec:examples}):
we assume that there is a \textit{known} set or concept class $\hyS$ containing $\sstar$, \ie{},
\[
    \phantom{.}\qquad\qquad
    \sstar\in \hyS\,.\hspace{-15mm} 
    \tag{Concept Class for Selection Function}
\]
We stress that this formulation is quite general and allows the selection function to have an unknown parametric form.
For instance, $\hyS$ can be the class of all functions corresponding to any $k$-layered neural networks with sigmoid activation for some value of $k$, or the class of functions corresponding to a decision tree of a specific depth.
It also captures the aforementioned models of \citet{heckman1979sample} and \citet{lee1982selectivity} and extends beyond them as we illustrate in \cref{sec:examples}. 

\paragraph{Family of Distributions.}
We allow similar flexibility on the distributions $(\cDX,\cDxi)$ of features and noise\amedit{,} respectively: there are \textit{known} concept classes \mbox{$(\hyX,\hyQ)$ containing $(\cDX, \cDxi)$\amedit{,} respectively, \ie{},}
\[
    \phantom{.}\qquad\qquad
    \cDX\in \hyDX
    \quadand
    \cDxi\in \hyDxi \,.\hspace{-15mm}
    \tag{Concept Classes for Distributions}
\]
As with concept classes of functions (\eg{}, $\hyS$), concept classes of distributions can also capture many interesting families. For instance, such a concept class can be the family of Gaussian mixture models (GMMs) or the family of all distributions representable by a specific diffusion model family.
It is perhaps clear that versatility in the modeling of the feature distribution is relevant. 
Versatility in the modeling of the distribution of noise, extending beyond the standard normal distribution\amedit{,} is also important:

For instance, letting $\hyQ=\sinbrace{\normal{0}{\sigma^2}\colon \sigma>0}$ allows us to model noise of unknown variance.
Further, allowing $\hyQ$ to include heavy-tailed distributions is essential in several settings (\eg{}, financial econometrics) \citep{mandelbrot1963speculative,pickland1976pareto,luttmer2007pareto}, and allowing distributions supported on $\mathbb{R}_{\geq 0}$ is natural for many scenarios such as for distributions of inefficiency in stochastic frontier models \citep{aigner1977stochasticFrontier} and for nonnegative quantities like utilization or durations. %
\begin{remark}[Heteroskedasticity]
    While the above framework is general, it does not capture \textit{heteroscedasticity,} since the distribution $\cDxi$ is independent of the feature $x$.
    The framework and our proofs can extend to heteroscedastic settings, which would lead to analogous but more complex characterizations. For simplicity, we only study the homoscedastic case.
\end{remark}

\subsubsection*{Formal Model and Main Questions}
We are now ready to state our formal model: 

\begin{definition}[Sample-Selection Model]\label{def:selection_bias_model}
    Fix a measurable feature space \(\cX\), concept classes $(\hyF,\hyS,\hyDX,\hyQ)$ and $\fstar\in \hyF$, $\sstar\in \hyS$, $\cDX\in \hyDX$, and $\cDxi\in \hyQ$. Assume that every \(f\in\hyF\) is measurable and every \(s\in\hyS\) is jointly measurable with range in \([0,1]\).
    Each sample of the limited-dependence model defined by $(\fstar,\sstar,\cDX,\cDxi)$ is generated as follows:
    \begin{enumerate}
      \item First, nature independently draws $X\sim\cDX$ and $\xi\sim \cDxi$ and fixes the outcome $Y\coloneqq f^\star(X)+\xi$. 
      \item Then, nature draws $U\sim\mathrm{Unif}\inparen{0,1}$ independent of $(X,\xi)$, and sets $D\coloneqq \mathds{1}\!\inbrace{U  \leq s^\star\!\inparen{X,Y}}$.
      \item If $D=1$, nature reveals $(X,Y)$ to the econometrician; otherwise, it reveals only $(X,\emptyset)$.\footnote{Here, $\emptyset$ is a symbol not in $\R$ indicating that the outcome is unobserved.}
    \end{enumerate}
    Let $\cC=\cC_{\fstar,\sstar,\cDX,\cDxi}$ be the (censored) distribution of samples generated from the above model.
\end{definition}
{Having specified the possible values of the individual components, it remains to specify which parameter tuples are admissible.}
{In the standard case, any $f\in\hyF$, $s\in\hyS$, $\cD\in\hyDX$, and $\cQ\in\hyQ$ may be combined, so the parameter class is $\hyF\times\hyS\times\hyDX\times\hyQ$.}
{More generally, we let $\Theta$ be a subset of this product specifying the admissible parameter tuples.}
{This formulation captures, for example, the Ahn--Powell instances described in \cref{def:ahn_powell}, whose index restriction jointly constrains $f$, $s$, and $\cQ$.}
The following notion of realizability will be important in our discussion.
\begin{definition}[Realizability]
    A distribution $\cC$ is said to be \underline{realizable} with respect to the joint concept class $\Theta$ if it is identical to the data-distribution arising from the sample-selection model specified by some $(\fstar,\sstar,\cDX,\cQ)\in \Theta$. In particular, in the Cartesian-product case $\Theta=\hyF\times\hyS\times\hyDX\times\hyQ$, this is equivalent to requiring each component to belong to its respective concept class.
\end{definition}
The idea of generalizing Heckman's models is not new:
several authors (\eg{}, \cite{lee1982selectivity,powell1987semiparametric,Powell1986,das2003nonparametric,newey2009series}) have considered models with fewer restrictions on the selection function $\sstar$ and noise distribution $\hyQ$ than in Heckman's model.
In the language of the above learning-theoretic framing, these works provide (several different)  conditions on the joint concept class $\Theta$ that are \textit{sufficient} to identify $\fstar$ from a censored distribution $\cC$ realizable with respect to $\Theta$. 
But they leave the following fundamental question open:
\begin{center}
    \textit{Are there interesting joint concept classes $\Theta$ that enable one to ``learn'' $\fstar\in \hyF$ from 
    a censored distribution $\cC$ realizable with respect to $\Theta$, but that are not captured by existing sufficient conditions?}
\end{center}
If we only strengthen existing sufficient conditions on $\Theta$, then we leave open the question of whether they can be strengthened further. 
Hence, to definitively settle the above question, we need to find necessary and sufficient conditions that enable an econometrician to learn $\fstar$ from censored data. 
To formalize this, we need to define what learning \mbox{$\fstar$ means; we consider two notions.}
\begin{problem}[Identification and Estimation]\label{problem:identification_estimation}
    Consider a censored distribution $\cC$ realizable with respect to $\Theta$.
    Let $(f,s,\cDX,\cQ)\in\Theta$ witness $\cC$'s realizability; to be explicit about this, we write $\cC=\cC_{f,s,\cDX,\cQ}$.
    Given $\Theta$ and access to \iid{} samples $(X,Y)\sim \cC$, the question is:
    \begin{enumerate}[leftmargin=15pt]
        \item \textbf{(Identification)}: What are  the minimal assumptions on $\Theta$ so that there is a deterministic mapping $\psi$ satisfying $\psi(\cC_{f,s,\cDX,\cQ}) = f$ (where the equality holds $\cDX$-a.e.) for every $(f,s,\cDX,\cQ)\in\Theta$?
        \item \textbf{(Estimation)}: What are the minimal assumptions on $\Theta$
        under which there is an algorithm that, given $n\geq 1$ \iid{} samples $S$ from $\cC=\cC_{f,s,\cDX,\cQ}$, outputs an estimate ${f}_{S}$ such that, uniformly over every $(f,s,\cDX,\cQ)\in\Theta$, the estimator error satisfies $\Ex_{S\sim \cC^n}\Ex_{X\sim \cDX}\sabs{{f}_S(X) - f(X)}\leq \eps(n)$ for some error rate $\eps(\cdot)$ satisfying $\lim_{n\to \infty}\eps(n)=0$?
    \end{enumerate}
\end{problem}
Some remarks are in order.
First, in general, $\fstar$ cannot always be identified from the censored distribution $\cC$ because there exist $(f,s,\cD,\cQ),(f',s',\cD',\cQ')\in\Theta$ with $f$ and $f'$ differing on a set of positive $\cD$-measure such that the censored distributions $\cC$ and $\cC'$ arising from $(f,s,\cD,\cQ)$ and $(f',s',\cD',\cQ')$ are identical.
 Next, we note that estimation is a harder problem than identification. 
{Indeed, suppose the same censored distribution \(\cC\) is generated by two admissible tuples with regressors \(f\) and \(f'\). Since \(X\) is always observed, the two tuples have the same covariate distribution \(\cDX\). For any estimator satisfying the estimation guarantee and every \(n\geq 1\), the triangle inequality gives
\[
    \Ex_{X\sim\cDX}\abs{f(X)-f'(X)}
    \leq
    \Ex_{S\sim\cC^n}\Ex_{X\sim\cDX}
    \left[\abs{f(X)-f_S(X)}+\abs{f_S(X)-f'(X)}\right]
    \leq 2\eps(n)\,.
\]
Letting \(n\to\infty\) shows that \(f=f'\), \(\cDX\)-a.e. Thus, estimation implies identification.}
Further, there are simple examples demonstrating that \mbox{estimation is strictly harder than identification}.{\footnote{For example, fix a covariate distribution, let \(\hyF=\{f_0,f_1\}\) with \(f_j(x)\equiv j\), let \(\hyQ=\{\normal{0}{1}\}\), and let \(\hyS=\{s_p:p\in(0,1]\}\), where \(s_p(x,y)\equiv p\). For every \(p>0\), the censored distribution identifies \(p=\Pr(D=1)\) and the conditional law \(Y\mid D=1\sim\normal{j}{1}\), and hence identifies \(f_j\). Nevertheless, uniform estimation over this class is impossible. Conditional on all \(n\) outcomes being hidden---an event of probability \((1-p)^n\) under either regressor---the observed samples have the same distribution. Hence, for any estimator, one of the two regressors has expected error at least \((1-p)^n/2\). Letting \(p\downarrow0\) for each fixed \(n\) shows that no uniform error rate can converge to zero.}}
%
\begin{remark}[Parameterization of Selection Function]
     We let the selection function $\sstar$ be a function of $x$ and $y$.
     Alternatively, we can let $\sstar$ be a function of $x$ and the noise $\xi$. 
     The two formulations are equivalent since \cref{eq:outcome-model} is a \mbox{deterministic relation connecting $y$ to $\xi$, for a specific value of $x$.}
\end{remark}

\section{Our Results }
\label{sec:identification}
\label{sec:results}
    In this section, we present our main results.
    In \cref{sec:identification:characterization}, we characterize when the joint concept class $\Theta$ enables identification of $\fstar$.
    Then, in \cref{sec:estimation}, we show that under a natural strengthening of our characterizing condition, $\fstar$ can be estimated from finite samples.

\subsection{Characterization of the Identifiability of $\fstar$}
    \label{sec:identification:characterization}
    We first present a condition on the joint concept class $\Theta$ that would be useful for understanding when regression is possible with sample selection bias.
    \vspace{-4mm}
    \begin{mdframed}[nobreak=true]
    \vspace{-3mm}{
        \begin{condition}[Identifiability condition]\label{cond:identifiability:f}
            {Recall that \(\nu\) denotes the Lebesgue measure on \(\R\), and suppose the noise component of every tuple in $\Theta$ admits a density with respect to \(\nu\).} For every \(\cDX\in\hyDX\), every \(f,f'\in\hyF\) such that
            \[
                \cDX(\{x\in\cX:f(x)\neq f'(x)\})>0\,,
            \]
            every \(s,s'\in\hyS\), and every \(\cQ,\cQ'\in\hyQ\) for which $(f,s,\cDX,\cQ),(f',s',\cDX,\cQ')\in\Theta$, the corresponding observed kernels
            \[
                (x,y) \mapsto s(x,y)\cdot \cQ(y-f(x))
            \quadand
                (x,y) \mapsto s'(x,y)\cdot \cQ'(y-f'(x))
            \]
            are not equal \((\cDX\otimes \nu)\)-a.e. Equivalently, the two must differ on a set of positive $(\cDX\otimes\nu)$-measure.
        \end{condition}
        }
    \end{mdframed}
    To gain some intuition about \cref{cond:identifiability:f}, consider two instances of the problem corresponding to the parameter tuples $\theta=(f,s,\cD,\cQ)$ and $\theta'=(f',s',\cD',\cQ')$ in $\Theta$.
    Assume that the data is generated by a tuple $(\fstar,\sstar,\cDX,\cDxi)$ which corresponds to one of these two tuples.
    Given the censored distribution $\cC^\star$, we want to identify $\fstar$, {up to \(\cDX\)-a.e. equivalence.}
    {If \(f=f'\) \(\cDX\)-a.e., then the two candidate regressors are indistinguishable for our purposes, and identifying either one identifies \(\fstar\) up to \(\cDX\)-a.e. equivalence. 
    Thus, the substantive case is when \(f\) and \(f'\) differ on a set of positive \(\cDX\)-measure.}
    Now, suppose that $\cD\neq \cD'$ {as probability measures}.
    Since there is no censoring on $x$,\footnote{Note that the censored distribution $\cC^\star$ has \iid{} samples $x\sim \cDX$; see \cref{def:selection_bias_model}.} we can identify $\cDX$ from the data.
    {
    Therefore, if \(\cD\neq \cD'\), we can distinguish the two tuples by checking which one has \(X\)-marginal equal to \(\cDX\). 
    }
    For the sake of proceeding with the example, {suppose instead that \(\cD=\cD'\) as probability measures.}
    {When \(\cD=\cD'\), the restrictions of the two censored laws to uncensored outcomes are finite subprobability measures on \(\cX\times\R\). They are absolutely continuous with respect to \(\cD\otimes\nu\), with Radon--Nikodym densities \(s(x,y)\cQ(y-f(x))\) and \(s'(x,y)\cQ'(y-f'(x))\), respectively. Thus, \cref{cond:identifiability:f} rules out equality of the censored laws whenever \(f\) and \(f'\) differ on a set of positive \(\cD\)-measure.}

    If \cref{cond:identifiability:f} holds, then one can identify $\fstar$ from the censored data by following {the identification strategy sketched above}.
    Thus, \cref{cond:identifiability:f} opens a pathway for identification of $\fstar$---it is a sufficient condition for identification.
    Our first result shows that \cref{cond:identifiability:f} captures all methods for identifying $\fstar$---it is also a necessary condition.
    \begin{theorem}[Characterization]
    \label{thm:characterization:f}
        Fix a joint concept class $\Theta\subseteq\hyF\times\hyS\times\hyDX\times\hyQ$. Suppose that the noise component of every tuple in $\Theta$ has a density {with respect to \(\nu\)}.
        Then, $\Theta$ enables identification of $\fstar$ {(up to \(\cDX\)-a.e. equivalence)} from censored data (\cref{problem:identification_estimation}) if and only if $\Theta$ satisfies \cref{cond:identifiability:f}.
    \end{theorem}
    To the best of our knowledge, {this} is the first characterization of when regression is possible under sample selection{bias}.
    This characterization adds to the classic, and still growing, list of identifiability results in {econometrics} (\eg{}, \cite{manski1990nonparametric,athey2002identification,gentry_li_2014,cherapanamjeri2022auctionEstimation,lee_park_2023}), {computer science} (\eg{}, \citep{angluin1980inductive,cherapanamjeri2023selfselection,charikar2024facets,kleinberg2025density,kalavasis2024characterizations}), and {statistics} (\eg{}, \citet{teicher1963identifiability,everitt2013finite}).
    
    For \cref{cond:identifiability:f} to be useful, given all the existing sufficient conditions, it needs to capture interesting examples which are not captured by existing conditions or their natural generalizations.
    {In \cref{sec:app}, we revisit several classic scenarios of sample selection bias and show that \cref{cond:identifiability:f} enables identification in more general cases.}
    {For instance, there are scenarios where $\fstar$ is identifiable even though $\sstar$ is not, so existing methods (most of which first identify $\sstar$) necessarily fail.}
    {We also explore two related models: (1) auctions with selective entry \citep{samuelson1985entry}, where bidders pay a cost before entering, and (2) an extension of \citet{roy1951earnings}'s model due to \citet{heckman_sedlacek_1985} to better capture US wage distributions.}
    {Interestingly, despite extensive study of these models (\eg{}, \citep{mcafee1987entry,levin1994entry,gentry_li_2014,gentry2017selective,lee_park_2023}), a characterization of identification in them was not known; as a corollary of \cref{thm:characterization:f}, we provide one (\cref{sec:app:otherModels}).}

    In the remainder of this section, we prove \cref{thm:characterization:f}, and in the next section, we show how to obtain {finite-sample} complexity guarantees for estimating $\fstar$ under a natural strengthening of \cref{cond:identifiability:f}.
    {Finally, we focus on the identifiability of $\fstar$. In some contexts, identifying $\sstar$ is of independent interest (beyond using it to identify $\fstar$), and here we can use natural variations of our techniques to characterize identifiability of $\sstar$; we show this in \cref{sec:characterization:s}.}

    \begin{remark}[Comparison to Existing Conditions]
        Since \cref{cond:identifiability:f} characterizes when $\fstar$ is identifiable, it must be weaker than any sufficient conditions proposed in prior work. As a sanity check, we verify this in \cref{sec:examples}. Further, the characterization enables more: after verifying \cref{cond:identifiability:f} in earlier models, we can immediately identify natural extensions that preserve identifiability, giving a more principled understanding of why these models are identifiable. {See \cref{sec:app} for further examples.}
    \end{remark}
    \vspace{-7mm}
\begin{proof}[Proof of \cref{thm:characterization:f}]
For any quadruple $(f,s,\cD,\cQ)\in\Theta$, let $\cC_{f,s,\cD,\cQ}$ denote the distribution of {the }{observed }sample $(X,{\widetilde Y})$ generated by the {sample-selection} model (\cref{def:selection_bias_model}) specified by $(f,s,\cD,\cQ)$.
Here, $D=\mathds{1}\!\inbrace{{\widetilde Y\neq\emptyset}}$;
that is, if ${\widetilde Y}=\emptyset$, then $D=0$ and, otherwise, if ${\widetilde Y}\neq \emptyset$, then $D=1$.
Let $\cQ$ also denote the density of the noise distribution $\cQ$ in this tuple (which exists by assumption). 
Conditional on $X=x$, the joint density of the observable event $\{D=1, Y\in \d y\}$ is
\[
  s(x,y)\,\cQ(y-f(x))\,.
  \yesnum\label{eq:kernel}
\]
{For every measurable \(A\subseteq\cX\) and \(B\subseteq\R\),
\begin{align*}
    \cC_{f,s,\cD,\cQ}(A\times B)
    &=
    \int_A\int_B s(x,y)\cQ(y-f(x))\,\nu(\d y)\,\cD(\d x)\,,\\
    \cC_{f,s,\cD,\cQ}(A\times\{\emptyset\})
    &=
    \int_A\left(1-\int_\R s(x,y)\cQ(y-f(x))\,\nu(\d y)\right)
    \cD(\d x)\,.
\end{align*}}
Consider any censored distribution $\cC$ realizable with respect to $\Theta$.
Let $(f,s,\cD,\cQ)\in\Theta$ witness the realizability of $\cC$.
To be explicit about this, we write $\cC=\cC_{f,s,\cD,\cQ}$.

Now, we are ready to prove the result.

\paragraph{($\Leftarrow$) \cref{cond:identifiability:f} implies identifiability.}
    Toward a contradiction, assume $\Theta$ satisfies \cref{cond:identifiability:f} but identifiability fails.
So, there is no mapping $\psi$ from distributions realizable with respect to $\Theta$ to $\hyF$ such that $\psi(\cC_{f,s,\cD,\cQ})=f$ (up to $\cD$-a.e. equivalence) for every $(f,s,\cD,\cQ)\in\Theta$.
    Therefore, there must be $(f,s,\cD,\cQ)\in\Theta$ and $(f',s',\cD',\cQ')\in\Theta$ such that
    \[
        \cC_{f,s,\cD,\cQ} = \cC_{f',s',\cD',\cQ'}
        \quadand
        \cD\!\left(\{x\in\cX:f(x)\neq f'(x)\}\right)>0\,.
    \]
    Equality of the two censored distributions implies equality of their \(X\)-marginals, and hence $\cD=\cD'$.
    Henceforth, we omit the dependence on $\cD=\cD'$ unless it is important.
    The restrictions of the two censored laws to \(\cX\times\R\) are identical.

    By \eqref{eq:kernel}, these restrictions have Radon--Nikodym densities \(s(x,y)\cQ(y-f(x))\) and \(s'(x,y)\cQ'(y-f'(x))\) with respect to the common measure \(\cD\otimes\nu\). Uniqueness of Radon--Nikodym derivatives gives
    \[
        s(x,y)\cQ(y-f(x))
        =
        s'(x,y)\cQ'(y-f'(x))
        \qquad (\cD\otimes\nu)\text{-a.e.}
        \yesnum\label{eq:kernel-equality-f}
    \]
    If \(f\) and \(f'\) are not equal \(\cD\)-a.e., then \cref{cond:identifiability:f} applied to the two admissible parameter tuples above, whose covariate distributions are equal, implies that the two sides of \eqref{eq:kernel-equality-f} differ on a set of positive \((\cD\otimes\nu)\)-measure. This contradicts \eqref{eq:kernel-equality-f}. 

\paragraph{($\Rightarrow$) Identifiability requires \cref{cond:identifiability:f}.}
    Suppose toward a contradiction that \cref{cond:identifiability:f} fails but identifiability holds.
    Since \cref{cond:identifiability:f} fails, 
    
    there exist \(f,f'\in\hyF\), \(s,s'\in\hyS\), \(\cQ,\cQ'\in\hyQ\), and \(\cD\in\hyDX\) such that $(f,s,\cD,\cQ),(f',s',\cD,\cQ')\in\Theta$ and $\cD\!\left(\{x\in\cX:f(x)\neq f'(x)\}\right)>0$ and \[ s(x,y)\,\cQ(y-f(x)) = s'\!(x,y)\,\cQ'\!(y-f'(x)) \qquad (\cD\otimes\nu)\text{-a.e.} \yesnum\label{eq:kernel-equality-f-global} \]
    
    We claim \eqref{eq:kernel-equality-f-global} implies $\cC_{f,s,\cD,\cQ}=\cC_{f',s',\cD,\cQ'}$.

    Once the claim is proved, it contradicts identifiability, since the two regressors differ on a set of positive \(\cD\)-measure. %

    It remains to prove $\cC_{f,s,\cD,\cQ}=\cC_{f',s',\cD,\cQ'}$.
    To simplify the notation, define
    \[
        \cC\coloneqq \cC_{f,s,\cD,\cQ}
        \quadand
        \cC'\coloneqq \cC_{f',s',\cD,\cQ'}\,.
    \]
We compare their \(X\)-marginals and conditional laws given \(X\).

    We have the following results:
    \begin{enumerate}[itemsep=2pt]
        \item First, since $X$ is always observed, both censored distributions have the same \(X\)-marginal
        \[
            \cC\!\inparen{
                A\times(\R\cup\{\emptyset\})
            }  = \cD\!\inparen{A}
            =\cC'\!\inparen{
                A\times(\R\cup\{\emptyset\})
	            }\,\qquad\text{for every measurable }A\subseteq\cX\,.
        \]
        \item %

        Second, consider the sub-probability kernel
        \[
            A \mapsto \cC(D=1,\widetilde Y\in A\mid X=x)
        \]
        on \(\R\). Under \(\cC\), this kernel has density
        \(s(x,y)\cQ(y-f(x))\) with respect to \(\nu\), while under \(\cC'\) it has
        density \(s'\!(x,y)\cQ'\!(y-f'(x))\). By
        \eqref{eq:kernel-equality-f-global}, these densities agree \(\nu\)-a.e. for
        \(\cD\)-a.e. \(x\), so the corresponding sub-probability kernels agree for
        \(\cD\)-a.e. \(x\).

        \item 
        Third, both distributions put zero mass on events of the form \(\{D=0,\widetilde Y\in B\}\) for measurable \(B\subseteq\R\), since \(D=0\) implies \(\widetilde Y=\emptyset\).
        
        \item %

        Finally, it remains to compare the censoring probabilities. For \(\cD\)-a.e. \(x\),
        \begin{align*}
            \cC(D=0,\widetilde Y=\emptyset\mid X=x)
	            &= 1-\int s(x,y)\cQ(y-f(x))\,\nu(\d y)\,,\\
            \cC'(D=0,\widetilde Y=\emptyset\mid X=x)
	            &= 1-\int s'\!(x,y)\cQ'\!(y-f'(x))\,\nu(\d y)\,.
        \end{align*}
        By \eqref{eq:kernel-equality-f-global}, the integrands are equal \(\nu\)-a.e. for \(\cD\)-a.e. \(x\), so these censoring probabilities are equal for \(\cD\)-a.e. \(x\). Thus, the \(X\)-marginals agree and the conditional laws of \((D,\widetilde Y)\) given \(X=x\) agree for \(\cD\)-a.e. \(x\).

    \end{enumerate}
    Hence, $\cC=\cC'$.
    This contradicts identifiability up to \(\cD\)-a.e. equivalence and proves the result.

\end{proof}
\vspace{-3mm}

\subsection{Estimation of $\fstar$ from Finite Samples}
    \label{sec:estimation}
    In this section, we explore when $\fstar$ can be estimated from a finite number of \iid{} samples from the censored distribution.
    Recall from \cref{problem:identification_estimation} that here we want to find an estimate $\hat{f}$ that is $\epsilon$-close to $\fstar$ in the following sense:
    \begin{definition}[$\eps$-separated regressors]\label{def:eps_separated_regressors}
        Two regressors $f,f'\in\hyF$ are $\eps$-separated (respectively, $\eps$-close) with respect to $\cD$ if $\Ex_{X\sim \cD}\abs{f(X)-f'(X)}\geq \eps$ (respectively, $\Ex_{X\sim \cD}\abs{f(X)-f'(X)}\leq \eps$).
    \end{definition}
    Now, a natural strengthening of our identification condition requires a \emph{quantitative} separation: whenever two admissible parameter tuples $\theta=(f,s,\cD,\cQ)$ and $\theta'=(f',s',\cD',\cQ')$ have regressors $f$ and $f'$ that are $\eps$-separated with respect to $\cD$, either their covariate distributions or the corresponding observed kernels from \cref{cond:identifiability:f} must be separated in an integrated sense.
    We formalize this via the following discrepancy:
    \begin{definition}[Discrepancy]\label{def:observable_discrepancy}
    Given a distribution $\cD$ and $\pi{=}(f,s,\cQ)$ and $\pi'{=}(f',s',\cQ')$, define
        \[
                d_{\cD}\!\inparen{\pi,\pi'}
                \coloneqq
                \iint \abs{
                    s(x,y)\,\cQ\!\inparen{y-f(x)}
                    -
                    s'\!(x,y)\,\cQ'\!\inparen{y-f'(x)}
                }\,\d\cD(x)\,\d y\,.
        \]
    \end{definition}
    This brings us to the following strengthening of our identification condition. 
    %
    \begin{mdframed}
    \vspace{-3mm}
        \begin{condition}[Estimation condition]\label{cond:estimation:f}
            A joint concept class $\Theta\subseteq\hyF\times\hyS\times\hyDX\times\hyQ$ is said to satisfy the estimation condition with rate function $R\colon \R_{\geq 0}\to (0,1)$ if, for every ordered pair of tuples $\theta=(f,s,\cD,\cQ)\in\Theta$, $\theta'=(f',s',\cD',\cQ')\in\Theta$, and every $\eps\geq 0$, writing $\pi_\theta=(f,s,\cQ)$ and $\pi_{\theta'}=(f',s',\cQ')$,
            \[
                \text{if}\quad\Ex\nolimits_{X\sim \cD}\abs{f(X)-f'(X)} \geq \eps\,,
                \quadtext{then,}
                \tv{\cD}{\cD'}+d_{\cD}\!\inparen{\pi_\theta,\pi_{\theta'}} \geq R(\eps)\,.
            \]
        \end{condition}
    \end{mdframed} 
    Hence, \cref{cond:estimation:f} can be read as a quantitative version of \cref{cond:identifiability:f}.
    The two terms capture complementary sources of separation: $\tv{\cD}{\cD'}$ measures a difference in the always-observed covariate distribution, while $d_{\cD}\!\inparen{\pi_\theta,\pi_{\theta'}}$ measures the remaining difference between the observed-outcome kernels when the covariate distributions are close.
    \cref{cond:estimation:f} requires the following: as the regressors $f$ and $f'$ become closer in $\Ex_{X\sim \cD}\abs{f(X)-f'(X)}$, the combined discrepancy $\tv{\cD}{\cD'}+d_{\cD}\!\inparen{\pi_\theta,\pi_{\theta'}}$ is allowed to become correspondingly smaller.
    In particular, as $R(\eps)$ increases (for a fixed value of $\eps$), the induced censored distributions become further and further from each other. 
    In contrast, \cref{cond:identifiability:f} only requires the observed kernels to differ on a set of positive \((\cD\otimes\nu)\)-measure whenever \(f\) and \(f'\) differ on a set of positive \(\cD\)-measure. It imposes no quantitative lower bound on the size of this difference.

    \subsubsection{Technical Overview: Challenges in Estimating $\fstar$ and Our Approach}
    Write $\pi^\star\coloneqq(\fstar,\sstar,\cQ^\star)$.
    Now, a natural approach to estimating $\fstar$ is to ``find'' an admissible tuple $\theta=(f,s,\cD,\cQ)\in\Theta$ that minimizes the discrepancy $\tv{\cD}{\cD^\star}+d_{\cD^\star}\!\inparen{\pi^\star,\pi_\theta}$ with respect to the true covariate distribution $\cD^\star$.
    Indeed, if we could ensure that $\tv{\cD}{\cD^\star}+d_{\cD^\star}\!\inparen{\pi^\star,\pi_\theta}< R(\eps)$, then \cref{cond:estimation:f} would immediately imply that $f$ and $\fstar$ are $\eps$-close with respect to $\cD^\star$.
    The difficulty is that evaluating (or even approximating) the second term requires access to the unknown kernel
    $k_{\pi^\star}(x,y)=\sstar(x,y)\,\cQ^\star\!\inparen{y-\fstar(x)}$, since
    \[
        d_{\cD^\star}\!\inparen{\pi^\star,\pi_\theta}
        =
        \iint \abs{k_{\pi_\theta}(x,y)-k_{\pi^\star}(x,y)}\,\d \cD^\star(x)\,\d y\,.
    \]
    While the censored distribution $\cC^\star$ (which we have sample access to) determines $k_{\pi^\star}$ at the population level, constructing a \emph{finite-sample} estimator of this quantity would amount to learning $k_{\pi^\star}$ in $L_1$-distance.
    However, learning the kernel $k_{\pi^\star}$ requires learning the unknown product $\sstar(x,y)\cQ^\star\!\inparen{y-\fstar(x)}$, {which involves the unknown regressor $\fstar$ and, hence, seems at least} as hard as the original estimation problem (which only requires us to estimate $\fstar$).
    Therefore, directly minimizing an empirical estimate of $d_{\cD^\star}\!\inparen{\pi^\star,\pi_\theta}$ without knowing $k_{\pi^\star}$ is not useful for learning {$\fstar.$}

    Our proof avoids estimating $d_{\cD^\star}\!\inparen{\pi^\star,\pi_\theta}$ directly.
    Instead, we construct a cover of the family of potential censored distributions arising from the joint concept class $\Theta$ and we select a candidate censored distribution  $\cC_{\wh{\theta}}$ from this cover using existing finite-sample estimation procedures due to \citet{yatracos1985rates}, which guarantee small total variation error $\tv{\cC_{\wh{\theta}}}{\cC^\star}$ (\cref{thm:distrLearningFiniteSet}).
    Then, we show that such control in total variation automatically yields a bound on the target discrepancy,
    namely $\tv{\wh{\cD}}{\cD^\star}+d_{\cD^\star}\!\inparen{\pi^\star,\wh{\pi}}\leq 5\,\tv{\cC_{\wh{\theta}}}{\cC^\star}$ (\cref{prop:disc-via-tv}).
    This reduction, from bounding the two-term discrepancy to bounding $\tv{\cC_{\wh{\theta}}}{\cC^\star}$, is the key innovation in our proof.

    \subsubsection{Our Estimation Result}
    Formally, our estimation result is as follows:
    \begin{theorem}[Finite-Sample Estimation]\label{thm:estimation:f}
        Let $\Theta\subseteq\hyF\times\hyS\times\hyD\times\hyQ$ be a joint concept class, and suppose $(\hyF,\hyD,\hyQ)$ satisfy \cref{def:regularity} (regularity conditions).
        Fix $\eta\in(0,1)$.
        {Suppose} $\Theta$ satisfies \cref{cond:estimation:f} with rate function {$R$}.
	        Then there exists an estimator such that, for any $\eps\in (0,\nfrac{\eta}{2}),\delta\in(0,1)$, given $n$ samples from any censored distribution $\cC^\star$ realized by a tuple $\theta^\star=(\fstar,\sstar,\cD^\star,\cQ^\star)\in\Theta$ satisfying $\Ex[\sstar(X,Y)]\geq\eta$, the estimator outputs $\wh{f}$ satisfying, with probability at least $1-\delta$,
        \[
            \Ex\nolimits_{X\sim \cD^\star}\abs{\fstar(X)-\wh{f}(X)} \leq \eps\,.
        \]
        Moreover, it suffices to take
        \newcommand{\epsEta}{\eps}
        \[
            n  =  O\inparen{
                \frac{1}{R(\epsEta)^2}\cdot
                { 
                    \log{\frac{M(\eta\, R(\epsEta)/C)}{\delta}}
                }
            }
            \quadwhere
            M(\gamma)
            = \cover{\hyF}{\gamma}{\norm{\cdot}_{L_1(\mu_X)}}\cdot 
              \cover{\hyS}{\gamma}{\norm{\cdot}_{L_1(\mu)}}\cdot 
              \cover{\hyD}{\gamma}{\mathsf{TV}}\cdot 
              \cover{\hyQ}{\gamma}{\mathsf{TV}}\,.
        \]
        Here, $C\geq 1$ is a sufficiently large universal constant and $\cover{\hyH}{\gamma}{d}$ is the size of a $\gamma$-cover of $\hyH$ in metric $d$.\footnote{The product bound $M(\gamma)$ may be loose because it ignores restrictions coupling the components of $\Theta$. Define
        $M_\Theta^{\rm int}(\rho)\coloneqq\min\{\abs{\mathcal N}:\mathcal N\subseteq\Theta,\ \sup_{\theta\in\Theta}\inf_{\bar\theta\in\mathcal N}\tv{\cC_\theta}{\cC_{\bar\theta}}\leq\rho\}$. The same argument gives a weakly tighter bound by replacing $M(\eta R(\eps)/C)$ with $M_\Theta^{\rm int}(R(\eps)/32)$.}
    \end{theorem}  
    Thus, up to standard regularity assumptions, \cref{cond:estimation:f} enables the estimation of $\fstar$ from a finite number of censored samples.
    The dependence on $R(\eps)$ reflects how strongly the censored data separates $\eps$-different regressors; as discussed above, smaller $R(\eps)$ leads to larger sample complexity.
    {The requirement of $\Ex[\sstar(X,Y)]\geq \eta$ is necessary for finite-sample estimation, and ensures that outcomes are observed on a nontrivial fraction of samples in expectation.}\footnote{{If this condition is violated, estimation can fail in a degenerate way. For example, let $\hyS$ contain selectors $s(x,y)\equiv s(x)$ that take value $1$ only on a set of measure 0 w.r.t.\ $\cD_X^\star$ and take value $0$ elsewhere. Then $\Ex[\sstar(X,Y)]=0$ and, with probability $1$, a finite sample contains no uncensored outcomes, so estimating $\fstar$ is impossible. In contrast, identification can still be possible (for instance, when $\hyF$ is the class of linear functions), since identifiers can depend on the censored density on sets of measure $0$ that are invisible to any finite sample procedure.}}
    The remaining dependence on the covering numbers in \cref{thm:estimation:f} is also expected.
    Concretely, a $\zeta$-cover is a finite subset of a class such that every element of the class is within ``distance'' $\zeta$ of at least one element in the subset.
    Similar complexity control over concept classes is necessary for regression even without selection bias \citep{alon1997scaleSensitive,neuralNetwork2009bartlett}.
    
    Next, we formally state the regularity conditions we mentioned in \cref{thm:estimation:f}.
    These requirements are routine in nonparametric estimation \citep{Tsybakov2009NonparametricEstimation,devroye2012combinatorial}.
    They ensure that small perturbations of $(f,s,\cD,\cQ)$ lead to small perturbations of the induced censored distribution (in total variation distance).
    \begin{assumption}[Regularity Conditions]
        \label{def:regularity}
        Let the concept classes $(\hyF,\hyD,\hyQ)$ satisfy the regularity conditions below with constants $(\sigma,L)=(\Omega(1),O(1))$ with respect to a {known} distribution $\mu$ over $\R^d \times \R$ (with marginals $\mu_X$ and $\mu_Y$):
        \begin{enumerate}
            \item Each $\cD \in \hyD$ is $\sigma$-smooth with respect to $\mu_X$.
            \item Each $\cQ \in \hyQ$ is $\sigma$-smooth with respect to $\mu_Y$ and is Lipschitz under translations (in $L_1$-norm) with constant $L$.
            
            \item For every $(f,\cD,\cQ)\in
            \hyF\times\hyD\times\hyQ$, let
            $\cP_{f,\cD,\cQ}$ denote the joint distribution of
            $(X,Y)$ generated by
            $X\sim\cD,
                \xi\sim\cQ,
                X\perp\xi,
                Y=f(X)+\xi.$
            Then $\cP_{f,\cD,\cQ}$ is $\sigma^2$-smooth with respect to $\mu$.
            
        \end{enumerate}
        Here, a distribution $\cP$ is said to be $\sigma$-smooth with respect to a reference distribution \(\mu_0\) if $\cP(z)\leq (\nfrac{1}{\sigma})\,\mu_0(z)$ for each $z\in \supp(\cP)$
        and a density $\cQ$ is said to be $L$-Lipschitz in $L_1$-norm if $\norm{\cQ(\cdot-z)-\cQ(\cdot-z')}_1 \leq L\abs{z-z'}$ for all $z,z'\in\R$. 
    \end{assumption}
    The smoothness assumptions prevent distributions in $\hyD$ and $\hyQ$ from concentrating too much probability mass on sets with very small $\mu_X$- or $\mu_Y$-mass.
    This is a common requirement for uniform learning guarantees, since without such control one can construct families of distributions that are statistically indistinguishable from finitely many samples (\eg{}, \citep{Tsybakov2009NonparametricEstimation}).
    The Lipschitz requirement on $\hyQ$ is similarly standard: it provides quantitative continuity of the noise family in $L_1$, which allows a finite $\zeta$-cover of $\hyQ$ to translate into a finite cover of the induced censored laws.

    \begin{proof}[Proof of \cref{thm:estimation:f}]
        Fix any censored distribution $\cC^\star=\cC_{\fstar,\sstar,\cD^\star,\cQ^\star}$ realizable with respect to $\Theta$, and write $\theta^\star=(\fstar,\sstar,\cD^\star,\cQ^\star)\in\Theta$.
        Let $\pi^\star\coloneqq(\fstar,\sstar,\cQ^\star)$.
        For any tuple $\theta=(f,s,\cD,\cQ)\in \Theta$, let $\cC_\theta$ denote the induced censored distribution under the selection-bias model (\cref{def:selection_bias_model}).
    
        \smallskip 
        \noindent\textbf{Step 1 (A finite candidate set via covers).}
        Fix $\zeta\coloneqq R(\eps)/32$.
        Applying \cref{thm:cover-size} with its accuracy parameter set to $\zeta/2$ and $L_Q=L$, using monotonicity of covering numbers, and recentering as in \cref{cor:internal-cover}, there is a finite set $\mathcal N_\zeta\subseteq \Theta$ such that, for any choice of realizable parameters $(\fstar,\sstar,\cD^\star,\cQ^\star)$, there exists
        $\tilde{\theta}=(\tilde{f},\tilde{s},\tilde{\cD},\tilde{\cQ})\in\mathcal N_\zeta$ with
        \[
            \tv{\cC_{\tilde{\theta}}}{\cC^\star}\leq \zeta
            \qquadand 
            \abs{\mathcal N_\zeta}\leq M(\eta R(\eps)/C)\,.
            \yesnum\label{eq:approx-by-cover-and-theta-size}
        \]
        Here, $M(\cdot)$ is defined in the statement of \cref{thm:estimation:f}.
        For the remainder of the proof, we keep the dependence on $\abs{\mathcal N_\zeta}$, and only substitute \eqref{eq:approx-by-cover-and-theta-size} at the end.
    
        \smallskip 
        \noindent\textbf{Step 2 (Estimator for $\fstar$).}
        We defer the details to \cref{sec:oracle}.
        We will use the following finite-class distribution selection guarantee due to \citet{yatracos1985rates}.
        \begin{theorem}[\citet{yatracos1985rates}]\label{thm:distrLearningFiniteSet}
            There is an algorithm that, given candidate distributions $\cC_1,\cC_2,\ldots,\cC_M$, a parameter $\zeta>0$, and
            $\bigl\lceil\log\inparen{3M^2/\delta}/\inparen{2\zeta^2}\bigr\rceil$ \iid{} samples from an unknown distribution $\cC^\star$,
            outputs $j\in[M]$ such that $\tv{\cC_j}{\cC^\star}
                 \leq 
                3\min\nolimits_{i\in[M]}\tv{\cC_i}{\cC^\star}  +  2\zeta,$
            with probability at least $1-\nfrac{\delta}{3}$.
        \end{theorem}
        Given $n$ \iid{} samples from $\cC^\star$, our estimator applies the procedure in \cref{thm:distrLearningFiniteSet} to the family $\inbrace{\cC_\theta:\theta\in\mathcal N_\zeta},$
        which we index as $\inbrace{\cC_1,\ldots,\cC_M}$ with $M=\abs{\mathcal N_\zeta}$.
        Let $\wh{\theta}\in\mathcal N_\zeta$ be the selected candidate.
        Writing $\wh{\theta}=(\wh{f},\wh{s},\wh{\cD},\wh{\cQ})$, our estimator outputs $\wh{f}$.

        \smallskip 
        \noindent\textbf{Step 3 (A key proposition and conclusion).}
        Define $\wh{\pi}\coloneqq(\wh{f},\wh{s},\wh{\cQ})$.
        The main step is the following proposition, which relates the two-term discrepancy to the total variation error of the selected candidate model.
    
        \begin{proposition}\label{prop:disc-via-tv}
            It holds that $\tv{\cD^\star}{\wh{\cD}}+d_{\cD^\star}\!\inparen{\pi^\star,\wh{\pi}}
                \leq 
                5\,\tv{\cC_{\wh{\theta}}~}{~~\cC^\star}.$
        \end{proposition}
        \noindent Assuming \cref{prop:disc-via-tv}, we complete the proof using \cref{thm:distrLearningFiniteSet} and \cref{cond:estimation:f}.
        Indeed, applying \cref{thm:distrLearningFiniteSet} to the finite family $\inbrace{\cC_\theta:\theta\in\mathcal N_\zeta}$ (so $M=\abs{\mathcal N_\zeta}$) yields the following:
        for all $\delta\in(0,1)$, with probability at least $1-\delta$,
        \begin{equation}
            \tv{\cC_{\wh{\theta}}}{\cC^\star}
            ~~\leq~~
            3\min\nolimits_{\theta\in\mathcal N_\zeta}\tv{\cC_\theta}{\cC^\star}
             + 
            2\sqrt{\frac{\log\inparen{2\abs{\mathcal N_\zeta}/\delta}}{2n}}\,.
            \yesnum\label{eq:finite-selection}
        \end{equation}
        Using \eqref{eq:approx-by-cover-and-theta-size} gives, with probability at least $1-\delta$,
        \begin{equation}
            \tv{\cC_{\wh{\theta}}}{\cC^\star}
            ~~\leq~~
            3\zeta
             + 
            2\sqrt{\frac{\log\inparen{2\abs{\mathcal N_\zeta}/\delta}}{2n}}\,.
            \yesnum\label{eq:tv-bound}
        \end{equation}
        Taking 
        \[
            n  \geq  2\cdot \frac{\log\inparen{2\abs{\mathcal N_\zeta}/\delta}}{\zeta^2}\,
            \yesnum\label{eq:n-choice}
        \]
        yields the right-hand side of \eqref{eq:tv-bound} is at most $4\zeta$, and hence, with probability at least $1-\delta$, 
        \[
            \tv{\cC_{\wh{\theta}}}{\cC^\star}\leq 4\zeta=\frac{R(\eps)}{8}\,.
        \]
        Substituting this in \cref{prop:disc-via-tv} gives 
        \[
            \tv{\cD^\star}{\wh{\cD}}+d_{\cD^\star}\!\inparen{\pi^\star,\wh{\pi}}
             \leq 
            5\cdot \tv{\cC_{\wh{\theta}}}{\cC^\star}
            \leq 20\cdot \zeta
             <
            R(\eps)\,.
        \]
        Taking the contrapositive of \cref{cond:estimation:f} for the ordered pair $(\theta^\star,\wh{\theta})\in\Theta^2$ gives that $\wh{f}$ and $\fstar$ are $\eps$-close with respect to $\cD^\star$ (\ie{}, $\Ex_{X\sim \cD^\star}\sabs{\wh{f}(X)-\fstar(X)} < \eps$).
        %
        Finally, substituting \eqref{eq:approx-by-cover-and-theta-size} into our choice of $n$ yields the stated sample complexity bound.

        \smallskip
        \noindent\textbf{Step 4 (Proof of \cref{prop:disc-via-tv}).}
        First, we recall the standard expression for the conditional density of the event $\{D=1,\,Y\in \d y\}$ given $X=x$:
        for $\pi=(f,s,\cQ)$, define
        \[
            k_\pi(x,y)\coloneqq s(x,y)\,\cQ\!\inparen{y-f(x)}\,.
        \]
        For any $\cD\in\hyD$ and any $\pi,\pi'\in\hyF\times\hyS\times\hyQ$, let
        $\cC_{\cD,\pi}\coloneqq \cC_{(f,s,\cD,\cQ)}$ and $\cC_{\cD,\pi'}\coloneqq \cC_{(f',s',\cD,\cQ')}$ denote the induced censored laws.
    
        \begin{lemma}\label{lem:tv-vs-discrepancy}
            For any $\cD\in\hyD$ and $\pi,\pi'\in\hyF\times\hyS\times\hyQ$,
            \[
                \frac{1}{2}\,d_{\cD}\!\inparen{\pi,\pi'}
                 \leq 
                \tv{\cC_{\cD,\pi}}{\cC_{\cD,\pi'}}
                 \leq 
                d_{\cD}\!\inparen{\pi,\pi'}\,.
            \]
        \end{lemma}
        \vspace{-7mm}
        \begin{proof}[Proof of \cref{lem:tv-vs-discrepancy}]
            Observe that 
            \begin{align*}
                2\,\tv{\cC_{\cD,\pi}}{\cC_{\cD,\pi'}}
                &=
                \int\! \abs{\cD(x)k_\pi(x,y)-\cD(x)k_{\pi'}(x,y)}~\d x\d y\\
                &\qquad\qquad
                +
                \int\!\abs{
                    \cD(x)\inparen{1-\int\!\!\! k_\pi(x,t)~\d t}
                    -
                    \cD(x)\inparen{1-\int\!\!\! k_{\pi'}(x,t)~\d t}
                }\d x\,.
            \end{align*}
            The first term equals $d_{\cD}\!\inparen{\pi,\pi'}$.
            Since the second term is non-negative, we obtain
            $2\,\tv{\cC_{\cD,\pi}}{\cC_{\cD,\pi'}}\geq d_{\cD}\!\inparen{\pi,\pi'}$.
            For the upper bound, apply the triangle inequality to the second term:
            \[
                \int \cD(x)\abs{\int \inparen{k_\pi(x,t)-k_{\pi'}(x,t)}\,\d t}\,\d x
                \le
                \int \cD(x)\int \abs{k_\pi(x,t)-k_{\pi'}(x,t)}\,\d t\,\d x
                =
                d_{\cD}\!\inparen{\pi,\pi'}\,.
            \]
            Hence $2\,\tv{\cC_{\cD,\pi}}{\cC_{\cD,\pi'}}\leq 2\,d_{\cD}\!\inparen{\pi,\pi'}$, that is,
            $\tv{\cC_{\cD,\pi}}{\cC_{\cD,\pi'}}\leq d_{\cD}\!\inparen{\pi,\pi'}$.
        \end{proof}
        Next, we compare $\cC_{\wh{\theta}}$ to the distribution that uses the \emph{same} $(\wh{f},\wh{s},\wh{\cQ})$ but the \emph{true} covariate distribution $\cD^\star$.
        This is an auxiliary distribution and need not be induced by a tuple in $\Theta$.
        By marginalization,
        \begin{equation}
            \tv{\wh{\cD}}{~~\cD^\star}
             \leq 
            \tv{\cC_{\wh{\theta}}}{~~\cC^\star}\,.
            \yesnum\label{eq:tv-marginal}
        \end{equation}
        Moreover, by the triangle inequality,
        \begin{equation}
            \tv{\cC_{\cD^\star,\wh{\pi}}~}{~~\cC^\star}
            ~\le~
            \tv{\cC_{\cD^\star,\wh{\pi}}~}{~~\cC_{\wh{\theta}}}
            ~+~
            \tv{\cC_{\wh{\theta}}}{~~\cC^\star}\,.
            \yesnum\label{eq:tv-triangle}
        \end{equation}
        The first term equals $\tv{\cD^\star}{\wh{\cD}}$ because $\cC_{\cD^\star,\wh{\pi}}$ and $\cC_{\wh{\theta}}$ have the same conditional law of $(D,Y)$ given $X$ and differ only in the $X$-marginal.
        This is formalized by \cref{lem:tv-same-conditional}, whose proof appears at the end of this section.
        \begin{lemma}[Changing only the $X$-marginal preserves TV]\label{lem:tv-same-conditional}
            Let $\cD$ and $\cD'$ be probability measures on $\cX$.
            For each $x\in\cX$, let $\cP_x$ be a probability measure on a measurable space $\cZ$.
            Define two probability measures $\cC$ and $\cC'$ on $\cX\times\cZ$ such that their densities/masses satisfy 
            \[
                \cC(x,z) \propto \cD(x)\,\cP_x(z)
                \qquadand
                \cC'(x,z) \propto \cD'(x)\,\cP_x(z)\,,
            \]
            where the proportionality constants are chosen so that $\cC$ and $\cC'$ are probability measures.
            Equivalently, $\cC$ and $\cC'$ have the same conditional law of $Z$ given $X=x$, and differ only in the marginal law of $X$.
            Then,
            \[
                \tv{\cC}{\cC'} = \tv{\cD}{\cD'}\,.
            \]
        \end{lemma}
        Combining \eqref{eq:tv-triangle}, $\tv{\cC_{\cD^\star,\wh{\pi}}}{\cC_{\wh{\theta}}}=\tv{\cD^\star}{\wh{\cD}}$, and \eqref{eq:tv-marginal} yields
        \begin{equation}
            \tv{\cC_{\cD^\star,\wh{\pi}}~}{~~\cC^\star}
             \leq 
            2\,\tv{\cC_{\wh{\theta}}~}{~~\cC^\star}\,.
            \yesnum\label{eq:tv-transfer}
        \end{equation}
        Finally, apply \cref{lem:tv-vs-discrepancy} with $\cD=\cD^\star$, $\pi=\pi^\star$, and $\pi'=\wh{\pi}$.
        Since $\cC^\star=\cC_{\cD^\star,\pi^\star}$, we have
        \[
            d_{\cD^\star}\!\inparen{\pi^\star,\wh{\pi}}
             \leq 
            2\,\tv{\cC_{\cD^\star,\wh{\pi}}}{\cC^\star}\,.
        \]
        Using \eqref{eq:tv-transfer} gives
        \[
            d_{\cD^\star}\!\inparen{\pi^\star,\wh{\pi}}
             \leq 
            2\cdot \tv{\cC_{\cD^\star,\wh{\pi}}}{\cC^\star}
             \leq 
            4\,\tv{\cC_{\wh{\theta}}}{\cC^\star}\,.
        \]
        Combining this with \eqref{eq:tv-marginal} proves \cref{prop:disc-via-tv}:
        \[
            \tv{\cD^\star}{\wh{\cD}}
            +d_{\cD^\star}\!\inparen{\pi^\star,\wh{\pi}}
            \leq
            5\,\tv{\cC_{\wh{\theta}}}{\cC^\star}\,.\qedhere
        \]
    \end{proof}
    \vspace{-5mm}
    \subsubsection*{Proof of \cref{lem:tv-same-conditional}}
        %
        By construction, for $E=B\times \cZ$, $\cC(E)=\cD(B)$ and $\cC'(E)=\cD'(B)$.
        Therefore, 
        \[
            \tv{\cC}{\cC'}
            ~=~
            \sup\nolimits_{E\subseteq \cX\times\cZ}\abs{\cC(E)-\cC'(E)}
            ~\ge~
            \sup\nolimits_{B\subseteq \cX}\abs{\cD(B)-\cD'(B)}
            ~=~
            \tv{\cD}{\cD'}\,.
        \]
        To see the upper bound, fix an arbitrary measurable event $E\subseteq \cX\times\cZ$.
        For each $x\in\cX$, define the $x$-section $E_x \;\coloneqq\; \inbrace{z\in\cZ:\ (x,z)\in E}.$
        Define the function $h:\cX\to[0,1]$ by $h(x)\coloneqq \cP_x(E_x).$
        Then $\cC(E)=\int h(x)\,\cD(\d x)$ and $\cC'(E)=\int h(x)\,\cD'(\d x)$,
        so that $\cC(E)-\cC'(E)=\int h(x)\,\inparen{\cD-\cD'}(\d x).$
        Since $0\leq h(x)\leq 1$ for all $x$, we have 
        \[
            \inf\nolimits_{B\subseteq \cX}\inparen{\cD(B)-\cD'(B)}
            ~~\leq~~
            \int h\,\d\cD-\int h\,\d\cD'
            ~~\leq~~
            \sup\nolimits_{B\subseteq \cX}\inparen{\cD(B)-\cD'(B)}\,.
        \]
        Taking absolute values gives 
        \[
            \abs{\cC(E)-\cC'(E)}
            \le
            \sup\nolimits_{B\subseteq \cX}\abs{\cD(B)-\cD'(B)}
            =
            \tv{\cD}{\cD'}\,.
        \]
        Since this holds for every $E\subseteq \cX\times\cZ$, a supremum over $E$ yields
        $\tv{\cC}{\cC'}\leq \tv{\cD}{\cD'}$.\qed{}
        %
        %

        \subsubsection{Oracle Efficiency and Implementation Considerations} \label{sec:estimation:oracle}
        \par
        Our estimator is sample efficient, but not necessarily computationally efficient. Roughly, our algorithm selects the ``best fitting'' candidate from a finite family 
        \[
            \inbrace{\cC_\theta : \theta \in \mathcal N_\zeta}
        \]
        which costs only $\log\abs{\mathcal N_\zeta}$
        samples but can require (at least) $\abs{\mathcal N_\zeta}$ running time in the worst case, which is prohibitively slow when $\mathcal N_\zeta$ arises from covers and is very large in rich nonparametric classes.

        However, our algorithm has a key structural property that makes it amenable to practical implementation: its computational bottleneck reduces to a \textit{single} call to a \emph{density-estimation oracle}.
        This is noteworthy because sample-efficient algorithms for selection-bias problems do not, in general, admit such reductions. For instance, the algorithm of \citet{Kontonis2019EfficientTS} for the closely related problem of mean estimation (a special case of regression) with \emph{unknown} truncation (a form of selection bias) also searches over a discretized set of candidate parameters but, unlike our algorithm, must solve an \textit{auxiliary} optimization problem for each candidate.
        This nested structure---a cover search with a nontrivial per-candidate inner computation---is common in algorithms for unknown selection bias and makes them difficult to implement in practice.
        
        By contrast, our algorithm simply applies the finite-class selection procedure of \citet{yatracos1985rates} to choose a near-best fitting $\wh{\theta} \in \mathcal N_\zeta$, outputting $\wh{f}$, the regressor component of $\wh{\theta}$.
        The entire computational challenge is thus isolated to a single primitive: approximately selecting (or fitting) a density from the induced finite family $\inbrace{\cC_\theta : \theta \in \mathcal N_\zeta}$.
        In many scenarios of interest, this primitive can be implemented efficiently using methods that are heuristic but practically successful---such as diffusion models or other modern generative modeling techniques---potentially enabling practical algorithms for regression with unknown selection bias.

    \begin{remark}[Connections to MLE]
        In any parametric instance of our framework, the density-estimation step can be implemented by maximum likelihood estimation (MLE) on the observed (censored) data. (This is sufficient because MLE minimizes empirical KL divergence from the true censored law, and KL upper-bounds total variation via Pinsker's inequality, yielding the TV upper bound needed by \cref{thm:estimation:f}.) For the classical Heckman model, which is also parametric, this means that we learn the selection correction and the outcome regression function jointly via MLE. This differs from Heckman's two-step procedure, which first estimated the probit selection equation, then added the inverse Mills ratio as a correction term to the outcome regression.

        More broadly, any procedure that fits the induced censored distribution well can be plugged into our framework. Beyond MLE on parametric families, this includes polynomial-based or, more generally, kernel-based density estimation methods that have been used for specific families (see \cite{chen2025generalmixtures} for a recent guarantee of this kind).
    \end{remark}

\section{{Implications} of Our Results}
    \label{sec:app}
    As we mentioned before, for Conditions~\ref{cond:identifiability:f} and \ref{cond:estimation:f} to be useful, given all the existing sufficient conditions, they need to capture interesting examples which are not captured by existing conditions or their natural generalizations. 
    In this section, we demonstrate this by revisiting several classic scenarios of sample selection bias and exploring some new ones, and showing that Conditions~\ref{cond:identifiability:f} and \ref{cond:estimation:f} enable identification and estimation in more general cases than earlier sufficient conditions.
    {Throughout this section, whenever we list four component classes without separately defining a joint concept class, $\Theta$ denotes their Cartesian product.}
 \subsection{$\fstar$ Can Be Identified Even When $\sstar$ Is Not Identifiable }
\label{sec:app:non-iden-s}
Most existing methods for regression with sample selection bias follow an ``estimate selection, then debias'' template: first identify (or consistently estimate) the selection mechanism $\sstar$ (or a suitable functional of it), and then plug this estimate into a second-stage regression procedure. 
{Indeed, \citet{heckman1979sample}'s original method follows this template: it first estimates a functional $v_{\sstar}$ of $\sstar$, and then uses $v_{\sstar}$ to debias OLS.
Likewise, \citet*{das2003nonparametric}'s method (among the most general available methods) first estimates the selection propensity $p(x)\coloneqq \Pr(D=1\mid X=x)$ nonparametrically, then uses comparisons between observations $(x_1,y_1)$ and $(x_2,y_2)$  with $p(x_1)\approx p(x_2)$ to identify $\fstar$ (under suitable assumptions).}

{In this section, using our characterization, we show that} this template does not capture the full picture: there are regimes in which \emph{no} procedure can identify $\sstar$ from censored data, yet $\fstar$ remains identifiable. 
The proof of \cref{thm:f-identifiable-without-s} appears in \cref{sec:proofof:thm:f-identifiable-without-s}.
\begin{restatable}[Identifying $\fstar$ without Identifiability of $\sstar$]
    {theorem}{fIdentifiableWithoutS}
\label{thm:f-identifiable-without-s}
There exist {component} classes $(\hyF,\hyS,\hyDX,\hyQ)$ {such that the Cartesian-product class $\Theta_0\coloneqq\hyF\times\hyS\times\hyDX\times\hyQ$ satisfies} \cref{cond:identifiability:f} {and} the selection function $\sstar\in\hyS$ is \emph{not} identifiable from censored distributions realizable with respect to {$\Theta_0$}.
\end{restatable}
Since \cref{cond:identifiability:f} holds, \cref{thm:characterization:f} guarantees that $\fstar$ is identifiable from {realizable} censored distributions.
At the same time, \cref{thm:f-identifiable-without-s} shows that multiple \textit{distinct} selectors $s\in\hyS$ can lead to the same censored distribution, making $\sstar$ unidentifiable.
Any approach that requires recovering $\sstar$, even approximately, as an intermediate step therefore cannot identify $\fstar$ in such settings.
This separation also extends to the finite-sample regime;  the specific estimation rates and sample sizes are not the key point of this section, so we defer the details to \cref{sec:proofof:thm:f-identifiable-without-s}.
Together, these results show that sometimes $\fstar$ can be learned using \cref{thm:characterization:f,thm:estimation:f} even when learning $\sstar$ is impossible, showing that our characterizations qualitatively extend beyond existing works.
That said, for these results to be meaningful, the concept classes witnessing the separation should be meaningful: 
 
\paragraph{Families in \cref{thm:f-identifiable-without-s}.}
The proof only requires simple concept classes: $\hyQ$ needs just two distributions and $\hyS$ needs just two rules (a constant function and an ``interval'' function  that takes a larger value on one interval in $y$).
The main structural restriction is \emph{one-sidedness} of noise: the noise distributions in our construction are supported on $\R_{\geq 0}$.
This is natural in settings where outcomes are inherently non-negative, such as durations, utilization measures, and inefficiency distributions in stochastic frontier models~\citep{aigner1977stochasticFrontier}.
Beyond this, the remaining requirements on $\hyQ$ are mild: we only need positive density on a short interval such as $(0,1)$, and the specific densities in our proof are chosen for convenience, and the same separation holds with other one-sided families (\eg{}, Gaussians truncated to $\R_{\geq 0}$).
See \cref{rem:app:construction} for a description of these assumptions.
{Two additional features are worth noting: the non-identifiability of $\sstar$ is \emph{robust to model enlargement} (taking supersets $\hyS\supseteq \hyS_{\min}$ and $\hyQ\supseteq \hyQ_{\min}$ preserves it), and it is \emph{not} an artifact of poor overlap (both selectors satisfy $s(x,y)\geq \eta$ for all $(x,y)$).}

\subsection{Revisiting Classical Examples from \cref{sec:illustrations}}\label{sec:illustrations_proofs}
{In this section, we revisit the classical models of limited dependence introduced in \cref{sec:illustrations}. Our goal is to provide a straightforward characterization of each variant and show that our identification condition (\Cref{cond:identifiability:f}) can be used to verify the identification of the target parameter in each of these settings. We briefly remark that in the following sections, we assume that $\cD_X$ is known exactly. This is a mild assumption because, regardless of the selection equation, the {covariates} are always observed.}

\subsubsection{Heckman's Sample-Selection Model}
We will now show that we may subsume Heckman's model into \Cref{def:selection_bias_model}. First, one can verify that Heckman's selection functions must {take} the form below.
\begin{proposition}\label{prop:heckman_selection_function}
    Let $s^\star(x,y) = \Pr\!\inparen{d=1\,\mid\,X=x,\ Y=y}$ be the probability {that $(x,y)$ is observed}, where $d=1$ is a selection indicator variable. Under Heckman's model, it holds that
    \[
    s^\star(x,y) =  \Phi\!\left(\frac{x^\top v/\sigma' + \rho\,\inparen{y-x^\top w}/\sigma}{\sqrt{1-\rho^2}}\right)\,.
\]
\end{proposition}
\begin{proof}
    {Fix \(x,y\) and write \(\epsilon\coloneqq y-w^\top x\). }First, observe that $d=1$ iff $d^\star = v^\top x + \xi' > 0$. Thus,
    \begin{align*}
        s^\star(x,y) = \Pr(d=1|X=x,Y=y) &= \Pr(v^\top x + \xi' > 0|X=x,Y=y)\\ 
        &= \Pr(v^\top x + \xi' > 0|X=x,\xi=\epsilon)\\
        &= \Pr(v^\top x + \xi' > 0|\xi=\epsilon)\,.
    \end{align*}
    {Here,} the last equality follows from $(\xi,\xi') \perp x$. Hence, $s^\star(x,y) = \Pr(\xi' > -v^\top x|\xi=\epsilon)$. Recall that $\xi$ and $\xi'$ are jointly normal; {thus, we have}
    \[
        \xi' | (\xi=\epsilon) \sim \normal{\frac{\rho \sigma' \epsilon}{\sigma}}{(\sigma')^2(1-\rho^2)}\,.
    \]
    Then, 
    \begin{align*}
         s^\star(x,y) &= \Pr(\xi' > -v^\top x|\xi=\epsilon)\\
        &= \Pr\left(\frac{\xi'-\frac{\rho \sigma' \epsilon}{\sigma}}{\sigma' \sqrt{1-\rho^2}} > \frac{-v^\top x -\frac{\rho \sigma' \epsilon}{\sigma}}{\sigma' \sqrt{1-\rho^2}} \;\middle|\; \xi = \epsilon\right)\,.
    \end{align*}
    The {left-hand} side is now distributed as a standard normal variable. Consequently,
    \begin{align*}
        s^\star(x,y) = 1- \Phi\left(\frac{-v^\top x -\frac{\rho \sigma' \epsilon}{\sigma}}{\sigma' \sqrt{1-\rho^2}}\right) 
        = \Phi\left(\frac{v^\top x + \frac{\rho \sigma' \epsilon}{\sigma}}{\sigma' \sqrt{1-\rho^2}}\right) = \Phi\!\inparen{\frac{v^\top x/\sigma' + \rho\,\inparen{y-w^\top x}/\sigma}{\sqrt{1-\rho^2}} }\,,
    \end{align*}
    where the second equality uses that $1-\Phi(-x) = \Phi(x)$ and the third equality uses $\xi = y-w^\top x$.

\end{proof}
\noindent Hence, the following concept classes capture \cref{def:heckman}:
        \begin{align*}
            \hyF & =  \inbrace{f(x)=\inangle{w, x}:\ w\in\R^d }\,,\\
        \hyS & =  \inbrace{ (x,y)\mapsto \Phi\!\inparen{\frac{v^\top x/\sigma' + \rho\,\inparen{y-w^\top x}/\sigma}{\sqrt{1-\rho^2}} } :\ v\in\R^d,\ \rho\in(-1,1),\ \sigma,\sigma'>0,\ w\in\R^d} \,,\\
	        \hyQ &= \{\normal{0}{\sigma^2} : {\sigma>0}\}\,.
        \end{align*}
         In the next theorem, we prove that \Cref{cond:identifiability:f} is satisfied with this choice of $\hyF$, $\hyS$, and $\hyQ$. 
         \begin{theorem}[Heckman Model Identification]
            \label{thm:heckman_identifiable}
	            Let $\hyF,\hyS,\hyQ$ be the Heckman concept classes defining the outcome equation, selection function, and noise distributions, respectively. Assume the covariate distribution $\cD_X$ has support with nonempty interior in $\R^d$.  Let \(\hyDX\coloneqq\{\cD_X\}\). Then, it holds that \Cref{cond:identifiability:f} is satisfied for the Cartesian-product class $\Theta_{\rm H}\coloneqq\hyF\times\hyS\times\hyDX\times\hyQ$.
        \end{theorem} 
        \begin{proof}
	            \sloppy Let $f_1(x) = w_1^\top x$ and $f_2(x) = w_2^\top x$ be two outcome functions such that \\ $\cD_X\left(\{x\in\cX:f_1(x)\neq f_2(x)\}\right)>0.$ Since $\supp(\cD_X)$ has nonempty interior and \(f_1-f_2\) is linear, this means \(w_1\neq w_2\).
Let 
            \[
            s_1(x,y) = \Phi\!\inparen{\frac{v_1^\top x/\sigma_1' + \rho_1\,\inparen{y-u_1^\top x}/\tau_1}{\sqrt{1-\smash{\rho_1^2}}} }\quadand
	             s_2(x,y) = \Phi\!\inparen{\frac{v_2^\top x/\sigma_2' + \rho_2\,\inparen{y-u_2^\top x}/\tau_2}{\sqrt{1-\smash{\rho_2^2}}} 
	            }\,
	            \]
	            be two selection functions.
            Define
            \[
                a_i \coloneq \frac{1}{\sqrt{1-\smash{\rho_i^2}}}\left(\frac{v_i}{\sigma_i'} - \frac{\rho_iu_i}{\tau_i}\right)\,, 
                \quad 
	                b_i \coloneq \frac{\rho_i}{\tau_i\sqrt{1-\smash{\rho_i^2}}}\,,
	            \]
	            for $i \in \{1,2\}$. We may then re-parameterize the selection functions as $s_1(x,y) = \Phi(a_1^\top x + b_1y)$ and $s_2(x,y) = \Phi(a_2^\top x + b_2y)$.
            Let \(\cQ_i=\normal{0}{\sigma_i^2}\in\hyQ\), where \(\sigma_i>0\), for \(i\in\{1,2\}\).
Now, assume for the sake of contradiction that \cref{cond:identifiability:f} fails for these choices. Then the corresponding observed kernels agree $(\cD_X \otimes \nu)$-a.e.:
            \[
                \Phi(a_1^\top x + b_1 y)\,\frac{1}{\sigma_1}\varphi(\sigma_1^{-1}(y - \langle w_1,x\rangle)) = \Phi(a_2^\top x + b_2 y)\,\frac{1}{\sigma_2}\varphi(\sigma_2^{-1}(y - \langle w_2,x\rangle))\,.
            \]
            Here, $\varphi(\cdot)$ is the standard normal PDF. 
	Taking the natural logarithm of both sides gives the following equality, \((\cD_X\otimes\nu)\)-a.e.:
            \begin{equation}
                \log\bigl(\Phi(a_1^\top x + b_1 y)/\sigma_1\bigr) - \frac{1}{2\sigma_1^2}(y-\langle w_1,x\rangle)^2 = \log\bigl(\Phi(a_2^\top x + b_2 y)/\sigma_2\bigr) - \frac{1}{2\sigma_2^2}(y-\langle w_2,x\rangle)^2\,.
                \label{eq:log_equality}
            \end{equation}

	            The two sides of the above equation are continuous functions of \((x,y)\). By assumption, \(\supp(\cD_X)\) contains a nonempty open set \(B\). Hence the equality above holds for all \((x,y)\in B\times\R\). 
            Indeed, if the two continuous functions differed at some point in \(B\times\R\), they would differ on a nonempty open neighborhood and therefore on a set of positive \((\cD_X\otimes\nu)\)-measure, contradicting \eqref{eq:log_equality}. 
            We may therefore argue pointwise on \(B\times\R\).

            Using \(\log\Phi(t)=-t^2/2+o(t^2)\) as \(t\to-\infty\) and \(\log\Phi(t)=o(t^2)\) as \(t\to\infty\), divide both sides of \eqref{eq:log_equality} by \(-y^2/2\) and let \(y\) tend to \(+\infty\) and \(-\infty\), respectively. For every fixed \(x\in B\), this gives
            \[
                \frac{1}{\sigma_1^2}+b_1^2\mathds{1}\{b_1<0\}=\frac{1}{\sigma_2^2}+b_2^2\mathds{1}\{b_2<0\}
                \qquad\text{and}\qquad
                \frac{1}{\sigma_1^2}+b_1^2\mathds{1}\{b_1>0\}=\frac{1}{\sigma_2^2}+b_2^2\mathds{1}\{b_2>0\}\,.
            \]
            For each \(i\), the smaller of the corresponding two coefficients is \(1/\sigma_i^2\). Hence, \(\sigma_1=\sigma_2\eqqcolon\sigma\).

Let $\lambda(z) = \varphi(z)/\Phi(z)$ be the inverse Mills ratio. Differentiating \eqref{eq:log_equality} with respect to $y$ gives:
            \begin{equation}
                b_1 \lambda(a_1^\top x + b_1 y) - b_2 \lambda(a_2^\top x + b_2 y) = -\frac{1}{\sigma^2}(\langle w_1,x\rangle - \langle w_2,x\rangle)\,.
                \label{eq:lambda_equality}
            \end{equation}
            Since the right-hand side is constant with respect to $y$, the left-hand side must also be constant for all $y \in \R$.
            We consider four cases depending on the values of $b_1$ and $b_2$ and in each case arrive at a contradiction. 
            Recall that we know that $\lim_{z\to\infty} \lambda(z) = 0$ and $\lambda(z) \sim -z$ as $z \to -\infty$. 
            

\paragraph{Case 1: $b_1 \neq 0$ and $b_2 \neq 0$.}
            \begin{itemize}[leftmargin=20pt]
                \item \textbf{Subcase 1.1: $b_1$ and $b_2$ have the same sign.} If $b_1 > 0$ and $b_2 > 0$, let $y \to \infty$; if \(b_1,b_2<0\), let \(y\to-\infty\). In either case, the arguments of both $\lambda$ functions tend to $\infty$, so $\lambda(\cdot) \to 0$. The left-hand side of \eqref{eq:lambda_equality} tends to 0. This implies
                \[
                    \langle w_1,x\rangle - \langle w_2,x\rangle = 0 \implies \langle w_1 - w_2, x \rangle = 0\,.
                \]
                Since this holds for all $x\in B$ and \(B\) is a nonempty open set, we must have $w_1 = w_2$.
                \cref{eq:lambda_equality} simplifies to $b_1 \lambda(a_1^\top x + b_1 y) = b_2 \lambda(a_2^\top x + b_2 y)$. Now, consider the opposite tail: $y \to -\infty$ when \(b_1,b_2>0\), and \(y\to\infty\) when \(b_1,b_2<0\). Using $\lambda(z) \sim -z$, we get:
                \[
                    b_1 (-(a_1^\top x + b_1 y)) \sim b_2 (-(a_2^\top x + b_2 y)) \implies -b_1 a_1^\top x - b_1^2 y \sim -b_2 a_2^\top x - b_2^2 y\,.
                \]
                For this to hold, the coefficients of $y$ must be equal, implying $b_1^2 = b_2^2$. Since $b_1,b_2$ have the same sign, we have $b_1=b_2$. Substituting \(b_1=b_2\) into the exact equality at the start of this subcase and using that \(\lambda\) is strictly decreasing gives $a_1^\top x = a_2^\top x$. Since \(B\) is a nonempty open set, $a_1=a_2$.
                Thus, $(a_1,b_1)=(a_2,b_2)$. The equality \(w_1=w_2\) contradicts the choice of \(f_1,f_2\).
                
                \item \textbf{Subcase 1.2: $b_1$ and $b_2$ have opposite signs.} Without loss of generality, assume $b_1 > 0$ and $b_2 < 0$. As $y \to \infty$, $a_1^\top x + b_1 y \to \infty$ and $a_2^\top x + b_2 y \to -\infty$. The left-hand side of \eqref{eq:lambda_equality} is asymptotic to:
                \[
                    b_1 \lambda(\text{large pos}) - b_2 \lambda(\text{large neg}) \sim b_1(0) - b_2(-(a_2^\top x + b_2 y)) = b_2 a_2^\top x + b_2^2 y\,.
                \]
                This diverges as \(y\to\infty\) (since $b_2 \neq 0$) and, hence, cannot be constant -- leading to a contradiction.
            \end{itemize}

\paragraph{Case 2: Exactly one of $b_1, b_2$ is zero.}
            Without loss of generality, assume $b_1 \neq 0$ and $b_2=0$. \cref{eq:lambda_equality} becomes:
            \[
                b_1 \lambda(a_1^\top x + b_1 y) = -\frac{1}{\sigma^2}\langle w_1 - w_2, x \rangle\,.
            \]
            The function $\lambda(z)$ is strictly decreasing. Therefore, the left-hand side is a non-constant function of $y$, while the right-hand side is constant in $y$. This is a contradiction.
            
\paragraph{Case 3: $b_1 = 0$ and $b_2=0$.}
                \cref{eq:lambda_equality} forces $w_1=w_2$ since it holds for all $x\in B$ and \(B\) is a nonempty open set, and the assumed
              kernel equality reduces to $\Phi(a_1^\top x)=\Phi(a_2^\top x)$ for all
              $x\in B$, implying $a_1=a_2$.  The equality \(w_1=w_2\) contradicts the choice of \(f_1,f_2\). 
            
            \medskip
            
            \noindent Thus, in each case, we arrive at a contradiction and, hence, \cref{cond:identifiability:f} must be satisfied.

        \end{proof}

    \subsubsection{Ahn--Powell's Sample-Selection Model}

    We characterize the joint concept class for the Ahn--Powell model.
    
    For a parameter tuple $\theta=(f,s,\cD,\cQ)$ specifying the sample-selection model in \cref{def:selection_bias_model}, suppose that $\cQ$ admits a density, also denoted by $\cQ$, and that $f(x)=w^\top x$. Define its observed kernel, induced propensity score, and selected conditional mean by
    \begin{align*}
        k_\theta(x,y)
        &\coloneqq s(x,y)\cQ\!\inparen{y-f(x)}\,,\\
        p_\theta(x)
        &\coloneqq \int_\R k_\theta(x,y)\,\nu(\d y)\,,\\
        m_\theta(x)
        &\coloneqq
        \frac{\int_\R yk_\theta(x,y)\,\nu(\d y)}{p_\theta(x)}
        \qquad\text{whenever }p_\theta(x)>0\,.
    \end{align*}
    Let $\Theta_{\rm AP}$ be the class of all such tuples for which $\cQ$ has a finite first moment and, for some smooth function $\lambda_\theta$,
    \[
        m_\theta(x)=f(x)+\lambda_\theta\!\inparen{p_\theta(x)}
    \]
    for $\cD$-almost every $x$ such that $p_\theta(x)>0$. Equivalently, $\E_\theta[\xi\mid X=x,D=1]=\lambda_\theta\!\inparen{p_\theta(x)}$, as in \cref{def:ahn_powell}, where the correction function is denoted by $g$.

    In the next theorem, we prove $\Theta_{\rm AP}$ satisfies \Cref{cond:identifiability:f}.
    \begin{theorem}\label{thm:ahn_identifiable}
        Fix a measurable instrument map $z:\cX\to\R^d$. For each $\theta\in\Theta_{\rm AP}$, define
        \[
            \overline{x}_\theta(t)
            \coloneqq
            \E_\theta\!\left[X\mid p_\theta(X)=t,D=1\right]
            \quad\text{and}\quad
            \overline{z}_\theta(t)
            \coloneqq
            \E_\theta\!\left[z(X)\mid p_\theta(X)=t,D=1\right].
        \]
        Suppose that the matrix \[
            \Gamma_\theta
            \coloneqq
            \E_\theta\!\left[
                p_\theta(X)^2
                \left(
                    z(X)-\overline{z}_\theta\!\left(p_\theta(X)\right)
                \right)
                \left(
                    X-\overline{x}_\theta\!\left(p_\theta(X)\right)
                \right)^\top
            \right]
        \label{eq:ahn-powell-rank-matrix}
    \]
    is well-defined and nonsingular. (This is
    Assumption~3.4 of \cite{Ahn1993semiparametric} with $g_i=p_\theta(X_i)$, $d_i=D_i$, and $z_i=z(X_i)$.) Then, $\Theta_{\rm AP}$ satisfies \cref{cond:identifiability:f}.
    \end{theorem}
    
    \begin{proof}
    Fix
    \[
        \theta_j=(f_j,s_j,\cD,\cQ_j)\in\Theta_{\rm AP},
        \qquad j\in\{1,2\},
    \]
    where $f_j(x)=w_j^\top x$, and suppose that their observed kernels
    agree $(\cD\otimes\nu)$-almost everywhere:
    \[
        k_{\theta_1}(x,y)=k_{\theta_2}(x,y).
    \]
    Integrating this equality over $y$ gives
    \[
        p_{\theta_1}(x)=p_{\theta_2}(x)
    \]
    for $\cD$-almost every $x$. Write the common propensity score as
    $p$. Because the noise distributions have finite first moments,
    multiplying the kernel equality by $y$ and integrating gives
    \[
        p_{\theta_1}(x)m_{\theta_1}(x)
        =
        p_{\theta_2}(x)m_{\theta_2}(x)
    \]
    for $\cD$-almost every $x$. Consequently,
    \[
        m_{\theta_1}(x)=m_{\theta_2}(x)
    \]
    for $\cD$-almost every $x$ such that $p(x)>0$.
    Let $\lambda_j$ be the correction function supplied by the index
    restriction for $\theta_j$. Since
    \[
        m_{\theta_j}(x)
        =
        w_j^\top x+\lambda_j\!\left(p_{\theta_j}(x)\right),
    \]
    the preceding equality implies
    \begin{equation}
        (w_1-w_2)^\top x
        =
        \left(\lambda_2-\lambda_1\right)\!\left(p(x)\right)
        \label{eq:ahn-index-equality}
    \end{equation}
    for $\cD$-almost every $x$ with $p(x)>0$.
    The joint law of $(X,D)$ is determined by $\cD$ and $p$ because
    \[
        \Pr(X\in A,D=1)
        =
        \int_A p(x)\,\cD(\d x).
    \]
    Thus, the two tuples induce the same conditional distribution of
    $X$ given $p(X)$ and $D=1$. In particular, we may use
    $\overline{x}_{\theta_1}$ as the common version of
    \[
        \E[X\mid p(X)=t,D=1].
    \]
    Taking the conditional expectation of
    \eqref{eq:ahn-index-equality} given $p(X)$ in the selected
    population and subtracting it from
    \eqref{eq:ahn-index-equality} yields
    \begin{equation}
        \left(
            X-\overline{x}_{\theta_1}\!\left(p(X)\right)
        \right)^\top(w_1-w_2)
        =0
        \qquad\text{almost surely given }D=1.
        \label{eq:ahn-residual-equality}
    \end{equation}
    The selected law of $X$ is equivalent to $\cD$ on $\{p>0\}$.
    Hence, \eqref{eq:ahn-residual-equality} holds
    $\cD$-almost everywhere on $\{p>0\}$. Multiplying it by
    \[
        p(X)^2
        \left(
            z(X)-\overline{z}_{\theta_1}\!\left(p(X)\right)
        \right)
    \]
    and taking expectations gives
    \begin{align*}
        \Gamma_{\theta_1}(w_1-w_2)
        &=
        \E_{\theta_1}\!\left[
            p(X)^2
            \left(
                z(X)-\overline{z}_{\theta_1}\!\left(p(X)\right)
            \right)
            \left(
                X-\overline{x}_{\theta_1}\!\left(p(X)\right)
            \right)^\top
        \right](w_1-w_2)\\
        &=0.
    \end{align*}
    The factor $p(X)^2$ makes the integrand vanish on $\{p=0\}$.
    Since $\Gamma_{\theta_1}$ is nonsingular, it follows that
    $w_1=w_2$, and therefore $f_1=f_2$. Thus, two tuples whose
    outcome regressors differ on a set of positive $\cD$-measure
    cannot induce the same observed kernel. This proves
    \Cref{cond:identifiability:f}.
\end{proof}

    \subsubsection{Censoring on Outcomes: Tobin Model and Variants}
    Models where revealment depends on the realized outcome, like Tobit, are captured by the following concept classes:

    \begin{align*}
        \hyF & = \inbrace{f(x)=\inangle{w,x}:w\in\R^d}\,,\\
        \hyS_{\rm out}
        & =
        \left\{
        (x,y)\mapsto \mathds{1}\{y\in S\}
        :
        S\subseteq\R\text{ measurable and }\nu(S)>0
        \right\}\,, \\
        \hyQ & = \inbrace{\cN(0,\sigma^2):\sigma^2>0} \,.
    \end{align*}
    
    In the following, we prove that these models also satisfy \cref{cond:identifiability:f}.
    
    \begin{theorem}[Outcome-only censoring identification]\label{thm:outcome_only_identifiable}
        Let $\hyF, \hyS_{\rm out}, \hyQ$ be the outcome-censored concept classes as defined above. 
        If \(\cD_X\) has full-dimensional support, let $\hyDX\coloneqq\{\cD_X\}$. Then,
        the Cartesian-product class $\Theta_{\rm out}\coloneqq\hyF\times\hyS_{\rm out}\times\hyDX\times\hyQ$ satisfies \cref{cond:identifiability:f}.
    \end{theorem}
    \begin{proof}
        Let \(f_1(x)=w_1^\top x\) and \(f_2(x)=w_2^\top x\) differ on a set of positive $\cD_X$-measure, so \(w_1\neq w_2\).
        Let \(s_i(y)=\mathds{1}\{y\in S_i\}\) with \(\nu(S_i)>0\), and let
        \(\cQ_i=\cN(0,\sigma_i^2)\) for \(i=1,2\). Suppose toward a contradiction
        that the observed kernels agree \((\cD_X\otimes\nu)\)-a.e.:
        \[
            \mathds{1}\{y\in S_1\}\,\varphi_{\sigma_1}(y-w_1^\top x)
            =
            \mathds{1}\{y\in S_2\}\,\varphi_{\sigma_2}(y-w_2^\top x)\,.
        \]
        Because Gaussian densities are strictly positive everywhere, this equality
        implies \(S_1=S_2\) \(\nu\)-a.e. Let \(S\) denote this common set, which has
        positive \(\nu\)-measure. Hence, for \(\cD_X\)-a.e. \(x\),
        \[
            \varphi_{\sigma_1}(y-w_1^\top x)
            =
            \varphi_{\sigma_2}(y-w_2^\top x)
            \qquad
            \nu\text{-a.e. } y\in S\,.
        \]
        Since both sides are real analytic functions of \(y\), equality on a set of
        positive Lebesgue measure implies equality for all \(y\in\R\). Therefore,
        for \(\cD_X\)-a.e. \(x\), the two Gaussian densities are identical, which
        forces
        \[
            w_1^\top x=w_2^\top x
            \quadand\quad
            \sigma_1=\sigma_2\,.
        \]
        Thus \(f_1=f_2\) \(\cD_X\)-a.e., contradicting the assumption that they
        differ on a set of positive \(\cD_X\)-measure.
    \end{proof}
    
    More generally, the censoring (as introduced in \cref{ex:outcome-only}) can also depend jointly on covariates and outcomes via an \emph{unknown} survival set $S^\star\subseteq\cX\times\R$,
        \[
        d  =  \mathds{1}\!\inbrace{(X,Y)\in S^\star}\quad\Longleftrightarrow\quad s^\star(x,y)=\mathds{1}\!\inbrace{(x,y)\in S^\star}\,.
        \]
    This broader class is captured by \cref{def:selection_bias_model}, but
	    \cref{cond:identifiability:f} need not hold for arbitrary measurable
	    \(S^\star\), and additional structure on the survival sets is necessary. Recent work \citep{Kontonis2019EfficientTS,lee2024unknown} develops algorithms and identifiability under mild complexity assumptions on $S^\star$ (\eg{}, that it has a low Gaussian surface area).

\subsection{Revisiting the Identification Condition of \citet{das2003nonparametric} }
\label{sec:revisiting:dvn} 
Since \citet{heckman1979sample}, a large body of work has progressively relaxed the parametric assumptions of the classical sample-selection model~\citep{lee1983genecon, Ahn1993semiparametric}. Among the most general contributions is \citet{das2003nonparametric}, who allow both the regression function and the selection mechanism to be nonparametric.
In this section, we embed the model of \citet{das2003nonparametric} into our framework and compare their sufficient conditions for identifying the regression function with \cref{cond:identifiability:f}. We verify that \cref{cond:identifiability:f} holds for the corresponding joint concept class, and we conclude with an outcome-dependent truncation example that violates the scalar-index restriction of \citet{das2003nonparametric} but remains identifiable under our characterization.
\par
We begin with the model studied by \citet{das2003nonparametric} with homoskedastic noise.\footnote{\citet{das2003nonparametric} also analyze other structures (\eg{}, multi-index and multi-outcome). We focus on the single-rule case as it already captures the main connection to our framework; the other cases can be analogously analyzed.}
\Needspace{6\baselineskip}
\begin{definition}[Das--Newey--Vella model]\label{def:dvn-model}
Let $x=(w,z)\in\cX\subseteq \R^d$ be observed covariates, where $z\in\R$, let $D\in\zo$ be a selection indicator, and let $\xi$ be an unobserved disturbance. The latent outcome is $y^\star = f^\star(x) + \xi$, where $f^\star:\R^d\to\R$ is an unknown smooth regression function satisfying $f^\star(w,z)=g^\star(w)$ for some unknown continuously differentiable function $g^\star$. The observed data consist of \iid{} samples of $(X,Y)$, where $(X,Y)\coloneqq (X,Y^\star)$ if $D=1$ and $(X,Y)\coloneqq (X,\emptyset)$ if $D=0$.
\end{definition}
\citet{das2003nonparametric} identify $\fstar$ under the following assumption. 
\par
\begin{assumption}\label{def:dvn-id-restrictions}
    Define the propensity score $p(x)\coloneqq \Pr(D=1\mid X=x)$.
    Assume that $\mathrm{int}(\cX)$ is nonempty and connected, that $\cX=\overline{\mathrm{int}(\cX)}$, that $\supp(\cD_X)=\cX$, that $\cD_X$ has a continuous distribution function, and that $p$ is continuously differentiable.\footnote{Here, $\mathrm{int}(\cX)$ denotes the interior of $\cX$. To the best of our understanding, the proof of \citet{das2003nonparametric} implicitly requires the interior of the support to be connected and dense in the support, although this is not stated explicitly. We make these requirements explicit here.}
    There is a continuously differentiable function $\lambda$ such that, for the decomposition $x=(w,z)$ fixed in \cref{def:dvn-model}, the following hold:
\begin{enumerate}[leftmargin=20pt,itemsep=2pt]
    \item \textbf{(Index restriction)} For $\cD_X$-a.e. $x$, $\E\insquare{Y \mid X=x, D=1} = \fstar(x)+\lambda\!\inparen{p(x)}$.
    \item \textbf{(Rank condition)} For each $x=(w,z)\in\supp(\cD_X)$, it holds that $\frac{\partial}{\partial z}\,p(w,z)\neq 0$.
    \item \textbf{(Positivity)} For $\cD_X$-almost every $x$, $p(x)>0$.  
    \item \textbf{(Constant)} It holds that $f^\star(x_0)=0$ for a known $x_0\in \cX$.
\end{enumerate}
\end{assumption}
The index restriction is the main modeling assumption: the selection-induced bias $\E[Y \mid X=x, D=1] - f^\star(x)$ depends on $x$ only through the scalar $p(x)$. The rank condition ensures that $p(\cdot)$ varies with (at least some of) the features; together with the exclusion of $z$ from $f^\star$, this enables the separation of $f^\star(x)$ from $\lambda(p(x))$.
Positivity is required to identify $\E[Y\mid X=x, D=1]$ for $\cD_X$-almost every $x$.
Without the ``constant'' restriction they can only identify $\fstar$ up to an additive constant.
\par
We capture the model of \citet{das2003nonparametric} in our framework using the following joint concept class.
For the common support $\cX$, decomposition $x=(w,z)$, and normalization point $x_0$ above, let $\Theta_{\rm DNV}$ be the class of all tuples $(f,s,\cD,\cQ)$ specifying the sample-selection model in \cref{def:selection_bias_model} such that $\cQ$ has a density and, after setting $f^\star=f$ and $\cD_X=\cD$, the induced instance is of the form in \cref{def:dvn-model} and satisfies \cref{def:dvn-id-restrictions}.

\begin{lemma}[Embedding]\label{lem:dvn-embed}
Every instance of the model in \Cref{def:dvn-model} satisfying \Cref{def:dvn-id-restrictions} with $\xi\perp X$ and $\xi$ having a density corresponds to an instance of \cref{def:selection_bias_model} with $(f^\star,s^\star,\cD_X,\cQ^\star)\in\Theta_{\rm DNV}$, where $\cQ^\star$ is the law of $\xi$ and $s^\star(x,y)\coloneqq\Pr(D=1\mid X=x,Y^\star=y)$ is a measurable version.
\end{lemma}
As a sanity check, we show \cref{cond:identifiability:f} is indeed satisfied (see \Cref{sec:das_indentifiable} for a proof).
\begin{theorem}\label{thm:das_indentifiable}
The joint concept class $\Theta_{\rm DNV}$ satisfies \cref{cond:identifiability:f}. 
\end{theorem}
The sufficient conditions of \citet{das2003nonparametric} require that the selection-induced bias in the conditional mean can be written as a function of the scalar propensity score $p(x)$. In particular, if two covariate values $x_1$ and $x_2$ have the same selection probability, $p(x_1)=p(x_2)$, then the bias must also be the same. This limits the range of selection mechanisms their framework can accommodate: it excludes settings where selection keeps \emph{different regions} of the outcome distribution at different covariate values, even when the overall acceptance rate is unchanged. As a result, there are identifiable instances that their conditions fail to recognize. 
Such outcome-dependent selection is common in physical-science measurement pipelines, where an instrument records a signal only if it falls in an acceptance region determined by detection thresholds, trigger rules, saturation windows, or sign-dependent clipping (see \cref{ex:intro:sciences} in \cref{sec:intro}).

\begin{example}[Survival-set truncation]\label{ex:trunc-identifiable}
Let $d=1$, $\cX\coloneqq \zo$ and consider the concept classes
\begin{align*}
\hyF_{\rm T} \coloneqq \hyFlin\,,\qquad
\hyQ_{\rm T} \coloneqq \inbrace{\cN(0,1)}\,,\qquad
\hyD_{\rm T} \coloneqq \inbrace{\cDX \;:\; \supp(\cDX)=\cX}\,,
\end{align*}
with a family of survival-set selectors
\[
\hyS_{\rm T}
\coloneqq
\inbrace{
s_S(x,y)=\mathds{1}\!\inbrace{(x,y)\in S} \;\Big|\; S\subseteq \R^{d+1},~S\cap \inbrace{x}\times \R \text{ is nonempty and open for each $x$}
}\,.
\]
That is, for each feature value $x$, the instrument records outcomes on a nontrivial range of $y$ values.

The following two lemmas show that the sufficient conditions of \citet{das2003nonparametric} fail for these concept classes, while \cref{cond:identifiability:f} holds.
\begin{lemma}\label{lem:trunc-dvn-fails}
The index restriction in \Cref{def:dvn-id-restrictions} fails for some instances arising from $\inparen{\hyF_{\rm T}, \hyS_{\rm T}, \hyD_{\rm T}, \hyQ_{\rm T}}$. %
\end{lemma}
Nevertheless, the above example remains identifiable:
\begin{lemma}\label{lem:trunc-cond1}
The concept classes $\inparen{\hyF_{\rm T}, \hyS_{\rm T}, \hyD_{\rm T}, \hyQ_{\rm T}}$ satisfy \cref{cond:identifiability:f}.
\end{lemma}
\end{example} 
\noindent This example isolates a qualitative gap between the two approaches: two truncation rules can yield the same acceptance rate $p(x)$ while inducing opposite biases in the selected sample. Such mechanisms are natural in physical-science data acquisition, where ``trigger'' or ``detection'' pipelines adjust thresholds to maintain a target event rate, and different instrument configurations can select opposite tails of a signal while achieving similar overall yield. In these settings, the propensity score does not summarize the selection effect, and the conditions of \citet{das2003nonparametric} fail even when the selected data remains informative. Our framework handles such mechanisms directly through outcome-dependent revealment $s(x,y)$, and \cref{cond:identifiability:f} characterizes precisely when identification is possible.

\subsection{Bounded-Friction Participation Costs}\label{sec:boundedfriction}
In this section, we show an economically meaningful setup where our identification condition (\Cref{cond:identifiability:f}) is satisfied but both \cref{def:dvn-id-restrictions} of \citet{das2003nonparametric} and the sufficient conditions of \citet{gentry_li_2014} fail. To begin, consider participation rules of the form
\[
    s(x,y)=G\!\inparen{\rho(x)y-h(x,y)}\,,
\]
where $G:\R\to[0,1]$ is a known nondecreasing link satisfying $G(t)\to1$ as $t\to\infty$, $\rho(x)$ is a positive responsiveness parameter, and $h(x,y)$ is a bounded participation friction. With $G=\Phi$ and an appropriately structured index, this contains familiar probit/Heckman-style selection rules as a special case. The bounded friction $h$ may also encode local obstacles such as taxes, audit risk, compliance costs, or auction entry costs. Importantly, because $h$ may depend jointly on $(x,y)$, participation need not be globally monotone in $y$.

\begin{example}[Bounded-Friction Participation Costs]\label{ex:bounded-friction-participation}
Fix constants $0<\rho_L\leq \rho_U<\infty$ and $C<\infty$. Let $G:\R\to[0,1]$ be nondecreasing with $\lim_{t\to\infty}G(t)=1$. Define
\[
\hyS_{\rm BF}
\coloneqq
\left\{
\begin{array}{l}
 s_{\rho,h}:\cX\times\R\to[0,1]\,,\quad
 s_{\rho,h}(x,y)=G\!\inparen{\rho(x)y-h(x,y)}\,,\\[1mm]
 \rho:\cX\to[\rho_L,\rho_U]\,,\quad
 h:\cX\times\R\to[-C,C]
\end{array}
\right\}
\]
and take
\[
    \hyF_{\rm BF}\coloneqq \hyFlin\,,
    \qquad
    \hyQ_{\rm BF}\coloneqq\inbrace{\cN(0,1)}\,.
\]
Finally, let \(\hyD_{\rm BF}\) be any class of covariate distributions that
separate distinct linear functions in \(L_1(\cDX)\): for every
\(\cDX\in\hyD_{\rm BF}\) and every \(f,f'\in\hyFlin\),
$\cDX(\{x:f(x)\neq f'(x)\})>0$ 
whenever \(f\neq f'\) as linear functions. In particular, full-dimensional
support is sufficient.
\end{example}
The following lemma shows that \Cref{ex:bounded-friction-participation} is identifiable.
\begin{lemma}\label{lem:bounded_friction}
    The Cartesian-product class $\Theta_{\rm BF}\coloneqq\hyF_{\rm BF}\times\hyS_{\rm BF}\times\hyD_{\rm BF}\times\hyQ_{\rm BF}$ satisfies \cref{cond:identifiability:f}.
\end{lemma} 
\begin{proof}
Let $\varphi$ denote the standard normal density. Suppose, toward a contradiction, that \cref{cond:identifiability:f} fails. Then there exist distinct $f,f'\in\hyFlin$ (differing on a set of positive $\cDX$-measure), selectors $s,s'\in\hyS_{\rm BF}$, and $\cDX\in\hyD_{\rm BF}$ such that the observed kernels agree \((\cDX\otimes\nu)\)-a.e.:
\begin{align*}
    s(x,y)\varphi\!\inparen{y-f(x)}
    =
    s'(x,y)\varphi\!\inparen{y-f'(x)}\,.
\end{align*}
By Fubini's theorem, for \(\cDX\)-a.e. \(x\), the equality holds for \(\nu\)-a.e. \(y\). Fix such an \(x\). Since $s,s'\in\hyS_{\rm BF}$, there are functions $\rho,\rho'$ and $h,h'$ with $\rho(x),\rho'(x)\geq\rho_L>0$ and $|h(x,y)|,|h'(x,y)|\leq C$. Therefore
\[
    \rho(x)y-h(x,y)\to\infty
    \quad\text{and}\quad
    \rho'(x)y-h'(x,y)\to\infty
    \qquad \text{as } y\to\infty\,,
\]
so $s(x,y)\to1$ and $s'(x,y)\to1$. In particular, $s(x,y)>0$ for all sufficiently large $y$. For such $y$, since the displayed equality holds for \(\nu\)-a.e. \(y\), it holds along a sequence
\(y_n\to\infty\):
\[
    \frac{\varphi\!\inparen{y-f(x)}}{\varphi\!\inparen{y-f'(x)}}
    =
    \frac{s'(x,y)}{s(x,y)}
    \rightarrow 1\,.
\]
Writing $a=f(x)$ and $b=f'(x)$, the Gaussian likelihood ratio equals
\[
    \frac{\varphi(y-a)}{\varphi(y-b)}
    =
    \exp\!\inparen{(a-b)y-\frac{a^2-b^2}{2}}\,.
\]
This expression converges to $1$ as $y\to\infty$ only if $a=b$. Hence \(f(x)=f'(x)\) for \(\cDX\)-a.e. \(x\), contradicting \(\cDX(\{x:f(x)\neq f'(x)\})>0\). Therefore \cref{cond:identifiability:f} holds.
\end{proof}

\paragraph{Why this goes beyond scalar-index selection corrections and global monotonicity.}
The proof uses only the right-tail saturation $s(x,y)\to1$, not a scalar-index restriction on the selected conditional mean. Thus the example contains instances that are identifiable by \cref{thm:characterization:f} but violate the sufficient conditions in \cref{def:dvn-id-restrictions}. To see this concretely, take a symmetric link satisfying $G(-t)=1-G(t)$, such as the logistic link or the Gaussian CDF $\Phi$, let $Y\sim\cN(0,1)$ be independent of $X$ and $f\equiv0$. Suppose $C>0$ and that some $\cD_X\in\hyD_{\rm BF}$ assigns positive probability to two disjoint measurable sets $A_0,A_1\subseteq\cX$. Fix $\rho\in[\rho_L,\rho_U]$ and $\epsilon\in(0,C]$, and define
\[
    s(x,y)=G\!\inparen{\rho y-\epsilon\bigl(\mathds{1}\{x\in A_0\}-\mathds{1}\{x\in A_1\}\bigr)\tanh(y)}
\]
This selector belongs to $\hyS_{\rm BF}$. Next, we can show that the propensity is equal on both groups. Namely, write
\[
    z_0(y)\coloneqq \rho y-\epsilon\tanh y\,,
    \qquad
    z_1(y)\coloneqq \rho y+\epsilon\tanh y\,.
\]
Since $\tanh$ is odd, both $z_0$ and $z_1$ are odd:
\[
    z_j(-y)=-z_j(y)\,,
    \qquad j\in\{0,1\}\,.
\]
Using the symmetry of the link, $G(-t)=1-G(t)$, and the symmetry of the
standard Gaussian density, $\varphi(-y)=\varphi(y)$, we obtain
\[
\begin{aligned}
    p(x)
    &=
    \E[s(x,Y)]
     =
    \int_{-\infty}^{\infty} G(z_j(y))\varphi(y)\,\d y \\
    &=
    \int_0^\infty
    \Big\{
        G(z_j(y))\varphi(y)
        +
        G(z_j(-y))\varphi(-y)
    \Big\}\,\d y \\
    &=
    \int_0^\infty
    \Big\{
        G(z_j(y))
        +
        G(-z_j(y))
    \Big\}\varphi(y)\,\d y \\
    &=
    \int_0^\infty \varphi(y)\,\d y
     =
    \frac12\,.
\end{aligned}
\]
This holds for every \(x \in A_j\). Therefore,

\[ p(x)=\frac12 \qquad \text{for every }x\in A_0\cup A_1 \,. \]

Even though the propensity scores are equal, the selected mean biases differ. Indeed, for every $x_0 \in A_0$, $x_1 \in A_1$,
\[
    \E\!\insquare{Y s(x_1,Y)}-\E\!\insquare{Y s(x_0,Y)}
    =
    \E\!\insquare{Y\Big(G(\rho Y+\epsilon\tanh Y)-G(\rho Y-\epsilon\tanh Y)\Big)}
    >0\,,
\]
whenever $G$ is strictly increasing. The inequality follows because the integrand is an even nonnegative function and is strictly positive on a set of positive measure. More precisely, write
\[
    z_+(y)\coloneqq \rho y+\epsilon\tanh y\,,
    \qquad
    z_-(y)\coloneqq \rho y-\epsilon\tanh y\,.
\]
Since $\tanh$ is odd, $z_+(-y)=-z_+(y)$ and $z_-(-y)=-z_-(y)$.
Using the symmetry $G(-t)=1-G(t)$,
\[
\begin{aligned}
    G(z_+(-y))-G(z_-(-y))
    &=
    G(-z_+(y))-G(-z_-(y)) \\
    &=
	    (1-G(z_+(y)))-(1-G(z_-(y))) \\
    &=
    -(G(z_+(y))-G(z_-(y)))\,.
\end{aligned}
\]
Therefore the above term is odd in $y$. Multiplying by $y$, which is also
odd, and by the Gaussian density $\varphi(y)$, which is even, shows that
\[
    y(G(z_+(y))-G(z_-(y)))\varphi(y)
\]
is even.

Moreover, $z_+(y)-z_-(y)=2\epsilon\tanh y$. Hence, for $y>0$, the bracketed
term is nonnegative, while for $y<0$ it is nonpositive. Multiplication by $y$
therefore makes the whole integrand nonnegative. If $G$ is strictly increasing
and $\epsilon>0$, the integrand is strictly positive for $y\neq0$, and hence
strictly positive on a set of positive Gaussian measure.

Since \(p(x)=1/2\) on both \(A_0\) and \(A_1\) but the selected mean bias differs between the two groups, the bias cannot be written as a function of \(p(x)\) alone. Therefore, there is no function $\lambda$ for which
\[
    \E[Y\mid X=x,D=1]=f(x)+\lambda(p(x))
\]
holds \(\cD_X\)-a.e. on \(A_0\cup A_1\). Thus the Das--Vella--Newey scalar-index condition fails, even though \cref{lem:bounded_friction} verifies our identification condition.

The same construction also illustrates why monotonicity of participation in the realized payoff is not necessary for identification. If $G$ is differentiable and strictly increasing and $\epsilon>\rho$, then the restriction of the selector to $A_0$ is locally decreasing at $y=0$, because the derivative of its index is $\rho-\epsilon\operatorname{sech}^2(y)$, which is negative at zero. Hence local hurdles can violate monotone-entry restrictions of the kind often imposed in auction-entry models \citep[e.g.,][]{gentry_li_2014}, while the regression function remains identifiable under \cref{cond:identifiability:f}.

\subsection{Characterizations of Related Models: Extended Roy and Auctions with Entry Costs}
    \label{sec:app:otherModels}

We now apply our framework to two well-studied econometric models: the extended Roy model and auctions with entry costs. For these models, previously only sufficient conditions for identification were known.
In both cases, each admissible model induces an instance of \cref{def:selection_bias_model}, and our characterization (\cref{cond:identifiability:f}) immediately yields necessary and sufficient conditions for identifiability over the resulting joint concept class, provided its noise components admit densities.

\subsubsection{Extended Roy Model}
\label{sec:extended-roy}
Roy-style models describe self-selection into occupations or treatments when agents choose among alternatives based on latent potential outcomes \citep{roy1951earnings}. A key difficulty in extended versions, such as that of \citet{heckman_sedlacek_1985}, is that the selection decision depends on unobserved components of utility, so only the chosen outcome is informative for learning the corresponding regression function. A large literature studies identification under additional structure; see, \eg{}, \citep{lee_park_2023} and references therein.
We work with the following formulation by \citet{lee_park_2023}. 

\newcommand{\cU}{\mathcal{U}}
\begin{definition}[Extended Roy model \citep{lee_park_2023}]
\label{def:extended-roy}
Let $X\in\cX\subseteq\R^d$ be observed covariates.
There are two latent potential outcomes,
\[
    Y_1^\star \coloneqq f_1^\star(X)+\xi_1
    \qquadand 
    Y_0^\star \coloneqq f_0^\star(X)+\xi_0 \,,
\]
where $f_1^\star,f_0^\star:\cX\to\R$ are unknown regression functions and $(\xi_1,\xi_0)$ are unobserved disturbances independent of $X$.
Each alternative $d\in\zo$ yields utility $\mathcal{U}_d\sinparen{X,Y_d^\star,\eta_d}$, where $\eta_d$ is an unobserved component (\eg{}, amenities or match quality).
The agent chooses
\[
     D\coloneqq \argmax\nolimits_{d\in\zo}\mathcal{U}_d\sinparen{X,Y_d^\star,\eta_d}\,,
\]
and the econometrician observes the chosen outcome $Y \coloneqq Y_D^\star$ together with $(X,D)$.
\end{definition}
The model naturally decomposes into two regression problems: learning $f_1^\star$ from the subsample with $D=1$, and learning $f_0^\star$ from the subsample with $D=0$. We show that each reduces to an instance of our sample-selection model.

\begin{proposition}[Extended Roy reduces to sample-selection regression]
\label{prop:roy-embed}
Consider the extended Roy model in \cref{def:extended-roy}. Define the regime-$1$ censored outcome
$Y^{(1)}\coloneqq Y$ if $D=1$ and $Y^{(1)}\coloneqq \emptyset$ otherwise. Then $(X,Y^{(1)})$ is an instance of
\cref{def:selection_bias_model} with regression function $f_1^\star$, covariate law $\cD_X$ equal to the marginal
distribution of $X$, noise law $\cQ_1^\star$ equal to the law of $\xi_1$, and selection function
\[s_1^\star(x,y)\coloneqq \Pr(D=1\mid X=x,\ Y_1^\star=y)\,.\]
An analogous statement holds for regime~$0$.
\end{proposition} 
\begin{proof}[Proof of \cref{prop:roy-embed}]
    
\cref{def:extended-roy} defines two latent outcome equations, $Y_1^\star=f_1^\star(X)+\xi_1$ and $Y_0^\star=f_0^\star(X)+\xi_0$, and an occupation choice rule that compares the two utilities. This naturally gives rise to two regression problems: learning $f_1^\star$ from the subsample with $D=1$, and learning $f_0^\star$ from the subsample with $D=0$. Without loss of generality, we focus on learning $f_1^\star$.

To make the reduction explicit, write the structural choice rule as
\[
D
=
\mathds{1}\!\inbrace{
\cU_1\sinparen{X,Y_1^\star,\eta_1}
\geq 
\cU_0\sinparen{X,Y_0^\star,\eta_0}
}\,,
\qquad
\text{where}\quad
Y_0^\star=f_0^\star(X)+\xi_0\,.
\]
For fixed $(X,Y_1^\star)$, all remaining randomness in the decision is carried by the latent variables
$W\coloneqq (\xi_0,\eta_0,\eta_1)$, so there exists a measurable function $\phi$ such that
\[
D=\phi\!\inparen{X,Y_1^\star,W}
\qquad
\text{with}\quad
\phi(x,y,w)\in\zo\,.
\]
This implies that the selection rule for regime $1$, viewed as a function of $(X,Y_1^\star)$, is summarized by the conditional selection probability
\[
s_1^\star(x,y)
\coloneqq
\Pr\!\inparen{D=1 \mid X=x,\ Y_1^\star=y}
=
\E\!\insquare{\phi\!\inparen{x,y,W}\mid X=x,\ Y_1^\star=y}\,.
\]
\noindent Introduce an auxiliary $U\sim \mathrm{Unif}\inparen{0,1}$ independent of all model variables and define
$\wt D\coloneqq \mathds{1}\sinbrace{U \le s_1^\star\!\inparen{X,Y_1^\star}}$.
Then $\Pr\sinparen{\wt D=1\mid X,Y_1^\star}=s_1^\star(X,Y_1^\star)=\Pr\sinparen{D=1\mid X,Y_1^\star}$, so $\wt D$ is observationally equivalent to $D$ once we condition on $(X,Y_1^\star)$.
Consequently, the regime-$1$ learning problem is an instance of \cref{def:selection_bias_model} with latent outcome $Y_1^\star=f_1^\star(X)+\xi_1$, covariate law $\cD_X$ equal to the marginal distribution of $X$, noise law $\cQ_1^\star$ equal to the law of $\xi_1$, and selection function $s_1^\star(x,y)$.

Concretely, from the observed triples $(X,D,Y)$ one forms the censored outcome $Y^{(1)}\coloneqq Y$ if $D=1$ and $Y^{(1)}\coloneqq \emptyset$ otherwise; then $(X,Y^{(1)})$ has exactly the censoring structure in \cref{def:selection_bias_model}. The reduction for learning $f_0^\star$ is symmetric.
\end{proof} 
\noindent Once the extended Roy model is embedded via \cref{prop:roy-embed}, identifiability of $f_1^\star$ (or $f_0^\star$) is characterized by \cref{cond:identifiability:f} applied to the corresponding joint concept class of admissible regime-specific parameter tuples. Recent work by \citet{lee_park_2023} provides sufficient conditions for identification in this model under monotonicity requirements on how certain instruments shift the selection rule. Our characterization does not require such monotonicity in general. It would be interesting to explore whether this yields identification in empirically relevant regimes not covered by the conditions of \citet{lee_park_2023}.

\subsubsection{Auctions with Entry Costs}
\label{sec:auctions-entry-costs}
 
In auctions with endogenous participation, a potential bidder must incur a cost before submitting a bid, so the auctioneer observes bids only from entrants. This creates sample selection bias when using observed bids to infer bidder values; see, \eg{}, \citep{samuelson1985entry,mcafee1987entry,levin1994entry,gentry_li_2014,gentry2017selective}.

\begin{definition}[Auctions with entry costs]
\label{def:auction-entry-costs}
Fix a feature space $\cX\subseteq\R^d$ describing bidder and/or auction covariates.
A (bidder, auction) instance has covariates $X\in\cX$ and a latent value
$Y^\star \coloneqq f^\star(X)+\xi$, where $f^\star:\cX\to\R$ is the target regression function and $\xi$ is an unobserved disturbance independent of $X$.
The bidder enters if and only if $D=1$, where $D\in\zo$ is an entry indicator that may depend on $(X,Y^\star)$ (\eg{}, through an entry rule of the form $\mathds{1}\!\sinbrace{Y^\star \geq c(X,Y^\star)}$).
The econometrician observes \iid{} samples of $(X,Y)$, where $(X,Y)\coloneqq (X,Y^\star)$ if $D=1$ and $(X,Y)\coloneqq (X,\emptyset)$ if $D=0$.
\end{definition}
This is a direct instance of \cref{def:selection_bias_model}: the latent value $Y^\star = f^\star(X)+\xi$ plays the role of the outcome, and the entry decision induces a selection function $s^\star(x,y) \coloneqq \Pr(D=1 \mid X=x, Y^\star=y)$. Once one specifies a joint concept class of admissible parameter tuples, identifiability is characterized by \cref{cond:identifiability:f}.

In this setting, \citet{gentry_li_2014} provide sufficient conditions for identification using exclusion restrictions and monotonicity conditions that link entry signals to values. Our characterization weakly generalizes such conditions; understanding when this generalization is meaningful in empirically relevant auction environments is an interesting direction for future work.

\subsection{Instrumental Variables and the Estimation Condition}
\label{sec:iv-estimation}
This section gives a sufficient, quantitative version of the instrumental-variable intuition behind semi-parametric selection estimators such as \citet{Ahn1993semiparametric}. We ask when an index restriction, combined with instrumental variable assumptions, forces two separated linear regressors to induce separated observed-data laws. In the language of \cref{cond:estimation:f}, this yields an explicit rate function.
For notational simplicity we state the result for linear regressors
\[
    \hyF_{\rm IV}
    \coloneqq
    \inbrace{f_\beta(x)=\inangle{\beta,x}:\beta\in\R^d}\,.
\]
Let $\Theta_{\rm IV}\subseteq\hyF_{\rm IV}\times\hyS_{\rm IV}\times\hyD_{\rm IV}\times\hyQ_{\rm IV}$ be a joint concept class of admissible parameter tuples.
The same statement applies when the observed covariates decompose as $x=(w,z)$, only $w$ enters the outcome equation, and $z$ plays the role of an excluded instrument: replace $x$ below by the outcome regressors $w$ and let the selection index depend on the full vector $(w,z)$.
Given a candidate tuple $\pi=(f_\beta,s,\cQ)$ and a covariate law $\cD$, define its observed kernel
\[
    k_\pi(x,y)
    \coloneqq
    s(x,y)\,\cQ (y-f_\beta(x))\,.
\]
Thus
\[
    p_\pi(x)
    \coloneqq
    \int k_\pi(x,y)\,\d y
    = \Pr_\pi(D=1\mid X=x)
\]
is the propensity score induced by $\pi$, and
\[
    \mu_\pi(x)
    \coloneqq
    \E_\pi[Y\mid X=x,D=1]
    =
    \frac{\int y\,k_\pi(x,y)\,\d y}{p_\pi(x)}
\]
is the conditional mean of observed outcomes. The following assumptions are analogous to the usual index and rank restrictions seen in \citet{das2003nonparametric,Ahn1993semiparametric}.
\begin{assumption}[Instrumental Variable Restrictions]
\label{asmp:iv-estimation}
Fix constants $\eta\in(0,1]$, $M<\infty$, $\rho>0$, $c_{\rm iv}>0$. For every ordered pair $\theta=(f_\beta,s,\cD,\cQ)$ and $\theta'=(f_{\beta'},s',\cD',\cQ')$ in $\Theta_{\rm IV}$, writing $\pi=(f_\beta,s,\cQ)$ and $\pi'=(f_{\beta'},s',\cQ')$, the following hold with $X\sim\cD$.
\begin{enumerate}[leftmargin=22pt,itemsep=2pt]
    \item \textbf{Index restriction.}  There are functions $\lambda_\pi,\lambda_{\pi'}:[\eta,1]\to\R$ such that, for $\cD$-a.e. $x$,
    \[
        \mu_\pi(x)=\inangle{\beta,x}+\lambda_\pi\!\inparen{p_\pi(x)}
        \quadand
        \mu_{\pi'}(x)=\inangle{\beta',x}+\lambda_{\pi'}\!\inparen{p_{\pi'}(x)}\,.
    \]
    \item \textbf{Overlap.}  The propensities are uniformly bounded away from zero:
    \[
        p_\pi(x)\geq \eta
        \quadand
        p_{\pi'}(x)\geq \eta
        \qquad\text{for $\cD$-a.e. }x\,.
    \]
    \item \textbf{Moment bound.}  The selected second moments are uniformly bounded:
    \[
        \iint y^2 k_\pi(x,y)\,\d y\,\d\cD(x)\leq M
        \quadand
        \iint y^2 k_{\pi'}(x,y)\,\d y\,\d\cD(x)\leq M\,.
    \]
    In addition, the covariate first moment is uniformly bounded:
    \[
        L_{\rm IV}
        \coloneqq
        \sup_{\cD\in\hyD_{\rm IV}}
        \E_{X\sim\cD}\norm{X}_2
        <\infty\,.
    \]
    \item \textbf{Instrument rank.}  Let $P_\pi\coloneqq p_\pi(X)$ and $P_{\pi'}\coloneqq p_{\pi'}(X)$.  For a conditional law of $X$ given $(P_\pi,P_{\pi'})=(u,v)$, write
    \[
        \Sigma_{\pi,\pi'}(u,v)
        \coloneqq
        \cov\!\inparen{X\mid P_\pi=u,\ P_{\pi'}=v}\,.
    \]
    For almost every $(u,v)$ in the support of $(P_\pi,P_{\pi'})$,
    \[
        \lambda_{\min}\!\inparen{\Sigma_{\pi,\pi'}(u,v)}\geq \rho\,.
    \]
    \item \textbf{Directional anti-concentration.}  If $X_1,X_2$ are independent draws from the same conditional law of $X$ given $(P_\pi,P_{\pi'})=(u,v)$, then for every $a\in\R^d$ and almost every $(u,v)$,
    \[
        \E\!\insquare{\abs{a^\top(X_1-X_2)}\mid P_\pi=u,\ P_{\pi'}=v}
        \geq
        c_{\rm iv}\,\norm{a}_2\sqrt{\lambda_{\min}\!\inparen{\Sigma_{\pi,\pi'}(u,v)}}\,.
    \]
\end{enumerate}
\end{assumption}
The index restriction makes the selection bias constant after conditioning on the propensity score. The overlap and moment bound assumptions are necessary to obtain an explicit rate function. The rank and anti-concentration assumptions say that, after conditioning on the selection indices, the regressors still contain nontrivial variation in every direction. In the following theorem, we show that these assumptions are sufficient to satisfy \Cref{cond:estimation:f} with an explicit rate function.
\begin{theorem}[Instrumental Variables Imply Estimation]
\label{thm:iv-estimation-condition}
Suppose \Cref{asmp:iv-estimation} holds uniformly over $\Theta_{\rm IV}$ and let
\[
    C_{\eta,M}\coloneqq \frac{\sqrt{2M}}{\eta}+\frac{\sqrt{M}}{\eta^2}\,.
\]
Then $\Theta_{\rm IV}$ satisfies \cref{cond:estimation:f} with rate
\[
    R_{\rm IV}(\eps)
    \coloneqq \frac{c_{\rm iv}^2\rho}{4L_{\rm IV}^2 C_{\eta,M}^2}\,\eps^2\,.
\]
\end{theorem}
    \begin{proof} 
    Fix an ordered pair $\theta=(f_\beta,s,\cD,\cQ)$ and $\theta'=(f_{\beta'},s',\cD',\cQ')$ in $\Theta_{\rm IV}$, and write $\pi=(f_\beta,s,\cQ)$ and $\pi'=(f_{\beta'},s',\cQ')$. For notational convenience, define
    \[
        p_\pi(x)
        \coloneqq
        \int k_\pi(x,y)\,\d y\,,
        \qquad
        a_\pi(x)
        \coloneqq
        \int y k_\pi(x,y)\,\d y\,,
        \qquad
        \mu_\pi(x)
        \coloneqq
        \frac{a_\pi(x)}{p_\pi(x)}\,.
    \]
    \paragraph{Step 1: the observed kernel controls the observed conditional mean.}
    For \(\cD\)-a.e. \(x\), the overlap condition gives
    \(p_\pi(x),p_{\pi'}(x)\geq\eta > 0\), and for such \(x\),
    \begin{align*}
        \mu_\pi(x)-\mu_{\pi'}(x)
        &=
        \frac{a_\pi(x)}{p_\pi(x)}
        -
        \frac{a_{\pi'}(x)}{p_{\pi'}(x)} \\
        &=
        \frac{a_\pi(x)}{p_\pi(x)}
        -
        \frac{a_{\pi'}(x)}{p_\pi(x)}
        +
        \frac{a_{\pi'}(x)}{p_\pi(x)}
        -
        \frac{a_{\pi'}(x)}{p_{\pi'}(x)} \\
        &=
        \frac{a_\pi(x)-a_{\pi'}(x)}{p_\pi(x)}
        +
        a_{\pi'}(x)
        \left(
            \frac{1}{p_\pi(x)}
            -
            \frac{1}{p_{\pi'}(x)}
        \right) \\
        &=
        \frac{a_\pi(x)-a_{\pi'}(x)}{p_\pi(x)}
        +
        a_{\pi'}(x)
        \frac{p_{\pi'}(x)-p_\pi(x)}
             {p_\pi(x)p_{\pi'}(x)}\,.
    \end{align*}
    Therefore, using the triangle inequality and the overlap condition
    $p_\pi(x),p_{\pi'}(x)\geq \eta$,
    \begin{align*}
        \abs{\mu_\pi(x)-\mu_{\pi'}(x)}
        &\leq
        \frac{\abs{a_\pi(x)-a_{\pi'}(x)}}{p_\pi(x)}
        +
        \abs{a_{\pi'}(x)}
        \frac{\abs{p_\pi(x)-p_{\pi'}(x)}}{p_\pi(x)p_{\pi'}(x)} \\
        &\leq
        \frac{\abs{a_\pi(x)-a_{\pi'}(x)}}{\eta}
        +
        \frac{\abs{a_{\pi'}(x)}\abs{p_\pi(x)-p_{\pi'}(x)}}{\eta^2}\,.
    \end{align*}
    For the first term, Cauchy--Schwarz and the moment bound give
    \begin{align*}
        \int \abs{a_\pi(x)-a_{\pi'}(x)}\,\d\cD(x)
        &\leq
        \iint \abs{y}\abs{k_\pi(x,y)-k_{\pi'}(x,y)}\,\d y\,\d\cD(x)\\
        &\leq
        \inparen{\iint y^2\abs{k_\pi(x,y)-k_{\pi'}(x,y)}\,\d y\,\d\cD(x)}^{1/2}
        d_\cD(\pi,\pi')^{1/2}\\
        &\leq
        \sqrt{2M}\, d_\cD(\pi,\pi')^{1/2}\,.
    \end{align*}
    For the second term, consider that
    \[
        p_\pi(x)-p_{\pi'}(x)
        =
        \int \left(k_\pi(x,y)-k_{\pi'}(x,y)\right)\,\d y\,.
    \]
    Thus, by the triangle inequality,
    \[
        \abs{p_\pi(x)-p_{\pi'}(x)}
        \leq
        \int
        \abs{k_\pi(x,y)-k_{\pi'}(x,y)}
        \,\d y\,.
    \]
    Integrating over $x\sim\cD$ gives
    \[
    \begin{aligned}
        \int \abs{p_\pi(x)-p_{\pi'}(x)}\,\d\cD(x)
        &\leq
        \iint
        \abs{k_\pi(x,y)-k_{\pi'}(x,y)}
        \,\d y\,\d\cD(x) 
        = d_\cD(\pi,\pi')\,.
    \end{aligned}
    \]
    Next, observe that
    \begin{align*}
        a_{\pi'}(x)^2
        =
        \left(
            \int y \cdot k_{\pi'}(x,y)\,\d y
        \right)^2  &\leq
        \left(
            \int y^2 k_{\pi'}(x,y)\,\d y
        \right)
        \left(
            \int k_{\pi'}(x,y)\,\d y
        \right) \\
        &=
        p_{\pi'}(x)
        \int y^2 k_{\pi'}(x,y)\,\d y \\
        &\leq
        \int y^2 k_{\pi'}(x,y)\,\d y\,.
    \end{align*}
    A second application of Cauchy--Schwarz therefore yields
    \[
        \int \abs{a_{\pi'}(x)}\abs{p_\pi(x)-p_{\pi'}(x)}\,\d\cD(x)
        \leq
        \sqrt{M}\,d_\cD(\pi,\pi')^{1/2}\,.
    \]
    Combining these bounds gives
    \begin{equation}
    \label{eq:iv-upper-mean}
        \E_{X\sim\cD}\!\insquare{\abs{\mu_\pi(X)-\mu_{\pi'}(X)}}
        \leq
        C_{\eta,M}\,d_\cD(\pi,\pi')^{1/2}\,.
    \end{equation}
    \paragraph{Step 2: the index restriction converts regressor separation into mean separation.}
    Let $G(x)\coloneqq \mu_\pi(x)-\mu_{\pi'}(x)$ and condition on $(P_\pi,P_{\pi'})=(u,v)$. If $X_1,X_2$ are independent draws from this conditional law, then the index restriction implies
    \[
        G(X_1)-G(X_2)
        =
        \inangle{\beta-\beta',X_1-X_2}\,,
    \]
    because the two selection-correction terms $\lambda_\pi(u)$ and $\lambda_{\pi'}(v)$ cancel within the cell. Hence, by symmetrization,
    \begin{align*}
        \E\!\insquare{\abs{G(X)}}
        \geq
        \frac12\,\E\!\insquare{\abs{G(X_1)-G(X_2)}}
        =
        \frac12\,\E\!\insquare{\abs{\inangle{\beta-\beta',X_1-X_2}}}
        \geq
        \frac12 c_{\rm iv}\sqrt{\rho}\,\norm{\beta-\beta'}_2\,,
    \end{align*}
    where the last inequality uses the rank and anti-concentration assumptions.  If the regressors are $\eps$-separated, then
    \[
        \eps
        \leq
        \E\!\insquare{\abs{\inangle{\beta-\beta',X}}}
        \leq
        L_{\rm IV}\norm{\beta-\beta'}_2\,,
    \]
    and therefore
    \begin{equation}
    \label{eq:iv-lower-mean}
        \E_{X\sim\cD}\!\insquare{\abs{\mu_\pi(X)-\mu_{\pi'}(X)}}
        \geq
        \frac{c_{\rm iv}\sqrt{\rho}}{2L_{\rm IV}}\,\eps\,.
    \end{equation}
    Combining \eqref{eq:iv-upper-mean} and \eqref{eq:iv-lower-mean} gives
    \[
        d_\cD(\pi,\pi')
        \geq
        \frac{c_{\rm iv}^2\rho}{4L_{\rm IV}^2C_{\eta,M}^2}\,\eps^2\,,
    \]
    Since $\tv{\cD}{\cD'}\geq0$, this implies
    \[
        \tv{\cD}{\cD'}+d_\cD(\pi,\pi')
        \geq R_{\rm IV}(\eps)\,.
    \]
    Thus, $\Theta_{\rm IV}$ satisfies \cref{cond:estimation:f} with the claimed rate.
\end{proof}
The above theorem provides a meaningful set of sufficient conditions, similar to those appearing in the existing literature, for when instrumental variables can be used to satisfy our estimation condition as defined in \Cref{cond:estimation:f}.

\newpage

\newpage
 
    \printbibliography
 
\newpage

\appendix 
\addtocontents{toc}{\protect\setcounter{tocdepth}{2}}


\section{Characterization of Identifiability of $\sstar$ }
    \label{sec:characterization:s} 
    In some applications, the selection rule $\sstar$ is itself a primary object of interest, since it describes the ``filter'' that determines which outcomes are observed. 
    Motivated by such settings, here we study when $\sstar$ is identifiable from censored observations.
    For example, in bioinformatics and astronomy, one often seeks to understand detection or censoring mechanisms in addition to recovering the underlying regression signal.

    We first define identifiability of $\sstar$ analogously to the identifiability of $\fstar$ in \cref{problem:identification_estimation}.

\begin{definition}[Identifiable selector]\label{def:identifiable-selector}
    Fix concept classes $(\hyF,\hyS,\hyDX,\hyQ)$.
    We say that $s^\star\in\hyS$ is identifiable from censored data if there exists a deterministic mapping
    $\psi_S$ from censored distributions realizable with respect to $(\hyF,\hyS,\hyDX,\hyQ)$ to $\hyS$ such that, for every
    $(f,s,\cD,\cQ)\in\hyF\times\hyS\times\hyDX\times\hyQ$,
    \[
      \psi_S\!\inparen{\cC_{f,s,\cD,\cQ}} = s \qquad (\cD\otimes\nu)\text{-a.e.}
    \]
    Equivalently, for any two quadruples $(f,s,\cD,\cQ)$ and $(f',s',\cD',\cQ')$ in $\hyF\times\hyS\times\hyDX\times\hyQ$,
    if $\cC_{f,s,\cD,\cQ}=\cC_{f',s',\cD',\cQ'}$, then $s=s'$, $(\cD\otimes\nu)${-a.e.}
\end{definition}
We next present a condition that characterizes when $\sstar$ is identifiable.

\begin{mdframed}
\vspace{-3mm}
\begin{condition}[$s$-identifiability condition]\label{cond:identifiability:s}
    
    Recall that \(\nu\) denotes the Lebesgue measure on \(\R\), and suppose every \(\cQ\in\hyQ\) admits a density with respect to \(\nu\). For every \(\cD\in\hyDX\), every \(s,s'\in\hyS\) such that 
    \[ (\cD\otimes\nu)\!\left( \{(x,y)\in\cX\times\R:s(x,y)\neq s'(x,y)\} \right)>0\,, \] 
    every \(f,f'\in\hyF\), and every \(\cQ,\cQ'\in\hyQ\), the corresponding observed kernels 
    \[ (x,y)\mapsto s(x,y)\cQ(y-f(x)) \quadand\quad (x,y)\mapsto s'\!(x,y)\cQ'\!(y-f'(x)) \] 
    are not equal \((\cD\otimes\nu)\)-a.e. Equivalently, they differ on a set of positive $(\cD\otimes\nu)$-measure.
    
\end{condition}
\end{mdframed}
The following theorem characterizes the identifiability of $\sstar$.
\begin{theorem}[Characterization of $\sstar$]
\label{thm:characterization:s}
    Fix concept classes $(\hyF,\hyS,\hyDX,\hyQ)$.
    Suppose all distributions in $\hyDX$ and $\hyQ$ have a density.
    Then, every selector $\sstar\in\hyS$ is identifiable from censored data (\cref{def:identifiable-selector})
    if and only if $(\hyF,\hyS,\hyDX,\hyQ)$ satisfies \cref{cond:identifiability:s}.
\end{theorem}

\begin{proof}[Proof of \cref{thm:characterization:s}]
    For any quadruple $(f,s,\cD,\cQ)\in\hyF\times\hyS\times\hyDX\times\hyQ$, let $\cC_{f,s,\cD,\cQ}$ denote the distribution of the sample $(X,D,Y)$ generated by the sample-selection model (\cref{def:selection_bias_model}) specified by $(f,s,\cD,\cQ)$.
    Here, $Y\in \R\cup\inbrace{\emptyset}$ and $D\coloneqq \mathds{1}\!\inbrace{Y\neq \emptyset}$.
    Let $\cQ$ denote the density of the noise distribution $\cQ\in\hyQ$ (which exists by assumption).
    Conditional on $X=x$, the joint density of the observable event $\inbrace{D=1,\,Y\in \d y}$ is
    \[
        s(x,y)\,\cQ\!\inparen{y-f(x)}\,.
        \yesnum\label{eq:kernel-s}
    \]

    \paragraph{($\Rightarrow$) \cref{cond:identifiability:s} implies identifiability.}
    Toward a contradiction, suppose $(\hyF,\hyS,\hyDX,\hyQ)$ satisfies \cref{cond:identifiability:s} but selector identifiability fails.
    Then there exist $(f,s,\cD,\cQ)$ and $(f',s',\cD',\cQ')$ such that
    \[
        \cC_{f,s,\cD,\cQ} = \cC_{f',s',\cD',\cQ'}
        \quadand\quad
        \text{$s$ and $s'$ are not equal under $(\cD \otimes \nu)$-a.e. equivalence}\,.
    \]
    Since there is no censoring on $X$, equality of the censored distributions implies their $X$-marginals are equal, hence $\cD=\cD'$.
    Comparing the conditional laws given $X=x$ on the event $\inbrace{D=1,\,Y\in \d y}$ and using \eqref{eq:kernel-s}, we obtain that for $(\cD \otimes \nu)$-a.e. $(x,y)$,
    \[
        s(x,y)\,\cQ\!\inparen{y-f(x)}
        ~~ = ~~
        s'\!(x,y)\,\cQ'\!\inparen{y-f'(x)}\,.
        \yesnum\label{eq:kernel-equality-s}
    \]
    Since $s$ and $s'$ are not equal under $(\cD \otimes \nu)$-a.e. equivalence, \cref{cond:identifiability:s} (applied to $s,s'$, $f,f'$, $\cQ,\cQ'$, and $\cD$) contradicts \eqref{eq:kernel-equality-s}.
     Hence identifiability holds.

    \paragraph{($\Leftarrow$) Identifiability requires \cref{cond:identifiability:s}.}
    Suppose toward a contradiction that \cref{cond:identifiability:s} fails but identifiability holds.
    Since \cref{cond:identifiability:s} fails, there exist \emph{distinct} $s,s'\in\hyS$ (differing on a set of positive $(\cD\otimes\nu)$-measure), some $f,f'\in\hyF$, some $\cQ,\cQ'\in\hyQ$, and some $\cD\in\hyDX$ such that
    \[
        s(x,y)\,\cQ\!\inparen{y-f(x)}
        =
        s'\!(x,y)\,\cQ'\!\inparen{y-f'(x)}
        \quad (\cD \otimes \nu)\text{-a.e.} ~(x,y)\,.
        \yesnum\label{eq:kernel-equality-s-global}
    \]
    We claim \eqref{eq:kernel-equality-s-global} implies
    $\cC_{f,s,\cD,\cQ}=\cC_{f',s',\cD,\cQ'}$.
    This contradicts identifiability because $s$ and $s'$ are not equal under $(\cD\otimes \nu)$-a.e. equivalence but the censored distributions are equivalent.

    It remains to prove the claim.
    Define
    \[
        \cC \coloneqq \cC_{f,s,\cD,\cQ}
        \quadand\quad
        \cC' \coloneqq \cC_{f',s',\cD,\cQ'}\,.
    \]
    Consider the factorization
    \begin{align*}
        \cC\!\inparen{X\in \d x,\,D=d,\,Y\in \d y}
        &=
        \cC\!\inparen{X\in \d x}\cdot \cC\!\inparen{D=d,\,Y\in \d y \mid X\in \d x}\\
        \cC'\!\inparen{X\in \d x,\,D=d,\,Y\in \d y}
        &=
        \cC'\!\inparen{X\in \d x}\cdot \cC'\!\inparen{D=d,\,Y\in \d y \mid X\in \d x}\,.
    \end{align*}
    We verify the two factors match:
    \begin{enumerate}
        \item Since there is no selection on $X$, $\cC(X\in\d x)=\cD(X\in\d x)=\cC'(X\in\d x)$.
        \item For $D=1$ and $y\in\R$, \eqref{eq:kernel-equality-s-global} implies the conditional densities of $\inbrace{D=1,\,Y\in \d y}$ given $X=x$ are identical under $\cC$ and $\cC'$, hence for $\cD$-a.e. $x$,
        \[
            \cC\!\inparen{D=1,\,Y\in \d y \mid X\in \d x}
            =
            \cC'\!\inparen{D=1,\,Y\in \d y \mid X\in \d x}\,.
        \]
        (Both sides put zero mass on $\inbrace{Y=\emptyset}$ when $D=1$.)
        \item For $D=0$ and $y\in\R$, both distributions assign zero mass:
        \[
            \cC\!\inparen{D=0,\,Y\in \d y \mid X\in \d x}
            = 0
            =
            \cC'\!\inparen{D=0,\,Y\in \d y \mid X\in \d x}\,.
        \]
        \item Finally, since $D=0$ implies $Y=\emptyset$,
        \begin{align*}
            \cC\!\inparen{D=0,\,Y=\emptyset \mid X\in \d x}
            &= 1 - \cC\!\inparen{D=1\mid X=x}
            = 1 - \int_{-\infty}^{\infty} s(x,y)\,\cQ\!\inparen{y-f(x)}\,\d y\,,\\
            \cC'\!\inparen{D=0,\,Y=\emptyset \mid X\in \d x}
            &= 1 - \cC'\!\inparen{D=1\mid X=x}
            = 1 - \int_{-\infty}^{\infty} s'\!(x,y)\,\cQ'\!\inparen{y-f'(x)}\,\d y\,.
        \end{align*}
        Using \eqref{eq:kernel-equality-s-global}, the integrands agree for \(\cD\)-a.e. \(x\), so the integrals are equal for each such $x$, hence
        \[
            \cC\!\inparen{D=0,\,Y=\emptyset \mid X\in \d x}
            =
            \cC'\!\inparen{D=0,\,Y=\emptyset \mid X\in \d x}\,.
        \]
    \end{enumerate}
    Substituting these equalities into the above factorization yields $\cC=\cC'$, as claimed.
    This contradicts identifiability, completing the proof.
\end{proof}

\section{Details Omitted from Proof of \cref{thm:estimation:f} (Cover Size)} 
\label{sec:oracle}

{In this section, we construct a finite cover of the family of \emph{censored} distributions induced by tuples
$\inparen{f,s,\cDX,\cQ}\in \hyF\times \hyS\times \hyX\times \hyQ$. This cover is the main combinatorial input for our
finite-sample estimation result in \cref{sec:estimation}: once we have a finite candidate family of censored laws, we
can select a near-best fitting candidate from samples using standard finite-class distribution selection routines.
Accordingly, our goal here is to show that the censored laws inherit a controlled covering number from the component
classes $\hyF,\hyS,\hyX,$ and $\hyQ$.

To state the result, we recall basic definitions of covers, norms, and distances.}

\begin{definition}[Covers and Covering Numbers]
    Consider the concept class $\hyH \subseteq \inbrace{h: \R^\ell \rightarrow [0,1]}$ with a metric $d(\cdot, \cdot)$. The function class $\hyH_\eps$ is an $\eps$-cover of $\hyH$ if, for every function $h \in \hyH$, there is a function $\bar{h} \in \hyH_\eps$ such that $d(h, \bar{h}) \leq \eps$. The size of the smallest cover $\hyH_\eps$ for $\hyH$ is called the covering number of $\hyH$ and is denoted by $N(\hyH, d, \eps)$. 
\end{definition}
Note that a concept class can contain infinitely many concepts, while being covered by a finite number. 
We also remind the reader of the $L_p$-norm for measurable functions, TV distance for distributions, and the relationship between them.

\begin{definition}[$L_p$-Norm and TV Distance]
    Let $f$ and $g$ be two functions measurable with respect to a measure $\mu$. Then $L_p(\mu)(f,g) = \norm{f - g}_{L_p(\mu)} \coloneqq \E_{x \sim \mu}\insquare{\abs{f(x) - g(x)}^p}^{1/p}$. When $\mu$ is Lebesgue, we may simply write $\norm{f - g}_{L_p(\mu)} = \norm{f-g}_{p}$. The total variation (TV) distance between two densities $f$ and $g$ is $\tv{f}{g} \coloneqq \frac{1}{2}\norm{f - g}_1$. 
\end{definition}
Finally, we remind the reader of Lipschitz continuity.
\begin{definition}[Lipschitz continuity]
    Given two metric spaces $(\hyA, \d_\hyA)$ and $(\hyB, \d_\hyB)$, a function $f: \hyA \mapsto \hyB$ is Lipschitz continuous if there exists a constant $K > 0$ such that for all $a, a' \in \hyA$,
    \[ \d_\hyB(f(a) - f(a')) \leq K \cdot \d_\hyA(a, a')\,. \]
    Any such $K$ is referred to as the Lipschitz constant for the function $f$, and $f$ may be referred to as $K$-Lipschitz. We may specify that $f$ is $K$-Lipschitz relative to the metrics $\d_\hyA$ and $\d_\hyB$. 
\end{definition}

\paragraph{Setup.}
    For each $\theta=(f,s,\cD_X,\cQ)\in \hyF\times\hyS\times\hyX\times\hyQ$, the induced censored law $\cC_\theta$ is a
    distribution on $\cX\times\inparen{\R\cup\inbrace{\emptyset}}$ with density on $\cX\times\R$ given by
    \begin{align*}
        c_\theta(x,\tilde y)
        \coloneqq
        \begin{cases}
        \cDX(x)\,\cQ\!\inparen{y-f(x)}\,s(x,y) & \text{if }\tilde y = y\in\R\,,\\[4pt]
        \cDX(x)-\int_{\R}\cDX(x)\,\cQ\!\inparen{y-f(x)}\,s(x,y)\,\d y & \text{if }\tilde y=\emptyset\,.
        \end{cases}
    \end{align*}
    Equivalently, the non-null density is $c_\theta(x,y)\coloneqq \cDX(x)\,\cQ\!\inparen{y-f(x)}\,s(x,y)$, and the null atom's (joint) density over $x$ is
    $c_\theta(x,\emptyset)\coloneqq \cDX(x)-\int c_\theta(x,y)\,\d y$. Our goal is to cover the family
    $\cC\inparen{\hyF,\hyS,\hyX,\hyQ}\coloneqq \inbrace{\cC_\theta:\theta\in\hyF\times\hyS\times\hyX\times\hyQ}$
    in total variation distance.
    We write $\cC_\theta$ for the induced censored law, and we write $\theta^\star=(\fstar,\sstar,\cDX,\cDxi)$ and $\cC^\star=\cC_{\theta^\star}$ for the realizable ground truth.

\paragraph{Cover Size.}
In this section, we prove the following result.

\begin{theorem}[Cover-Size]
    \label{thm:cover-size}
    Fix any $\eps \in (0, 1)$, $\sigma \in (0, 1]$, $\eta \in (0, 1)$, $L_Q>0$, and a distribution $\mu$ over $\R^d \times \R$ with marginals $\mu_X$ and $\mu_Y$. Let concept classes $\hyF$, $\hyS$, $\hyX$ and $\hyQ$ satisfy the following conditions:
    \begin{enumerate}
        \item $\hyF$ has a finite covering number in $L_1(\mu_X)$-norm $N_\hyF = \cover{\hyF}{\frac{\sigma \eta}{24 L_Q}\eps}{\norm{\cdot}_{L_1(\mu_X)}}.$ 
        \item $\hyS$ has a finite covering number in $L_1(\mu)$-norm $N_\hyS = \cover{\hyS}{\frac{\sigma^2\eta\eps}{12}}{\norm{\cdot}_{L_1(\mu)}}$.
        \item $\hyX$ has a finite covering number in TV distance $N_\hyX = \cover{\hyX}{\frac{\eta}{24}\eps}{\mathsf{TV}}$,
        and each $\cD \in \hyX$ is $\sigma$-smooth with respect to $\mu_X$.
        \item $\hyQ$ has a finite covering number in TV distance $N_\hyQ = \cover{\hyQ}{\frac{\eta\eps}{48}}{\mathsf{TV}}$,
        and each $\cQ \in \hyQ$ is $\sigma$-smooth with respect to $\mu_Y$ and is $L_Q$-Lipschitz under translations (relative to the $L_1$ norm).
        
        \item For every $(f,\cD,\cQ)\in\hyF\times\hyX\times\hyQ$, the distribution $\cP_{f,\cD,\cQ}$ defined in \cref{def:regularity} is $\sigma^2$-smooth with respect to $\mu$.
        
    \end{enumerate}
    Then the family $\cC\inparen{\hyF,\hyS,\hyX,\hyQ}$ admits an $\eps$-cover in total variation distance of
    size $M(\eps)$ satisfying
    \[ M\inparen{\eps} \leq  N_\hyF\cdot N_\hyS\cdot N_\hyX\cdot N_\hyQ\,. \]
\end{theorem}
\begin{proof}
With a slight abuse of notation, define the class of functions
\[
\hyQ \circ \hyF \coloneqq \inbrace{ (x,y)\mapsto \cQ\!\inparen{y - f(x)} : \cQ \in \hyQ,\, f \in \hyF } \,.
\]

\paragraph{Step 1: Cover for $\hyQ \circ \hyF$.}
We establish a cover for $\hyQ \circ \hyF$ from covers of $\hyQ$ and $\hyF$.
Let $\inbrace{\cQ_j}$ be an $\frac{\eta\eps}{48}$-cover of $\hyQ$ in TV distance and let $\inbrace{f_i}$ be a
$\frac{\sigma \eta}{24 L_Q}\eps$-cover of $\hyF$ in $\norm{\cdot}_{L_1(\mu_X)}$. Then there exist indices $j$ and $i$ such that
$\norm{\cQ_j-\cDxi}_1\leq \frac{\eta\eps}{24}$ and $\norm{f_i-\fstar}_{L_1(\mu_X)}\leq \frac{\sigma \eta}{24L_Q}\eps$.
Fix any $x\in\R^d$. By translation invariance of $\norm{\cdot}_1$ and the shift-Lipschitz property of $\cQ_j$,
\begin{align*}
\int \abs{\cDxi\!\inparen{y-\fstar(x)} - \cQ_j\!\inparen{y-f_i(x)}}\,\d y
&\leq
\int \abs{\cDxi\!\inparen{y-\fstar(x)} - \cQ_j\!\inparen{y-\fstar(x)}}\,\d y
\\
&\quad+
\int \abs{\cQ_j\!\inparen{y-\fstar(x)} - \cQ_j\!\inparen{y-f_i(x)}}\,\d y
\\
&\leq
\norm{\cDxi-\cQ_j}_1
+
L_Q\,\abs{\fstar(x)-f_i(x)}
\\
&\leq
\frac{\eta\eps}{24}
+
L_Q\,\abs{\fstar(x)-f_i(x)}\,.
\end{align*}
In particular, integrating over $x$ and using $\sigma$-smoothness of $\cDX$ with respect to $\mu_X$ yields
\[
\iint \abs{\cDxi\!\inparen{y-\fstar(x)} - \cQ_j\!\inparen{y-f_i(x)}}\,\d y\,\d\cDX(x)
\leq
\frac{\eta\eps}{24}
+
\frac{L_Q}{\sigma}\,\norm{\fstar-f_i}_{L_1(\mu_X)}
\leq
\frac{\eta\eps}{12}\,.
\]
Thus, the Cartesian product $\inbrace{\cQ_j}\times \inbrace{f_i}$ induces a finite $\frac{\eta\eps}{12}$-cover of $\hyQ\circ \hyF$ in the above metric. %


\paragraph{Step 2: Constructing a cover in $L_1$-norm for $\hyF \times \hyS \times \hyX \times \hyQ$.}
Fix a realizable ground truth $\inparen{\fstar,\sstar,\cDX,\cDxi}\in \hyF\times\hyS\times\hyX\times\hyQ$.
Choose $\cD\in\hyX$ from a TV-cover of $\hyX$ so that $\tv{\cD}{\cDX}\leq \frac{\eta\eps}{24}$, choose
$s\in\hyS$ from an $L_1(\mu)$-cover of $\hyS$ so that $\norm{s-\sstar}_{L_1(\mu)}\leq \frac{\sigma^2\eta\eps}{12}$,
and choose $\cQ\in\hyQ$ and $f\in\hyF$ so that the corresponding element of $\hyQ\circ\hyF$ satisfies
\[
\iint \abs{\cQ\!\inparen{y-f(x)}-\cDxi\!\inparen{y-\fstar(x)}}\,\cDX(x)\,\d y\,\d x \leq \frac{\eta\eps}{12}\,,
\]
which exists from the previous step.

We now show that these choices yield
\[
\iint 
\abs{
\cD(x)\,\cQ\!\inparen{y-f(x)}\,s(x,y)
-
\cDX(x)\,\cDxi\!\inparen{y-\fstar(x)}\,\sstar(x,y)
}\,\d y\,\d x
\leq
\frac{\eta\eps}{4}\,.
\]
Define
\begin{align*}
g(x,y) &\coloneqq \cD(x)\,\cQ\!\inparen{y-f(x)}\,s(x,y)\,, \\
g^\star(x,y) &\coloneqq \cDX(x)\,\cDxi\!\inparen{y-\fstar(x)} \,\sstar(x,y)\,, \\
g_1(x,y) &\coloneqq \cDX(x)\,\cDxi\!\inparen{y-\fstar(x)}\,s(x,y)\,, \\
g_2(x,y) &\coloneqq \cDX(x)\,\cQ\!\inparen{y-f(x)}\,s(x,y)\,.
\end{align*}
Then
\[
\norm{g-g^\star}_1
\leq
\underbrace{\norm{g_1-g^\star}_1}_{(A)}
+
\underbrace{\norm{g-g_2}_1}_{(B)}
+
\underbrace{\norm{g_2-g_1}_1}_{(C)}\,,
\]
and we bound each of $(A),(B),(C)$ by $\frac{\eta\eps}{12}$, which yields $\norm{g-g^\star}_1\leq \frac{\eta\eps}{4}$.
\begin{enumerate}
    \item \textit{Bounding (A).}
    Note that
    \[
    (A)
    =
    \iint \abs{s(x,y)-\sstar(x,y)}\cdot \cDX(x)\,\cDxi\!\inparen{y-\fstar(x)}\,\d y\,\d x\,.
    \]
    By $\sigma^2$-smoothness of $\cP_{\fstar,\cDX,\cDxi}$ with respect to $\mu$, we have
    $\cDX(x)\,\cDxi\!\inparen{y-\fstar(x)} \leq \frac{1}{\sigma^2}\mu(x,y)$ pointwise , hence
    \[
    (A)
    \leq
    \frac{1}{\sigma^2}\iint \abs{s(x,y)-\sstar(x,y)}\,\d\mu(x,y)
    =
    \frac{1}{\sigma^2}\norm{s-\sstar}_{L_1(\mu)}
    \leq
    \frac{\eta\eps}{12}\,.
    \]
    \item \textit{Bounding (B).}
    Using $0\leq s\leq 1$ and $\int \cQ\!\inparen{y-f(x)}\,\d y=1$,
    \begin{align*}
    (B)
    &=
    \iint \abs{\cD(x)-\cDX(x)}\,\cQ\!\inparen{y-f(x)}\,s(x,y)\,\d y\,\d x \\
    &\leq
    \int \abs{\cD(x)-\cDX(x)}\inparen{\int \cQ\!\inparen{y-f(x)}\,\d y}\,\d x
    =
    \int \abs{\cD(x)-\cDX(x)}\,\d x \\
    &=
    \norm{\cD-\cDX}_1
    \leq
    2\,\tv{\cD}{\cDX}
    \leq
    \frac{\eta\eps}{12}\,.
    \end{align*}
    \item \textit{Bounding (C).}
    Using $0\leq s\leq 1$,
    \begin{align*}
    (C)
    &=
    \iint \cDX(x)\,s(x,y)\,\abs{\cQ\!\inparen{y-f(x)} - \cDxi\!\inparen{y-\fstar(x)}}\,\d y\,\d x \\
    &\leq
    \iint \cDX(x)\,\abs{\cQ\!\inparen{y-f(x)} - \cDxi\!\inparen{y-\fstar(x)}}\,\d y\,\d x
    \leq
    \frac{\eta\eps}{12}\,,
    \end{align*}
    where the last inequality uses the  $\frac{\eta\eps}{12}$-cover of $\hyQ\circ\hyF$ from the previous step.
\end{enumerate}

\paragraph{Step 3: Concluding an $\eps$-cover in TV distance.}
Fix a realizable ground truth $\theta^\star=\inparen{\fstar,\sstar,\cDX,\cDxi}$ and let $\bar\theta=\inparen{f,s,\cD,\cQ}$
be the candidate produced by the previous steps. Step~2 shows that the non-null densities satisfy
\[
\norm{c_{\bar\theta}-c_{\theta^\star}}_1
=
\iint
\abs{
\cD(x)\,\cQ\!\inparen{y-f(x)}\,s(x,y)
-
\cDX(x)\,\cDxi\!\inparen{y-\fstar(x)}\,\sstar(x,y)
}\,\d y\,\d x
\leq
\frac{\eta\eps}{4}\,.
\]
To obtain a TV bound for the \emph{censored} laws on $\cX\times\inparen{\R\cup\inbrace{\emptyset}}$, we also control the
null outcome. Recall that
$c_\theta(x,\emptyset)=\cDX(x)-\int c_\theta(x,y)\,\d y$.
Thus, writing $r(x)\coloneqq c_{\bar\theta}(x,\emptyset)$ and $r^\star(x)\coloneqq c_{\theta^\star}(x,\emptyset)$,
\begin{align*}
\int \abs{r(x)-r^\star(x)}\,\d x
&=
\int \abs{
\cD(x)-\cDX(x) - \int\inparen{c_{\bar\theta}(x,y)-c_{\theta^\star}(x,y)}\,\d y
}\,\d x
\\
&\leq
\norm{\cD-\cDX}_1 + \norm{c_{\bar\theta}-c_{\theta^\star}}_1\,.
\end{align*}
%
Now,
\begin{align*}
2\,\tv{\cC_{\bar\theta}}{\cC_{\theta^\star}}
&=
\norm{c_{\bar\theta}-c_{\theta^\star}}_1
+
\int \abs{r(x)-r^\star(x)}\,\d x
\\
&\leq
2\,\norm{c_{\bar\theta}-c_{\theta^\star}}_1
+
\norm{\cD-\cDX}_1
\\
&\leq
2\cdot \frac{\eta\eps}{4}
+
2\,\tv{\cD}{\cDX}
\leq
\frac{\eta}{2}\eps+\frac{\eta}{12}\eps
=
\frac{7\eta}{12}\eps\,,
\end{align*}
where we used $\norm{\cD-\cDX}_1\leq 2\tv{\cD}{\cDX}\leq \frac{\eta\eps}{12}$ by the TV-cover choice of $\cD$.
Therefore,
\[
\tv{\cC_{\bar\theta}}{\cC_{\theta^\star}}
\leq
\frac{7\eta}{24}\eps
\leq
\eps\,,
\]
since $\eta\in(0,1)$ implies $\frac{7\eta}{24}<1$.

\paragraph{Step 4: Counting candidates.}
The candidate family is the Cartesian product of the chosen component covers for $\hyF$, $\hyS$, $\hyX$, and $\hyQ$,
so its size is at most $N_\hyF N_\hyS N_\hyX N_\hyQ$. Since every realizable $\cC_{\theta^\star}$ lies within TV
distance $\eps$ of some element of this finite family, it is an $\eps$-cover. This proves the theorem.

\end{proof}

\begin{corollary}[Internal cover for a joint concept class]\label{cor:internal-cover}
Under the hypotheses of \cref{thm:cover-size}, fix any joint concept class
$\Theta\subseteq\hyF\times\hyS\times\hyX\times\hyQ$.
There is a finite set $\mathcal N\subseteq\Theta$ with $\abs{\mathcal N}\leq M(\eps)$ such that, for every $\theta\in\Theta$, some $\bar\theta\in\mathcal N$ satisfies
\[
    \tv{\cC_\theta}{\cC_{\bar\theta}}\leq 2\eps\,.
\]
\end{corollary}
\begin{proof}
Let $\mathcal A$ be the $\eps$-cover supplied by \cref{thm:cover-size}. Assign each law in $\inbrace{\cC_\theta:\theta\in\Theta}$ to one cover ball containing it. From each occupied ball, choose one tuple in $\Theta$ whose induced law lies in the ball, and let $\mathcal N$ be the set of chosen tuples. Then $\abs{\mathcal N}\leq\abs{\mathcal A}\leq M(\eps)$. Moreover, every $\cC_\theta$ and the chosen law associated with its cover ball are each within $\eps$ of the same center. The triangle inequality gives the stated $2\eps$ bound.
\end{proof}

\begin{remark}
    While \cref{thm:cover-size} directly assumed a cover for real-valued function classes (like $\hyF$ and $\hyS$), one can assume instead a finite fat-shattering dimension for these classes which is more natural for real-valued functions. (The fat-shattering dimension can be thought of as the regression analogue of VC dimension for classification.) Due to \cite{RudelsonVershynin2006}, one can bound the $L_1$ covering number of the class as a function of the fat-shattering dimension, provided a uniform fourth-moment bound on all elements of the function class.
\end{remark}

    \section{Omitted Details from \cref{sec:app:non-iden-s}}   
        In this section, we present the details omitted from \cref{sec:app:non-iden-s}:
        %
        %
        First, we give a proof of \cref{thm:f-identifiable-without-s} followed by a proof of the finite-sample estimation result.
        Then we present the finite-sample version of \cref{thm:f-identifiable-without-s}.

        \begin{remark}[Formal construction in proof of \cref{thm:f-identifiable-without-s}]
            \label{rem:app:construction}
            Fix any $\eta\in (0,1)$, set $c\coloneqq \sfrac{(1+\eta)}{2}$, and fix an arbitrary reference regressor $f_0\in\hyF$.
            Let $\cQ$ be the uniform distribution on $[0,1]$, and let $\cQ'$ be the piecewise-constant distribution on $[0,1]$ that places density $\nfrac{\eta}{c}$ on $[0,\nfrac{1}{2}]$ and density $\nfrac{1}{c}$ on $(\nfrac{1}{2},1]$.
            Define the selection-function family as $\hyS_{\min} \coloneqq \inbrace{s_{\mathrm{const}}, s_{\mathrm{step}}}$
            for $s_{\mathrm{const}}(x,y)\coloneqq c$ and $s_{\mathrm{step}}(x,y)\coloneqq \eta + (1-\eta)\,\mathds{1}\!\inbrace{y\in(f_0(x)+\nfrac{1}{2},\, f_0(x)+1]}.$
            Set $\hyQ_{\min}\coloneqq \inbrace{\cQ,\cQ'}$.
            In the proof of \cref{thm:f-estimable-without-s} we show that the non-identifiability of $\sstar$ holds for any $\hyQ\supseteq \hyQ_{\min}$ and $\hyS\supseteq \hyS_{\min}$, and the identifiability (and estimation) of $\fstar$ holds whenever each $\cQ\in \hyQ$ has $(0,1)\subseteq \supp(\cQ)\subseteq \R_{\geq 0}$ and $\hyF$ admits a cover.
        \end{remark}

        \subsection{Proof of \cref{thm:f-identifiable-without-s}}
        \label{sec:proofof:thm:f-identifiable-without-s}

        We restate \cref{thm:f-identifiable-without-s} below before proceeding with its proof. Refer to \cref{cond:identifiability:s} for identifiability of $\sstar$. 

        \fIdentifiableWithoutS*

        \begin{proof}[Proof of \cref{thm:f-identifiable-without-s}]
        We exhibit an explicit family of concept classes for which $\fstar$ is identifiable but $\sstar$ is not.
        Fix any $\eta\in(0,1)$ and any nonempty regressor class $\hyF\subseteq \hyF_{\rm all}$, where
        $\hyF_{\rm all}$ denotes the set of all measurable functions $f:\cX\to\R$.
        Define the selector and noise classes
        \[
            \hyS_\eta \coloneqq \inbrace{s:\cX\times\R\to[\eta,1]}
            \quadand
            \hyQ^{\geq 0}
            \coloneqq
            \inbrace{
                \cQ:\ \supp(\cQ)\subseteq \R_{\geq 0}\ \text{and}\ \cQ(t)>0\ \forall t\in(0,1)
            }\,,
        \]
        Finally, let $\hyDX$ be any nonempty class such that every $\cDX\in\hyDX$ has a density and $\supp(\cDX)=\cX$.
        Consider the Cartesian-product class $\Theta_0\coloneqq\hyF\times\hyS_\eta\times\hyDX\times\hyQ^{\geq 0}$.
        We show that this class satisfies the desired requirements.
        \begin{lemma}
            $\Theta_0$ satisfies \cref{cond:identifiability:f}.
        \end{lemma}
        \vspace{-7mm}
        \begin{proof}
            Take any $\cDX\in\hyDX$, \underline{distinct} $f,f'\in\hyF$ (which differ on a set of positive $\cDX$-measure), any $s,s'\in\hyS_\eta$, and any $\cQ,\cQ'\in\hyQ^{\geq 0}$.
            Let $B_+\coloneqq\{x:f(x)<f'(x)\}$ and $B_-\coloneqq\{x:f(x)>f'(x)\}$. At least one of these sets has positive $\cDX$-measure. Interchanging $(f,s,\cQ)$ and $(f',s',\cQ')$ if necessary, assume that $\cDX(B_+)>0$. For each $x\in B_+$, set $\Delta(x)\coloneqq f'(x)-f(x)>0$.
            
            Consider $y$ taking values in the following interval:
            \[ I_x \coloneqq \inparen{f(x), f(x) + \frac{\min\!\inbrace{\Delta(x),1}}{2}}\,. \]
            
            Then for all $y \in I_x$, $y-f(x)\in(0,1)$, so $\cQ\!\inparen{y-f(x)}>0$ by definition of $\hyQ^{\geq 0}$, and also $s(x,y)\geq \eta$ by definition of $\hyS_\eta$.
            Hence,
            \[
                s(x,y)\,\cQ\!\inparen{y-f(x)} > 0\,.
            \]
            On the other hand, $y-f'(x) \leq \frac{\min\inbrace{\Delta(x),1}}{2}-\Delta(x)<0$, for $y \in I_x$, so $\cQ'\!\inparen{y-f'(x)}=0$ since $\supp(\cQ')\subseteq\R_{\geq 0}$.
            Thus
            \[
                s(x,y)\,\cQ\!\inparen{y-f(x)}
                \neq
                s'(x,y)\,\cQ'\!\inparen{y-f'(x)}\,.
            \]
            The set $A\coloneqq\{(x,y):x\in B_+,\ y\in I_x\}$ is measurable, and Tonelli's theorem gives
            \[
                (\cDX\otimes\nu)(A)=\int_{B_+}\nu(I_x)\,\d\cDX(x)>0\,.
            \]
            The observed kernels differ throughout $A$. This is exactly \cref{cond:identifiability:f}.
        \end{proof} 
        Thus, by \cref{thm:characterization:f}, it follows that every $\fstar\in\hyF$ is identifiable from censored data.
        \begin{lemma}
            \label{lem:scenario:non-identifiability-of-s}
            $(\hyF,\hyS_\eta,\hyDX,\hyQ^{\geq 0})$ violates \cref{cond:identifiability:s}.
        \end{lemma} 
        \vspace{-7mm}
        \begin{proof}
            Fix any $f\in\hyF$ and any $\cDX\in\hyDX$.
            Let $c \coloneqq \frac{1+\eta}{2}\in(\eta,1),$
            and define two noise densities on $[0,1]$ by
            \[
              \cQ(t)\coloneqq \mathds{1}\!\inbrace{t\in[0,1]}
              \qquadand 
              \cQ'(t)\coloneqq
              \frac{\eta}{c}\,\mathds{1}\!\inbrace{t\in[0,\nfrac{1}{2}]}
              +
              \frac{1}{c}\,\mathds{1}\!\inbrace{t\in(\nfrac{1}{2},1]}\,.
            \]
            Both satisfy $\supp(\cQ),\supp(\cQ')\subseteq[0,1]$ and are strictly positive on $(0,1)$, hence $\cQ,\cQ'\in\hyQ^{\geq 0}$.
            
            Next define two selectors $s,s'\in\hyS_\eta$ by
            \[
              s(x,y)\coloneqq
              \begin{cases}
                \eta\,, & y \leq f(x)+\nfrac12\,,\\
                1\,, & y\in (f(x)+\nfrac12,\ f(x)+1]\,,\\
                \eta\,, & y>f(x)+1\,,
              \end{cases}
              \qquadand 
              s'(x,y)\coloneqq c\quad \text{for all }(x,y)\,.
            \]
            Clearly $s,s'\in\hyS_\eta$, and $s$ and $s$' differ on a set of positive measure. (For instance, $s(x,y) = 1 \neq c = s'(x,y)$ for all $y \in (f(x) + \nfrac12, f(x) + 1]$, for every $x$. This interval has positive Lebesgue measure in $y$ for every $x$.)
            
            We claim that for every $(x,y)\in\supp(\cDX)\times\R$ (thus for $(\cDX \otimes \nu)$-a.e. $(x,y)$),
            \[
                s(x,y)\,\cQ\!\inparen{y-f(x)}
                =
                s'(x,y)\,\cQ'\!\inparen{y-f(x)}\,.
            \]
            Indeed, let $t\coloneqq y-f(x)$ and consider the following cases:
            \begin{itemize}
            \item If $t\notin[0,1]$, then $\cQ(t)=\cQ'(t)=0$, so both sides equal $0$.
            \item If $t\in[0,\nfrac12]$, then $s(x,y)=\eta$ and $\cQ'(t)=\eta/c$, so $s(x,y)\,\cQ(t)=s'(x,y)\,\cQ'(t)=\eta.$ 
            \item If $t\in(\nfrac12,1]$, then $s(x,y)=1$ and $\cQ'(t)=\nfrac1c$, so $s(x,y)\,\cQ(t)=s'(x,y)\,\cQ'(t)=1.$ 
            \end{itemize}
            Therefore, taking $f'=f$ (and any $\cDX \in \hyDX$), we have found \underline{distinct} $s,s'\in\hyS_\eta$ (differing on a set of positive measure) and $\cQ,\cQ'\in\hyQ^{\geq 0}$  such that the inequality in \cref{cond:identifiability:s} never occurs. Hence \cref{cond:identifiability:s} fails.
        \end{proof}
        Thus, by \cref{thm:characterization:s}, it follows that not every selector in $\hyS_\eta$ is identifiable from censored data; equivalently, there exists a non-identifiable $\sstar\in\hyS_\eta$.
        Combining this with the previous lemma implies the result.
    \end{proof}

    \subsection{Finite-Sample Version of \cref{thm:f-identifiable-without-s}}\label{sec:proofof:thm:f-estimable-without-s}
 
The finite-sample version of \cref{thm:f-identifiable-without-s} is as follows:
        \begin{theorem}[Estimation of $\fstar$ without Identifiability of $\sstar$]
            \label{thm:f-estimable-without-s}
            Let $\cX\subseteq \R^d$ be bounded.
            There exist component classes $(\hyF,\hyS,\hyDX,\hyQ)$ over $\cX$ such that the Cartesian-product class $\Theta_0\coloneqq\hyF\times\hyS\times\hyDX\times\hyQ$ satisfies \cref{cond:estimation:f} with rate function $R(\eps)=O(\eps)$ and the selection function $\sstar\in\hyS$ is not identifiable from censored distributions realizable with respect to $\Theta_0$.
            Moreover, the construction can be chosen so that the following hold:
            \begin{itemize}
                \item $\hyF$ admits a $\zeta$-cover of size $\inparen{1+\nfrac{1}{\zeta}}^{O(d)}$ (for any $\zeta>0$);
                \item $\hyS$ and $\hyQ$ have two elements each, so their covering numbers are two at every scale;
                \item $\hyDX$ is a singleton, so its covering number is 1 at every scale.
            \end{itemize}
            \end{theorem}

\begin{proof}[Proof of \cref{thm:f-estimable-without-s}]
We give a Cartesian-product class satisfying the estimation condition \cref{cond:estimation:f} with a linear rate, while $\sstar$ is not identifiable.
Fix any $\eta\in(0,1)$ and let $c \coloneqq \sfrac{(1+\eta)}{2}\in(\eta,1).$
Let the covariate space be $\cX\subseteq \R^d$ and assume $\sup_{x\in\cX}\norm{x}_2\leq R<\infty$.
Fix $B>0$ such that $BR\leq 1$, and define the regressor class
\[
    \hyF
    \coloneqq
    \inbrace{
        f_w:\cX\to\R \ \colon\ f_w(x)=\inangle{w,x},\ \norm{w}_2\leq B
    }\,.
\]
Let $\hyDX$ be a singleton $\hyDX\coloneqq\sinbrace{\cDX}$, where $\cDX$ has a density and $\supp(\cDX)=\cX$.
Fix an arbitrary reference regressor $f_0\in\hyF$ (\eg{}, $f_0=f_{w_0}$ for $w_0=0$).
Define two noise densities on $[0,1]$:
\[
  \cQ(t)\coloneqq \mathds{1}\!\inbrace{t\in[0,1]}
  \qquadand
  \cQ'(t)\coloneqq
  \frac{\eta}{c}\,\mathds{1}\!\inbrace{t\in[0,\nfrac12]}
  +
  \frac{1}{c}\,\mathds{1}\!\inbrace{t\in(\nfrac12,1]}\,,
\]
and set $\hyQ\coloneqq \inbrace{\cQ,\cQ'}$.
Finally define a \emph{two-element} selection class $\hyS$ by
\[
    \hyS \coloneqq \inbrace{s_0,s_{\mathrm{step}}}\,,
\]
where, for each $(x,y)\in\cX\times\R$,
\[
    s_0(x,y)\coloneqq c
    \qquadand  
    s_{\mathrm{step}}(x,y)\coloneqq
    \begin{cases}
        \eta\,, & y \leq f_0(x)+\nfrac12\,,\\
        1\,, & y\in (f_0(x)+\nfrac12,\ f_0(x)+1]\,,\\
        \eta\,, & y>f_0(x)+1\,.
    \end{cases}
\]
Note that $\hyS\subseteq \inbrace{s:\cX\times\R\to[\eta,1]}$ and $\hyQ\subseteq \inbrace{\text{densities supported on }[0,1]}$; hence, for any realizable tuple we have $\Ex[\sstar(X,Y)]\geq \eta$.
Let $\Theta_0\coloneqq\hyF\times\hyS\times\hyDX\times\hyQ$.

\begin{lemma}[Existence of finite cover]\label{lem:existence-cover}
For every $\zeta>0$, there exists a finite set $\cF_\zeta\subseteq \hyF$ such that for every $f\in\hyF$ there exists $\tilde f\in\cF_\zeta$ with $\snorm{f-\tilde f}_\infty \leq \zeta.$
Moreover, one can ensure that
\[
    \abs{\cF_\zeta}  \leq  \inparen{1+\frac{2BR}{\zeta}}^d\,.
\]
\end{lemma}
\vspace{-7mm}
\begin{proof}[Proof of \cref{lem:existence-cover}]
For any $w,w'\in\R^d$ and $x\in\cX$, it holds that 
\[
    \abs{f_w(x)-f_{w'}(x)}=\abs{\inangle{w-w',x}}
    \leq \norm{w-w'}_2\,\norm{x}_2
    \leq R\norm{w-w'}_2\,,
\]
so $\norm{f_w-f_{w'}}_\infty\leq R\norm{w-w'}_2$.
Thus, any $(\zeta/R)$-net of the Euclidean ball $\inbrace{w:\norm{w}_2\leq B}$ induces a $\zeta$-net of $\hyF$ in $\norm{\cdot}_\infty$.
A standard bound gives an $(\zeta/R)$-net of size at most $\inparen{1+2B/(\zeta/R)}^d=\inparen{1+2BR/\zeta}^d$ (see, \eg{}, \cite[Corollary 4.2.13]{vershynin2018high}).
\end{proof}
\begin{lemma}[A rate function for \cref{cond:estimation:f}]\label{lem:rate-function}
The Cartesian-product class $\Theta_0$ satisfies \cref{cond:estimation:f} with rate function $R(\cdot)$ defined by $R(0)\coloneqq 0$ and for each $\eps>0$,
\[
    R(\eps)\ \coloneqq\
    \min\!\inbrace{
        \frac{2\eta^2}{c\,\max\inbrace{2BR,1}}\cdot \eps,\ \frac12
    }\,.
\] 
\end{lemma}
\vspace{-7mm}
\begin{proof}[Proof of \cref{lem:rate-function}]
Fix $\theta=(f,s,\cD,\cQ_1)$ and $\theta'=(f',s',\cD,\cQ_2)$ in $\Theta_0$ (where necessarily $\cD=\cDX$), and write $\pi=(f,s,\cQ_1)$ and $\pi'=(f',s',\cQ_2)$.
Recall
\[
    d_{\cD}\!\inparen{\pi,\pi'}
    =
    \iint \abs{
        s(x,y)\,\cQ_1\!\inparen{y-f(x)}
        -
        s'(x,y)\,\cQ_2\!\inparen{y-f'(x)}
    }\,\d\cD(x)\,\d y\,.
\]
For each $x\in\cX$, define the intervals
\[
    I_x \coloneqq \insquare{f(x),f(x)+1}
    \qquadand
    I_x' \coloneqq \insquare{f'(x),f'(x)+1}\,.
\]
Since both $\cQ$ and $\cQ'$ are supported on $[0,1]$, we have
\[
    s(x,y)\,\cQ_1\!\inparen{y-f(x)}=0\ \text{for }y\notin I_x
    \qquadand
    s'(x,y)\,\cQ_2\!\inparen{y-f'(x)}=0\ \text{for }y\notin I_x'\,.
\]
Moreover, for every $t\in(0,1)$ and every admissible $(s,\cQ_1)\in\hyS\times\hyQ$,
\[
    s\!\inparen{x,f(x)+t}\,\cQ_1(t)\ \geq\ \eta\cdot \min\inbrace{1,\frac{\eta}{c}}\ =\ \frac{\eta^2}{c}\,,
\]
since $s(\cdot,\cdot)\geq \eta$ and $\inf_{t\in(0,1)} \cQ_1(t)\geq \eta/c$. The same lower bound holds for $s'(x,f'(x)+t)\cQ_2(t)$.
Therefore, for each fixed $x$,
\begin{align*}
&\int \abs{
    s(x,y)\,\cQ_1\!\inparen{y-f(x)}
    -
    s'(x,y)\,\cQ_2\!\inparen{y-f'(x)}
}\,\d y\\
&\quad\qquad\geq\quad
\int_{I_x\setminus I_x'}
s(x,y)\,\cQ_1\!\inparen{y-f(x)}\,\d y
+
\int_{I_x'\setminus I_x}
s'(x,y)\,\cQ_2\!\inparen{y-f'(x)}\,\d y\\
&\quad\qquad\geq\quad    
\frac{\eta^2}{c}\,\lambda\inparen{I_x\triangle I_x'}\,.
\end{align*}
Since $I_x$ and $I_x'$ are unit-length intervals, $\lambda\inparen{I_x\triangle I_x'}=2\min\!\inbrace{\abs{f(x)-f'(x)},1}.$
Integrating over $x\sim\cD$ yields
\[
    d_{\cD}\!\inparen{\pi,\pi'}
    \ge
    \frac{2\eta^2}{c}\,\Ex_{X\sim\cD}\min\!\inbrace{\abs{f(X)-f'(X)},1}\,.
\]
Next, since $f,f'\in\hyF$ and $\sup_{x\in\cX}\norm{x}_2\leq R$, we have $\norm{f-f'}_\infty \leq 2BR$.
Let $K\coloneqq \max\inbrace{2BR,1}$. Then for all $z\in\insquare{0,2BR}$, $\min\!\inbrace{z,1}\  \geq  \frac{z}{K}.$
Applying this pointwise with $z=\abs{f(X)-f'(X)}$ gives
\[
    d_{\cD}\!\inparen{\pi,\pi'}
    \ge
    \frac{2\eta^2}{cK}\,\Ex_{X\sim\cD}\abs{f(X)-f'(X)}\,.
\]
Therefore, if $\Ex_{X\sim\cD}\abs{f(X)-f'(X)}\geq \eps$, then $d_{\cD}\!\inparen{\pi,\pi'}\geq \frac{2\eta^2}{cK}\,\eps.$
Clipping by $\nfrac{1}{2}$ and defining $R(0)\coloneqq 0$ yields the stated rate function.
\end{proof}

\begin{lemma}[Regularity of $(\hyF,\hyDX,\hyQ)$]\label{lem:smoothness}
Let $H\coloneqq 2BR+1$. Let $\mu_X\coloneqq \cDX$ and let $\mu_Y$ be the uniform distribution on $[-BR,BR+1]$.
Then $(\hyF,\hyDX,\hyQ)$ satisfy the regularity conditions in \cref{def:regularity} with constants
\[
    \sigma\coloneqq c/H
    \qquadand
    L\coloneqq \frac{6}{c}\,,
\]
with respect to $\mu\coloneqq \mu_X\times \mu_Y$.
\end{lemma}
\vspace{-7mm}
\begin{proof}[Proof of \cref{lem:smoothness}]
Since $\hyDX=\inbrace{\cDX}$ and $\mu_X=\cDX$, we have $\cDX(A)=\mu_X(A)\leq \frac{1}{\sigma}\mu_X(A)$ for all measurable $A$, so $\hyDX$ is $\sigma$-smooth with the stated $\sigma$.

For $\hyQ$, note that $\mu_Y$ has density $1/H$ on $[-BR,BR+1]$.
We have $\cQ(t)\leq H/c\cdot \mu_Y(t)$ and $\cQ'(t)\leq H/c\cdot \mu_Y(t)$ for all $t$, hence both are $\sigma$-smooth with the stated $\sigma$.

Moreover, for every $(f,\cDX,q)\in\hyF\times\hyDX\times\hyQ$, we have $\abs{f(x)}\leq BR$, and hence $\cP_{f,\cDX,q}$ is supported on $\cX\times[-BR,BR+1]$. Its density satisfies
\[
    \cDX(x)q\!\inparen{y-f(x)}
    \leq \frac{\cDX(x)}{c}
    = \frac{H}{c}\,\mu(x,y)
    = \frac{1}{\sigma}\mu(x,y)
    \leq \frac{1}{\sigma^2}\mu(x,y),
\]
where the last inequality uses $\sigma\leq 1$. Thus, every $\cP_{f,\cDX,q}$ is $\sigma^2$-smooth with respect to $\mu$.

Finally, consider the shifts $q_u(y)\coloneqq q(y-u)$ and let $\Delta\coloneqq \abs{u-v}$.
For $q=\cQ$ (uniform on a unit interval),
\[
    \int \abs{\cQ_u(y)-\cQ_v(y)}\,\d y
    =
    2\min\!\inbrace{\Delta,1}
    \le
    2\Delta
    \le
    \frac{6}{c}\Delta\,.
\]
For $q=\cQ'$, the function is piecewise constant with jump locations at $u$, $u+\frac12$, and $u+1$; the set where $\cQ'_u\neq \cQ'_v$ is contained in the union of three intervals of total length at most $3\Delta$, and pointwise $\abs{\cQ'_u-\cQ'_v}\leq 2\sup_t \cQ'(t)=2/c$. Hence
\[
    \int \abs{\cQ'_u(y)-\cQ'_v(y)}\,\d y
    \le
    \frac{2}{c}\cdot 3\Delta
    =
    \frac{6}{c}\abs{u-v}\,.
\]
\end{proof}

\begin{lemma}[Non-identifiability of $\sstar$]\label{lem:s-nonidentifiable}
The concept classes $(\hyF,\hyS,\hyDX,\hyQ)$ violate \cref{cond:identifiability:s}. In particular, there exists a non-identifiable $\sstar\in\hyS$.
\end{lemma}
\begin{proof}[Proof of \cref{lem:s-nonidentifiable}]
Fix $\cD\in\hyDX$ and consider the two selectors $s_{\mathrm{step}}\in\hyS$ and $s_0\in\hyS$, the two noises $\cQ,\cQ'\in\hyQ$, and the regressor $f_0\in\hyF$ used to define $s_{\mathrm{step}}$.
We claim that for every $(x,y)\in\supp(\cD)\times\R$,
\[
    s_{\mathrm{step}}(x,y)\,\cQ\!\inparen{y-f_0(x)}
    =
    s_0(x,y)\,\cQ'\!\inparen{y-f_0(x)}\,.
\]
Let $t\coloneqq y-f_0(x)$.
If $t\notin[0,1]$, then $\cQ(t)=\cQ'(t)=0$, so both sides equal $0$.
If $t\in[0,\nfrac12]$, then $s_{\mathrm{step}}(x,y)=\eta$ and $\cQ'(t)=\eta/c$, so
$s_{\mathrm{step}}(x,y)\cQ(t)=\eta\cdot 1=\eta$ and $s_0(x,y)\cQ'(t)=c\cdot(\eta/c)=\eta$.
If $t\in(\nfrac12,1]$, then $s_{\mathrm{step}}(x,y)=1$ and $\cQ'(t)=\nfrac1c$, so
$s_{\mathrm{step}}(x,y)\cQ(t)=1\cdot 1=1$ and $s_0(x,y)\cQ'(t)=c\cdot(\nfrac1c)=1$.
Thus the identity holds for all $(x,y)$.

Moreover, $s_{\mathrm{step}}\neq s_0$ (\eg{}, for any $x\in\cX$, $s_{\mathrm{step}}(x,f_0(x)+\nfrac34)=1$ while $s_0(x,f_0(x)+\nfrac34)=c<1$).
Therefore, taking $f=f'=f_0$ in \cref{cond:identifiability:s}, we have found distinct selectors $s_{\mathrm{step}},s_0\in\hyS$ (differing on a set of positive measure, as mentioned earlier) and noises $\cQ,\cQ'\in\hyQ$ such that the strict inequality required by \cref{cond:identifiability:s} never occurs. Hence \cref{cond:identifiability:s} fails, and $\sstar$ is not identifiable in general.
\end{proof}

\paragraph{Concluding Proof of \cref{thm:f-estimable-without-s}.}
We now conclude \cref{thm:f-estimable-without-s}.
By \cref{lem:rate-function}, $\Theta_0$ satisfies the estimation condition \cref{cond:estimation:f} with a linear rate.
By \cref{lem:smoothness} (and the pointwise lower bound $s\geq \eta$), the regularity and overlap hypotheses needed to invoke \cref{thm:estimation:f} hold.
By \cref{lem:s-nonidentifiable}, $\sstar$ is not identifiable from censored distributions realizable with respect to $\Theta_0$.

Finally, the covering-number bullets in \cref{thm:f-estimable-without-s} hold:
\begin{itemize}
    \item $\hyF$ admits a $\zeta$-cover in $\snorm{\cdot}_\infty$ of size at most $\inparen{1+\nfrac{2BR}{\zeta}}^d$ by \cref{lem:existence-cover};
    \item $\hyS$ and $\hyQ$ each have exactly two elements, so their covering numbers are $2$ at every scale;
    \item $\hyDX$ is a singleton, so its covering number is $1$ at every scale.
\end{itemize}
This completes the proof.
\end{proof}

\section{Proofs from \cref{sec:revisiting:dvn}}
    \label{sec:das_indentifiable}

\subsection{Proof of \cref{lem:dvn-embed}}

\begin{proof}[Proof of \cref{lem:dvn-embed}]
Let $\cD_X$ and $\cQ^\star$ be the distributions of $X$ and $\xi$, respectively, and set
\[
    s^\star(x,y)\coloneqq\Pr(D=1\mid X=x,Y^\star=y)\,.
\]
Consider the model in \cref{def:selection_bias_model} with parameters $(f^\star,s^\star,\cD_X,\cQ^\star)$.
Since $\xi\perp X$, the resulting $(X,Y)$ has the same distribution as $(X,Y^\star)$ in the original instance.
By the choice of $s^\star$, the conditional distribution of $D$ given $(X,Y)$ also agrees.
Thus the two models give the same joint distribution and the same observations.
The restrictions in \cref{def:dvn-model,def:dvn-id-restrictions} therefore hold, and $\cQ^\star$ has a density by assumption.
Hence $(f^\star,s^\star,\cD_X,\cQ^\star)\in\Theta_{\rm DNV}$.
\end{proof}

    \subsection{Proof of \cref{thm:das_indentifiable}}

\begin{proof}[Proof of \cref{thm:das_indentifiable}]
Fix $(f_j,s_j,\cD,\cQ_j)\in\Theta_{\rm DNV}$ for both $j=1$ and $j=2$, and suppose their observed kernels agree $(\cD\otimes\nu)$-almost everywhere.
We show that $f_1=f_2$ on $\cX$.
Integrating over $y$ gives $p_1=p_2$ $\cD$-almost everywhere.
By continuity and $\supp(\cD)=\cX$, this equality holds throughout $\mathrm{int}(\cX)$.
Let $p$ denote the common propensity score.

Since $p>0$ almost everywhere and the selected conditional means are finite, these means also agree $\cD$-almost everywhere.
Let $f_j(w,z)=g_j(w)$ and let $\lambda_j$ be as in \cref{def:dvn-id-restrictions}.
The index restriction, continuity, and $\supp(\cD)=\cX$ then give
\[
g_1(w)+\lambda_1\!\inparen{p(w,z)}
=
g_2(w)+\lambda_2\!\inparen{p(w,z)}
\qquad\text{for every }(w,z)\in\mathrm{int}(\cX)\,.
\]
Let $\ell\coloneqq g_1-g_2$ and $b\coloneqq\lambda_1-\lambda_2$.
Differentiating $\ell(w)+b(p(w,z))=0$ with respect to $z$ gives
\[
b'\!\inparen{p(w,z)}\,\frac{\partial p(w,z)}{\partial z}=0\,.
\]
By the rank condition, $b'\!\inparen{p(w,z)}=0$.
Differentiating the same identity with respect to $w$ now gives $\nabla_w\ell(w)=0$.
The projection of $\mathrm{int}(\cX)$ onto the $w$-coordinates is open and connected, so $\ell$ is constant there.
Thus $f_1-f_2$ is constant on $\mathrm{int}(\cX)$ and, by continuity, on $\cX=\overline{\mathrm{int}(\cX)}$.
Since $f_1(x_0)=f_2(x_0)=0$, this constant is zero.
Hence $f_1=f_2$ on $\cX$, proving \cref{cond:identifiability:f}.
\end{proof}

    \subsection{Proof of \cref{lem:trunc-dvn-fails}}
        \begin{proof}[Proof of \cref{lem:trunc-dvn-fails}]
            Fix $\cDX\in\hyD_{\rm T}$, take $\cQ^\star=\cN(0,1)$, and let $f^\star\equiv 0$, so $Y^\star=\xi$.
            Choose $\sstar(x,y)=\mathds{1}\!\inbrace{y>1}$ for $x>0$ and $\sstar(x,y)=\mathds{1}\!\inbrace{y<-1}$ otherwise.
            This selector belongs to $\hyS_{\rm T}$.
            By symmetry of $\xi$, both values of $x\in\zo$ have propensity $p(x)=\Pr(\xi>1)$.
            However, $\E\insquare{\xi\mid X=1,D=1}\geq 1$ and $\E\insquare{\xi\mid X=0,D=1}\leq -1$.
            Since $\supp(\cDX)=\zo$, both covariate values have positive probability.
            Thus $\E\insquare{\xi\mid X,D=1}$ cannot be a function of $p(X)$, and the index restriction fails.
        \end{proof}

    \subsection{Proof of \cref{lem:trunc-cond1}}

    \begin{proof}[Proof of \cref{lem:trunc-cond1}]
        Fix $\cDX\in\hyD_{\rm T}$, $s,s'\in\hyS_{\rm T}$, and $f,f'\in\hyF_{\rm T}$ that differ on a set of positive $\cDX$-measure.
        Since $\cX=\zo$, there is an $x\in\cX$ with $\cDX(\inbrace{x})>0$ and $f(x)\neq f'(x)$.
        Both noise distributions are $\cN(0,1)$.

Let $A_x$ be the set on which the observed kernels differ at this $x$:
        \[
            A_x\coloneqq\inbrace{y\in\R:
            s(x,y)\varphi(y-f(x))
            \neq
            s'(x,y)\varphi(y-f'(x))}\,.
        \]
        We show that $\nu(A_x)>0$.
        
If $s(x,\cdot)$ and $s'(x,\cdot)$ differ on a set of positive $\nu$-measure, then one selector is zero and the other is one on that set.
        The Gaussian density is positive everywhere, so the kernels differ there as well.
        Hence $\nu(A_x)>0$.
        
Otherwise, $s(x,\cdot)=s'(x,\cdot)$ $\nu$-almost everywhere.
        By the definition of $\hyS_{\rm T}$, there is a nontrivial open interval $I_x$ on which $s(x,y)=1$.
        Thus $s(x,y)=s'(x,y)=1$ for almost every $y\in I_x$.
        
Since $f(x)\neq f'(x)$, the densities $\varphi(y-f(x)) = \cN(0,1;\,y-f(x))$ and $\varphi(y - f'(x)) = \cN(0,1;\,y-f'(x))$ agree only at $y=\frac{f(x)+f'(x)}{2}$.
        The kernels therefore differ for almost every $y\in I_x$, so $\nu(A_x)\geq\nu(I_x)>0$.
        
In either case, \cref{cond:identifiability:f} holds, as the kernels differ on $\inbrace{x}\times A_x$ and
        \[
            (\cDX\otimes\nu)(\inbrace{x}\times A_x)=\cDX(\inbrace{x})\nu(A_x)>0\,.\qedhere
        \]
    \end{proof}

\section{OLS-Based Method with Sub-Gaussian Noise with (Approximately) Known $\sstar$}
        In this section, we provide an OLS-based method to estimate the outcome function when $\hyF=\hyFlin$ (the set of $d$-dimensional hyperplanes; see \cref{eq:def:linear-functions}) and the selection function satisfies a mild pointwise overlap condition (\Cref{asmp:point-wise_overlap}). For finite-sample confidence intervals on our estimate, we additionally restrict $\hyQ$ to be the set of sub-Gaussian distributions (\Cref{def:sub-gaussian}). We later relax this assumption to only require that $\hyQ$ consist of distributions with finite variance and provide a method to build asymptotically valid confidence intervals.
        Following \Cref{eq:outcome-model}, our model is
        \begin{equation}\label{eq:linear_model}
            Y = X^\top \beta^\star + \xi\,,
        \end{equation}
        where $\beta^\star \in \R^d$ is the target parameter and $\xi$ is noise. We assume that the distribution of $\xi$ satisfies the standard regularity assumptions below.
        \begin{assumption}[Noise Constraints]\label{asmp:ols_noise}
            Let $\hyQ$ be the noise distribution concept class. Let $\xi \in \hyQ$ be the random variable denoting noise in the outcome equation (\cref{eq:outcome-model}). 
            Assume that $\xi$ satisfies the following conditions for some known parameter $\sigma>0$:
            \begin{enumerate}
                \item (Exogeneity) It holds that for each $x\in \cX$, $\E[\xi | X=x] = 0$;
                \item (Finite Variance) It holds that $\Var[\xi] = \sigma^2 < \infty$\,.
            \end{enumerate}
        \end{assumption}
        The exogeneity is a standard OLS assumption that rules out confounding between the regressors and the error term, ensuring that the OLS estimate is unbiased. The finite variance assumption is also a standard assumption in OLS theory that is necessary to build confidence intervals for our estimate.

	         Before we proceed, we briefly define some additional notation. For a symmetric square matrix $M$, let $\lambda_{\text{min}}(M)$ and $\lambda_{\text{max}}(M)$ denote the smallest and largest eigenvalues of $M$, respectively. A square matrix is positive semidefinite (PSD) if $\lambda_{min}(M) \geq 0$ and positive definite (PD) if $\lambda_{min}(M) > 0$. It is well known that any PD matrix $M$ is invertible and that $\norm{M}_2 = \lambda_{\text{max}}(M)$. 
         
	         With this in mind, the next assumption below is on $\cDX$ and ensures that the OLS estimate is well-defined.
        \begin{assumption}[Non-Degeneracy]\label{asmp:ols_non_degen}
            Let $M =\E[XX^\top]$ be an outer-product matrix where $X \sim \cDX$. It holds that $M$ is positive definite.
        \end{assumption}
        Under this assumption, we can invert the matrix $M$ to produce our estimator. Finally, we will also need a standard overlap assumption on $\hyS$.
	        \begin{assumption}[Pointwise Overlap]\label{asmp:point-wise_overlap}
	            Let $s \in \hyS$ be a selection function. Fix \(\alpha>0\). It holds that
            \[
                \inf_{x,y} s(x,y) \geq \alpha\,. 
            \]
        \end{assumption}
        Pointwise overlap ensures that the estimate defined in \Cref{eq:ols_estimate} is stable and converges to a normal distribution. Now, under \Cref{asmp:ols_non_degen}, it is straightforward to verify that $\beta^\star$ uniquely minimizes the mean squared error (MSE), \ie{},
        \begin{equation}\label{eq:opt_ols}
            \beta^\star = \argmin_{\beta \in \R^d} \E\left[\frac{1}{2}(Y-X^T\beta)^2\right]\,.
        \end{equation}
        Let $\psi(X,Y,\beta) = X(Y-X^\top \beta)$ be the corresponding score function of the MSE. Since the MSE is convex, it follows that $\beta^\star$ is strongly identified by the following moment equation
        \begin{equation}\label{eq:ols_moment}
            \Ex\insquare{\psi(X,Y,\beta^\star)} = 0\,.
        \end{equation}
        We must adapt this moment equation to work in the limited independence model, where data is missing. To accomplish this, recall that $D_i$ is an indicator variable for whether the $i$th sample is observed and let $s(x,y) = \Pr(D=1 | X=x,Y=y)$.
        We may now define a new score function
        \begin{equation}\label{eq:ols_score}
            \psi'(X,Y,D,s(X,Y),\beta) = \frac{D}{s(X,Y)}\, X(Y-X^\top \beta)\,.
        \end{equation}
	        
        We will now show that $\psi'$ allows for strong identification of $\beta^\star$.
        \begin{claim}
            It holds that
            $\E\insquare{\psi'(X,Y,D,s(X,Y),\beta^\star)} = 0.$
        \end{claim}
        \begin{proof}
            \begin{align*}
                \E\insquare{\psi'(X,Y,D,s(X,Y),\beta^\star)}
                &= 
                \E\insquare{\frac{D}{s(X,Y)}\,X(Y-X^\top \beta^\star)}\\
                &= \E_{X,Y}\insquare{\E_{D}\insquare{\frac{D}{s(X,Y)}\,X(Y-X^\top \beta^\star)\given X,Y}} \\
                %
                &= \E\insquare{\frac{X(Y-X^\top \beta^\star)}{s(X,Y)}\E\insquare{D\given X,Y}}\\
                &= \E[X(Y-X^\top \beta^\star)] \tag{by the definition of $s$} \\
                &= 0\,. \tag{by the $\psi$ moment equation \eqref{eq:ols_moment}}
            \end{align*}
        \end{proof}
        From this modified moment equation, we derive an estimator for the target parameter,
        \begin{equation}\label{eq:ols_estimate}
            \hat{\beta} = \left(\frac{1}{n}\sum_{i=1}^n \frac{D_i}{s(x_i,y_i)}x_ix_i^\top \right)^{-1}\left(\frac{1}{n}\sum_{i=1}^n\frac{D_i}{s(x_i,y_i)}x_iy_i\right)\,.  
        \end{equation}
        Next, we describe how to bound our estimation error using concentration inequalities assuming that the noise $\xi$ is sub-gaussian. We will then relax this assumption and only assume that $\xi$ has finite variance. In this regime, we will characterize the asymptotic distribution of the $L_2$ error and use this result to build $1-\delta$ confidence regions.
        
        \paragraph{Finite-Sample Results.}   Our goal is to bound $\snorm{\hat{\beta}-\beta^\star}_2$ under the assumption that $\xi$ is a sub-Gaussian random variable. We begin by recalling the definition of a sub-Gaussian random variable. We note that there are various equivalent definitions of sub-Gaussianity (see \cite[Section 2.6]{vershynin2018high}); here, we use the one most convenient for us.
        \begin{definition}[Sub-Gaussianity]\label{def:sub-gaussian}
            A random variable $\xi$ is called sub-Gaussian with variance proxy $\sigma^2$ if $\E[\xi] = 0$ and its moment generating function satisfies
            \[
	                \E[e^{t\xi}] \leq e^{\sigma^2t^2/2}\,, \quad \forall t \in \R\,.
            \]
        \end{definition}
        Informally, a sub-Gaussian random variable has tails that decay at least as fast as those of a Gaussian random variable. We will assume that the noise $\xi$ in \Cref{eq:linear_model} is sub-Gaussian.
        \begin{assumption}[Sub-Gaussian Noise]\label{asmp:sub-gauss}
	            Let $\xi \in \hyQ$ be the noise defined in \Cref{eq:linear_model}. It holds that $\xi$ is a sub-Gaussian random variable with variance proxy \(\sigma^2\).
        \end{assumption}
        Next, we will also need an assumption on the magnitude of the covariates. 
        \begin{assumption}[Bounded Covariates]\label{assumption:bounded_covariates}
            Suppose that for a known constant $L>0$, for each $x\in \supp(\cDX)$, it holds that $\norm{x}_2\leq L$.
        \end{assumption}
        This assumption enables us to bound the operator norm of $\E[XX^T]$. 
	        We now prove a helpful claim, which shows that even under selection bias the noise remains sub-Gaussian.
	        \begin{claim}\label{claim:sub_gaussian_ols}
	            Let $\tilde{\xi}_i =\frac{D_i\xi_i}{s(X_i,Y_i)}$, where $\xi_i$ is sub-Gaussian with variance proxy $\sigma^2$. Then, $\tilde{\xi}_i$ is a sub-Gaussian random variable with variance proxy $\frac{\sigma^2}{\alpha^2}$ conditional on $D_i = 1$. Moreover, for every $u\in\mathbb S^{d-1}$, the random variable $u^\top X_i\tilde{\xi}_i$ is mean zero and sub-Gaussian with variance proxy $L^2\sigma^2/\alpha^2$.
	        \end{claim}
	        \begin{proof}
	            Write $s_i\coloneqq s(X_i,Y_i)$ and $p_i\coloneqq\Pr(D_i=1)=\E[s_i]\geq\alpha$. Since $s_i\geq\alpha$, convexity of the exponential implies, pointwise,
	            \[
	                s_i e^{t\xi_i/s_i}
	                \leq \alpha e^{t\xi_i/\alpha}+s_i-\alpha\,.
	            \]
	            The moment generating function of $\tilde{\xi}_i$ conditional on $D_i=1$ can be written as
	            \begin{align*}
	                \E[e^{t\tilde{\xi}_i}\mid D_i=1]
	                &=\frac{1}{p_i}\E[s_i e^{t\xi_i/s_i}]\\
	                &\leq 1+\frac{\alpha}{p_i}
	                    \left(e^{\sigma^2t^2/(2\alpha^2)}-1\right)\\
	                &\leq e^{\sigma^2t^2/(2\alpha^2)}\,.
	            \end{align*}
	            Lastly,
	            \[
	                \E[\tilde{\xi}_i\mid D_i=1]
	                =\frac{1}{p_i}\E\!\left[
	                    \frac{D_i\xi_i}{s_i}\right]
	                =\frac{1}{p_i}\E[\xi_i]
	                =0\,.
	            \]
	            Collectively, this proves that $\tilde{\xi}_i$ is a sub-Gaussian random variable. For the second assertion, fix $u\in\mathbb S^{d-1}$ and write $a_i\coloneqq u^\top X_i$. Applying the same convexity inequality and using the independence of $X_i$ and $\xi_i$ and the bound $|a_i|\leq L$ gives
	            \[
	                \E[e^{t a_i\tilde{\xi}_i}]
	                =\E[1-s_i+s_i e^{t a_i\xi_i/s_i}]
	                \leq \E[1-\alpha+\alpha e^{t a_i\xi_i/\alpha}]
	                \leq 1-\alpha+\alpha e^{L^2\sigma^2t^2/(2\alpha^2)}
	                \leq e^{L^2\sigma^2t^2/(2\alpha^2)}\,.
	            \]
	            Moreover,
	            \[
	                \E[a_i\tilde{\xi}_i]
	                =\E[a_i\xi_i]
	                =0\,,
	            \]
	            which completes the proof.
	        \end{proof}
	        We are now ready to prove finite-sample bounds for our OLS estimator. We concentrate the IPW score and weighted Gram matrix separately.
	        \begin{theorem}[Finite-Sample OLS Guarantee]
	            Let $\hat{\beta}$ be the OLS estimator defined by \Cref{eq:ols_estimate} and let $\beta^\star$ be the optimal OLS parameter defined by \Cref{eq:opt_ols}. Fix \(\delta\in(0,1)\). Let $n_0 = \frac{64 L^4}{\alpha^2 \lambda_{\min}(M)^2}\log\left(\frac{4d}{\delta}\right)$. There is a universal constant \(C>0\) such that, for every \(n\geq n_0\), with probability at least \(1-\delta\), 
	            \[
	            \snorm{\hat{\beta} - \beta^\star}_2 \leq \frac{CL\sigma}{\alpha\lambda_{\min}(M)}\sqrt{\frac{d+\log(4/\delta)}{n}}
	            \]
	            where $\sigma^2$ is the sub-Gaussian variance proxy and $\lambda_{\min}(M)$ is the smallest eigenvalue of $M =\E[XX^\top]$.
	        \end{theorem}
	        \begin{proof}
		             We begin by plugging \Cref{eq:linear_model} into our expression for $\hat{\beta}$. Write
		             \[
		             \widehat M\coloneq \frac1n\sum_{i=1}^n\frac{D_i}{s(X_i,Y_i)}X_iX_i^\top\,,\qquad
		             V_i\coloneq \frac{D_i}{s(X_i,Y_i)}X_i\xi_i\,,\qquad
		             \overline V_n\coloneq\frac1n\sum_{i=1}^nV_i\,.
		             \]
		             Rearranging then gives
	            \[
	               \hat{\beta}-\beta^\star=\widehat M^{-1}\overline V_n\,.
	            \]
	            By \Cref{claim:sub_gaussian_ols}, a standard \(1/2\)-net argument gives, for a universal constant \(C_1>0\),
	            \[
	               \|\overline V_n\|_2
	                 \leq \frac{C_1L\sigma}{\alpha}
	                 \sqrt{\frac{d+\log(4/\delta)}{n}}\,.
	            \]
	            This event has probability at least \(1-\delta/2\). It remains to relate $\lambda_{\min}(\widehat{M})$ to $\lambda_{\min}(M)$ via the matrix Bernstein inequality.

            We begin by observing that Weyl's inequality gives $\lambda_{\min}(\widehat{M}) \geq \lambda_{\min}(M) - \snorm{\widehat{M}-M}_2$. Hence, it suffices to prove that $\snorm{\widehat{M}-M}_2$ is bounded with high probability. Since $\E[\widehat{M}] = M$, we may use Bernstein's matrix inequality to show this bound. We briefly define some additional notation. Let $E_i = \frac{D_i}{s(x_i,y_i)}x_ix_i^\top$. Then,
            \[
                \widehat{M} = \frac{1}{n}\sum_{i=1}^n E_i \quad \text{and} \quad M = \E\insquare{\frac{1}{n}\sum_{i=1}^n E_i}\,.
            \]
            Define 
            \[
                Z_i \coloneq E_i - \E[E_i]
            \]
            to be the centered matrices. Then, 
            \[
            \widehat{M} - M = \frac{1}{n}\sum_{i=1}^n Z_i\,.
            \]
	            First, observe that
	            \[
	                 \norm{Z_i}_2 \leq \norm{E_i}_2+\norm{M}_2 \leq \frac{L^2}{\alpha}+L^2 \leq \frac{2L^2}{\alpha}\,.
	            \]
	            Next, we compute the matrix variance statistic defined by
	            \[
	                v \coloneq \norm{\sum_{i=1}^n \E[Z_i^2]}_2\,.
	            \]
	            In particular, 
	            \[
	                v = \norm{\sum_{i=1}^n \E[Z_i^2}]_2 \leq n \cdot \norm{\E[Z_i^2]}_2\, = n \cdot \norm{\E[(E_i-\E[E_i])^2]}_2\,.
            \]
            Now note that $E_i$ and $\E[E_i]$ are both PSD matrices. Hence, 
            $(E_i-\E[E_i])^2 \preceq 2E_i^2 + 2\E[E_i]^2$. Thus,
            \[
		            \norm{\E[Z_i^2}]_2 \leq 2 \norm{\E[E_i^2]}_2 + 2 \norm{\E[E_i]^2}_2\,.
            \]
            This follows from the triangle inequality. Now observe that
            \[
            E_i^2 = \left(\frac{D_i}{s(x_i,y_i)}\right)^2 (x_ix_i^\top)^2\,.
            \]
            Thus, 
            \[
                \norm{E_i^2}_2 = \left(\frac{D_i}{s(x_i,y_i)}\right)^2\norm{x_i}_2^4 \leq \frac{L^4}{\alpha^2}\,.
            \]
            The term $\norm{\E[E_i]^2}_2$ can similarly be bounded by the same value. Hence, 
            \[
                v \leq \frac{4nL^4}{\alpha^2}\,.
            \]
	            We may now apply the matrix Bernstein inequality to find that, with probability $1-\delta/2$,
	            \[
	                \norm{\widehat{M}-M}_2
	                \leq \frac{2\sqrt{2}L^2}{\alpha}\sqrt{\frac{\log(4d/\delta)}{n}}
	                +\frac{4L^2}{3\alpha}\frac{\log(4d/\delta)}{n}\,.
	            \]
	            Since $n \geq \frac{64 L^4}{\alpha^2 \lambda_{\min}(M)^2}\log\left(\frac{4d}{\delta}\right)$, it holds that $\norm{\widehat{M}-M}_2 \leq \frac{1}{2}\lambda_{\min}(M)$. Thus, \(\lambda_{\min}(\widehat M)\geq \lambda_{\min}(M)/2\), and, on the intersection of the two concentration events,
	            \[
	                \|\widehat\beta-\beta^\star\|_2
	                \leq \frac{2}{\lambda_{\min}(M)}\|\overline V_n\|_2\,.
	            \]
	            The result follows by applying a union bound over the matrix Bernstein and score-concentration events.
        \end{proof}

	    \paragraph{Misspecified Selection Function.} In practice, the selection function is often unknown and must be estimated first from data. We study how misspecifying the selection function changes the population target of the weighted moment equation. To proceed, we make a basic assumption on the quality of our selection function estimate.
	    \begin{assumption}\label{assumption:selection_quality}
	    Let $\hat{s}\colon\cX\times\R\to(0,1]$ be our estimate for the selection function $s(x,y)$. It holds that
    \[
         \E\left[\left(\frac{s(X,Y)}{\hat{s}(X,Y)}-1\right)^2\right]^{1/2} \leq \epsilon\,.
    \]
    \begin{remark}
	        We assume without loss of generality that $\hat{s}(x,y) \geq \alpha$. Indeed, replacing \(\hat{s}\) by \(\hat{s}_{\rm clip}(x,y)\coloneq\max\{\alpha,\hat{s}(x,y)\}\) cannot increase the relative distortion, since
	        \[
	            \left|\frac{s(x,y)}{\hat{s}_{\rm clip}(x,y)}-1\right|
	            \leq
	            \left|\frac{s(x,y)}{\hat{s}(x,y)}-1\right|\,.
	        \]
	    \end{remark}
	    \end{assumption}
	    For the population analysis below, fix \(\hat{s}\) satisfying \cref{assumption:selection_quality}, and write \(r(X,Y)\coloneq s(X,Y)/\hat{s}(X,Y)\) and
	    \[
	        A_{\hat{s}}\coloneq \E[r(X,Y)XX^\top]\,.
	    \]
	    Whenever \(A_{\hat{s}}\) is invertible, let
	    \[
	        \beta_{\hat{s}}\coloneq A_{\hat{s}}^{-1}\E[r(X,Y)XY]
	    \]
	    denote the population target of the misspecified weighted moment equation.
	    
	     We will now characterize how error in the selection function propagates into this population target.  
	    \begin{theorem}[Population OLS Bias under Misspecified Selection Function]\label{thm:ols_miss_select}
	       Let $\beta^\star$ be the optimal parameter defined by \Cref{eq:opt_ols} and let \(\beta_{\hat{s}}\) be the population target defined above. Assume
	       \[
	           \epsilon\,\E[\|X\|_2^4]^{1/2}<\lambda_{\min}(\E[XX^\top])\,.
	       \]
	       It holds that 
	        \[
	            \norm{\beta_{\hat{s}}-\beta^\star}_2 \leq \frac{\epsilon \cdot \E\insquare{\norm{X\xi}_2^2}^{1/2}}{\lambda_{\min}(\E[XX^\top]) - \epsilon \E\insquare{\norm{X}_2^4}^{1/2}}\,.
	        \]
         
	    \end{theorem}
	    \begin{proof}
	        By definition, \(\beta_{\hat{s}}\) satisfies
	        \[
	            \E\!\left[r(X,Y)X(Y-X^\top\beta_{\hat{s}})\right]=0\,.
	        \]
	        We can substitute in $r(x,y) = s(x,y)/\hat{s}(x,y)$ and $Y = X^\top \beta^\star + \xi$ to find that
	        \[
	            \E\left[r(X,Y)X(X^\top \beta^\star - X^\top \beta_{\hat{s}} + \xi)\right] = 0\,.
	        \]
	        By rearranging, it follows that 
	        \[
	        \beta_{\hat{s}}-\beta^\star = (\E[r(X,Y)XX^\top ])^{-1}\E\left[r(X,Y)X\xi\right]\,.
	        \]
	        Since \(\E[\xi\mid X]=0\),
	        \[
	            \E[X\xi]=\E[X\E[\xi\mid X]]=0\,,
	            \qquad
	            \E[r(X,Y)X\xi]=\E[(r(X,Y)-1)X\xi]\,.
	        \]
	         We may apply norms to both sides of the statement,
	    \begin{align*}
	        \norm{\beta_{\hat{s}}-\beta^\star}_2 
	        &= \norm{(\E[r(X,Y)XX^\top ])^{-1}\E\left[(r(X,Y)-1)X\xi\right]}_2
        \\
	        &\leq \norm{(\E[r(X,Y)XX^\top ])^{-1}}_2\cdot \norm{\E[(r(X,Y)-1)X\xi]}_2
        \tag{using the definition of the spectral norm}
        \\
         &= \norm{(\E[XX^\top +(r(X,Y)-1)XX^\top ])^{-1}}_2\cdot 
        \norm{\E[(r(X,Y)-1)X\xi]}_2\,,\\
    \end{align*}
    and bound each term separately. First, observe that
    \begin{align*}
         \norm{\E[(r(X,Y)-1)X\xi]}_2 &\leq \E[\norm{(r(X,Y)-1)X\xi}] \tag{using Jensen's inequality}\\
         &\leq \E[(r(X,Y)-1)^2]^{1/2}\cdot \E\insquare{\norm{X\xi}_2^2}^{1/2} \tag{using the Cauchy--Schwarz inequality}\\
         &\leq \epsilon \cdot \E\insquare{\norm{X\xi}_2^2}^{1/2}\,.
            \tag{using \Cref{assumption:selection_quality}}
    \end{align*}
	    Next, write \(M\coloneq\E[XX^\top]\) and \(\Delta\coloneq\E[(r(X,Y)-1)XX^\top]\). Weyl's inequality gives
	    \[
	        \lambda_{\min}(A_{\hat{s}})
	        =\lambda_{\min}(M+\Delta)
	        \geq\lambda_{\min}(M)-\|\Delta\|_2\,.
	    \]
	    Moreover,
	    \begin{align*}
	        \|\Delta\|_2
	        &\leq \E\!\left[|r(X,Y)-1|\,\|XX^\top\|_2\right]\\
	        &\leq \E[(r(X,Y)-1)^2]^{1/2}\E[\|X\|_2^4]^{1/2}\\
	        &\leq \epsilon\,\E[\|X\|_2^4]^{1/2}\,.
	    \end{align*}
	    The theorem's strict inequality makes the denominator below positive, and hence
	    \[
	        \|A_{\hat{s}}^{-1}\|_2
	        \leq \frac{1}{\lambda_{\min}(M)-\epsilon\,\E[\|X\|_2^4]^{1/2}}\,.
	    \]
	    The main result follows by combining both term bounds. 
	    \end{proof}
	    This theorem controls population misspecification bias, rather than sampling error. For a finite-sample estimator \(\hat\beta\) based on \(\hat{s}\),
	    \[
	        \|\hat\beta-\beta^\star\|_2
	        \leq \|\hat\beta-\beta_{\hat{s}}\|_2
	        +\|\beta_{\hat{s}}-\beta^\star\|_2\,.
	    \]
	    The first term requires a separate sampling argument; if \(\hat{s}\) is learned from the same sample, that argument must also account for first-stage dependence, for example through sample splitting or cross-fitting together with suitable first-stage rates.
	    \begin{remark}
        We briefly present below a simpler, albeit weaker, bound that uses \Cref{asmp:point-wise_overlap}. As before,
        \begin{align*}
	            \norm{\beta_{\hat{s}}-\beta^\star}_2 
	            &= \norm{(\E[r(X,Y)XX^\top ])^{-1}\E\left[(r(X,Y)-1)X\xi\right]}_2
            \\
	            &\leq \norm{(\E[r(X,Y)XX^\top ])^{-1}}_2\cdot \norm{\E[(r(X,Y)-1)X\xi]}_2\,.
        \end{align*}
	        Now observe that $r(x,y) = s(x,y)/\widehat{s}(x,y) \geq \alpha$ by \Cref{asmp:point-wise_overlap} and the range \(\widehat{s}(x,y)\leq1\). Thus, 
        \begin{align*}
	            \norm{(\E[r(X,Y)XX^\top ])^{-1}}_2 \leq \frac{1}{\lambda_{\min}(\E\insquare{r(X,Y)XX^\top)}} \leq \frac{1}{\alpha \cdot \lambda_{\min}}\left(\E\insquare{XX^\top}\right)\,.
        \end{align*}
	        The second term can be bounded as in \Cref{thm:ols_miss_select} by $\norm{\E[(r(X,Y)-1)X\xi]}_2 \leq \epsilon \cdot \E\insquare{\norm{X\xi}_2^2}^{1/2}$. This yields the overall bound of
	        \[
	             \norm{\beta_{\hat{s}}-\beta^\star}_2 \leq \frac{\epsilon \cdot\E\insquare{\norm{X\xi}_2^2}^{1/2}}{\alpha \cdot \lambda_{\min}}\left(\E\insquare{XX^\top}\right)\,.
        \]
        This bound is much weaker due to its dependence on $1/\alpha$.
    \end{remark}
    
	    \paragraph{Asymptotic Distribution.} We now study the asymptotic behavior of the weighted OLS estimator defined by \Cref{eq:ols_estimate}. Whereas the finite-sample analysis assumed sub-Gaussian noise and bounded covariates, in the asymptotic regime we replace sub-Gaussianity with finite variance and boundedness with \(\E\|X\|_2^4<\infty\). We additionally assume that the selection function $s(x,y)$ is fully known and satisfies the overlap condition in \Cref{asmp:point-wise_overlap}.
    
	    The main theorem below follows via standard M-estimation theory (\eg{}, \citet[Chapter~5]{vaart1998}) applied to our IPW-style OLS estimator.
	    \begin{theorem}[Asymptotic Confidence Intervals ]\label{thm:ols_asymptotic}
	        Let \(\hat{\beta}\) be the estimator defined in \Cref{eq:ols_estimate} and let \(\beta^\star\) be the target parameter defined in \Cref{eq:opt_ols}. Suppose \(d\) is fixed, \(X\perp\xi\) as in \Cref{def:selection_bias_model}, \(\E\|X\|_2^4<\infty\), and \cref{asmp:ols_noise,asmp:ols_non_degen,asmp:point-wise_overlap} hold, with overlap constant \(\alpha>0\). Then
	        \[
	            \sqrt{n}(\hat{\beta}-\beta^\star) \xrightarrow{d} \cN(0,\Sigma)\,,
        \]
        where $\Sigma = M^{-1}BM^{-1}$ with
        \[
        M \coloneq \E[XX^\top ] \quad \text{and} \quad B \coloneq \E\insquare{\frac{\xi^2}{s(X,Y)}XX^\top}\,.
        \] 
    \end{theorem}
    %
	    \begin{proof}[Proof Sketch]
	        We begin by recalling that $\beta^\star$ satisfies $\E\insquare{\psi'(X,Y,D,s(X,Y),\beta^\star)} = 0$, where $\psi'$ is the score function defined by \Cref{eq:ols_score}. It's straightforward to verify that
	        \begin{align*}
	            \nabla_\beta\E\insquare{(\psi'(X,Y,D,s(X,Y),\beta))} &= \nabla_\beta\E\insquare{\frac{D}{s(X,Y)}\, X(Y-X^\top \beta)} \\
	            &= \nabla_\beta\E\insquare{X(Y-X^\top \beta)}  \tag{using that $\E\insquare{D/s(X,Y) |X,Y} = 1$}\\
	            &= \E\insquare{\nabla_\beta (X(Y-X^\top \beta))}  \tag{using \(\E[\|X\|_2^2]<\infty\)}\\
	            &= -\E[XX^\top]\\ 
	            &= -M\,.
        \end{align*}
        By \Cref{asmp:ols_non_degen}, $M$ is positive definite and hence invertible. Now define the sample versions of the score functions:
        \[
		            \psi_i'(\beta) \;\coloneqq\; \frac{D_i}{s(X_i,Y_i)}\,X_i\bigl(Y_i - X_i^\top \beta\bigr)\,,
	            \qquad
		            \psi_n'(\beta) \;\coloneqq\; \frac1n\sum_{i=1}^n \psi_i'(\beta)\,,
	            \qquad
		            \psi'(\beta) \;\coloneqq\; \E[\psi_i'(\beta)]\,.
	        \]
	    By construction, $\psi_n'(\hat{\beta})=0$. Define
	    \[
	        \widehat M_n\coloneq\frac1n\sum_{i=1}^n
		        \frac{D_i}{s(X_i,Y_i)}X_iX_i^\top\,.
	    \]
	    Because the score is affine in \(\beta\), we have the exact identity
	       \[
	            0=\psi_n'(\hat{\beta})
		            =\psi_n'(\beta^\star)-\widehat M_n(\hat{\beta}-\beta^\star)\,.
	        \]
	    Rearranging gives
	    \[
	        \sqrt n(\hat{\beta}-\beta^\star)
		        =\widehat M_n^{-1}\frac1{\sqrt n}\sum_{i=1}^n\psi_i'(\beta^\star)\,.
	    \]
	    The law of large numbers gives \(\widehat M_n\xrightarrow{p}M\). Moreover, the multivariate central limit theorem gives
	    \[
	        \frac1{\sqrt n}\sum_{i=1}^n\psi_i'(\beta^\star)
	        \xrightarrow{d}\cN(0,B)\,,
	    \]
	    since
	    \[
	        \E\!\left[\frac{D_i}{s(X_i,Y_i)^2}X_iX_i^\top\xi_i^2\right]
	        =\E\!\left[\frac{\xi_i^2}{s(X_i,Y_i)}X_iX_i^\top\right]
	        =B\,.
	    \]
	    Slutsky's theorem proves the claim.
    \end{proof}
	    \paragraph{Asymptotic confidence regions.} By \Cref{thm:ols_asymptotic}, it holds that 
	        \[
	            n(\hat{\beta}-\beta^\star)^\top \Sigma^{-1}(\hat{\beta}-\beta^\star) \xrightarrow{d} \chi^2_d\,,
        \]
        where $\Sigma = M^{-1}BM^{-1}$ and $\chi^2_d$ is the chi-squared distribution with $d$ degrees of freedom. To build a feasible confidence region, we may define the plug-in estimators 
	        \[
	            \widehat{M}_n \coloneq \frac{1}{n}\sum_{i=1}^n \frac{D_i}{s(X_i,Y_i)}X_iX_i^T\,, \quad \widehat{B}_n \coloneq \frac{1}{n}\sum_{i=1}^n \frac{D_i}{s(X_i,Y_i)^2}X_iX_i^T \hat{\xi}^2_i\,,
        \]
        where $\hat{\xi}_i \coloneq Y_i-X_i^T\hat{\beta}$. Then, let
        \[
            \widehat{\Sigma}_n \coloneq \widehat{M}_n^{-1}\widehat{B}_n\widehat{M}_n^{-1}\,.
        \]
	        Under the assumptions of \Cref{thm:ols_asymptotic},
	        \[
		            \widehat M_n\xrightarrow{p}M\,,\qquad
		            \widehat M_n^{-1}\xrightarrow{p}M^{-1}\,,\qquad
		            \widehat B_n\xrightarrow{p}B\,,\qquad
		            \widehat\Sigma_n\xrightarrow{p}\Sigma\,.
	        \]
	        Since \(\Sigma\) is positive definite, \(\widehat\Sigma_n\) is invertible with probability tending to one and \(\widehat\Sigma_n^{-1}\xrightarrow{p}\Sigma^{-1}\).
	        Consequently,
	        \[
	            n(\hat{\beta}-\beta^\star)^\top
	            \widehat\Sigma_n^{-1}
	            (\hat{\beta}-\beta^\star)
		            \xrightarrow{d}\chi_d^2\,.
	        \]

	        \begin{corollary}[Asymptotic Confidence Region]
	        Let $c_\delta$ be the value which satisfies $\Pr(\chi^2_{d} > c_\delta) = \delta$. Then,
	        \[
	        C_\delta = \{\beta \in \R^d : n(\hat{\beta}-\beta)^\top \widehat{\Sigma}_n^{-1}(\hat{\beta}-\beta) \leq c_\delta\}
        \]
        is an asymptotically valid $(1-\delta)$ confidence region for $\beta^\star$.
        \end{corollary}

\end{document}